\documentclass[11pt,a4paper]{article}

\usepackage[margin=1in]{geometry}
\usepackage[T1]{fontenc}
\usepackage[utf8]{inputenc}
\usepackage{lmodern}
\usepackage{microtype}

\usepackage{amsmath,amssymb,amsthm,mathtools}
\usepackage{booktabs,longtable,array,multirow}
\usepackage{graphicx}
\usepackage{xcolor}
\usepackage[hyphens]{url}
\usepackage[round,authoryear]{natbib}
\usepackage[unicode,hidelinks]{hyperref}
\usepackage[capitalise,noabbrev]{cleveref}
\usepackage{caption}
\usepackage{enumitem}
\usepackage{amssymb}   

\theoremstyle{definition}
\newtheorem{definition}{Definition}[section]

\newtheorem{example}{Example}[section]
\newtheorem{remark}{Remark}[section]

\theoremstyle{plain}
\newtheorem{theorem}{Theorem}[section]
\newtheorem{lemma}{Lemma}[section]
\newtheorem{proposition}{Proposition}[section]
\newtheorem{corollary}{Corollary}[section]
\crefname{definition}{Definition}{Definitions}
\Crefname{definition}{Definition}{Definitions}
\crefname{theorem}{Theorem}{Theorems}
\Crefname{theorem}{Theorem}{Theorems}
\crefname{lemma}{Lemma}{Lemmas}
\Crefname{lemma}{Lemma}{Lemmas}
\crefname{proposition}{Proposition}{Propositions}
\Crefname{proposition}{Proposition}{Propositions}
\crefname{corollary}{Corollary}{Corollaries}
\Crefname{corollary}{Corollary}{Corollaries}
\crefname{assumption}{Assumption}{Assumptions}
\Crefname{assumption}{Assumption}{Assumptions}
\crefname{example}{Example}{Examples}
\Crefname{example}{Example}{Examples}
\crefname{remark}{Remark}{Remarks}
\Crefname{remark}{Remark}{Remarks}

\newcommand{\R}{\mathbb{R}}
\newcommand{\N}{\mathbb{N}}
\newcommand{\Z}{\mathbb{Z}}
\newcommand{\E}{\mathbb{E}}
\newcommand{\Prob}{\mathbb{P}}
\newcommand{\Ind}[1]{\mathbf{1}\!\left[#1\right]}
\newcommand{\Var}{\operatorname{Var}}
\newcommand{\Cov}{\operatorname{Cov}}
\DeclareMathOperator{\JSD}{JSD}

\newcommand{\graph}{G}
\newcommand{\Yg}{Y_{\graph}}

\title{\bfseries Agent Behavioral Contracts II:\\[3pt]
  {\normalfont\large Certifying Compositional Reliability Without Assuming Independence}}
\author{%
  Varun Pratap Bhardwaj\\
  Qualixar / Independent Researcher, India\\
  \texttt{varun.pratap.bhardwaj@gmail.com}\\[1pt]
  {\small ORCID: 0009-0002-8726-4289}
  \and
  Garima Singh\\
  Independent Researcher, India\\
  \texttt{garima1213@gmail.com}
  \and
  Arun Pratap Bhardwaj\\
  Independent Researcher, India\\
  \texttt{arun.pratap.bhardwaj@gmail.com}
}
\date{}

\begin{document}
\maketitle

\begin{abstract}
\noindent
Compositional reliability bounds for multi-agent systems multiply component
reliabilities, a step licensed by a conditional-independence assumption that is
routinely stated and rarely tested. We test it. Two instances of one model,
composed in a two-agent handoff, co-fail on $90.0\%$ of the missions on which
either fails ($\log\mathrm{OR}=6.66$, 95\% CI $[6.38,7.00]$; $\phi=0.916$). The
evidence is a preregistered confirmatory evaluation of $18{,}000$ missions with
deterministic scoring and no model in the judging loop, inside a larger campaign
whose other topologies we report as secondary. Substituting a
different model reduces the association significantly in the confirmatory
motif and in both secondary topologies (six of six contrasts); substituting a
different \emph{vendor},
with the model already different, does not --- a registered hypothesis that fails
to replicate and that we report as a null. An unmanipulated same-model pair,
present in every arm, returns fifteen of fifteen null contrasts: the result a
design free of graph-wide confounding would produce, and one that would have
been informative had it come out otherwise.

The error is signed and runs against the operator: positive dependence inflates
joint failure above the independence product, so redundancy is over-credited
exactly when components share a model. The assumption-free alternative is often
vacuous --- the certified floor is zero whenever mean component reliability falls
below $1-1/m$ --- and fitting a dependence model is worse: we prove that a
bootstrap bound on a fitted model's functional loses coverage of the true
reliability as $n \to \infty$, because the identification gap is $O(1)$ while the
bootstrap haircut is $O(n^{-1/2})$. More data makes such a certificate
worse, with no visible symptom.

We give a finite-sample certificate that assumes no dependence structure: a
linear program over the joint, taken over a Bonferroni--Clopper--Pearson box
around measured co-execution moments. It is sound, sharp for the information
supplied, and monotone in the moment family under a stated Bonferroni
allocation. On four-stage data, enriching from ten moment functionals to fourteen
narrows the identified interval by $85.7\%$ and lifts the certified floor from
$0.2455$ to $0.4116$. A companion anytime-valid certificate holds its empirical
type-I error at $0.0471$ or below across every admissible betting fraction,
recovering the SPRT exactly at the optimal bet. An ablation at a conceded design
effect $1.31$ times the largest measured moves the floor by at most $2.69$
percentage points.

We also show that the dependence statistics in common use --- Jaccard, $\phi$,
Kendall's $\tau_a$ --- are bounded by the marginals and can significantly
\emph{reverse} an apparent ordering of conditions when the compared agents fail
at different rates, a reversal we observe and then replicate on two further
inference backends. The contracts, the mission generators, the scoring code, the
analysis scripts, and the preregistration are released; every reported statistic
is regenerated by those scripts rather than transcribed.
\end{abstract}

\section{Introduction}
\label{sec:intro}

\subsection{The independence gap}

A multi-agent pipeline is certified the way a series system is certified: bound
each component's reliability, multiply, report the product. The v1 Agent
Behavioral Contract framework \citep{bhardwaj2026abc} does this, and so does
every compositional reliability argument we are aware of for agent systems. The
step is licensed by a conditional-independence condition --- C5 in
\cref{def:c5} --- which asserts that a downstream agent's contract satisfaction
is independent of an upstream agent's internal execution given a compliant
handoff.

C5 is stated in v1 and never tested. It is also, on reflection, implausible
exactly where multi-agent systems are most often deployed: a reviewer agent
checking a writer agent is very often the same model with a different prompt.
Two instances of one model do not have independent blind spots. They share
them.

We measure how much this matters. The preregistered confirmatory evaluation
is a two-agent handoff over 18{,}000 missions with deterministic scoring and no
model in the judging loop; the registration places the other topologies of the
campaign in secondary work, and we report them as such.
Two instances of \texttt{mistral-small-24b} co-fail on
$90.0\%$ of the missions on which either fails, with $\log\mathrm{OR} = 6.66$
(95\% CI $[6.38, 7.00]$). Replacing the second agent with a different model
reduces the association significantly, in the confirmatory motif and in both
secondary topologies. Replacing the \emph{vendor}, with the model already
different, does not.

The consequence runs against the operator. By \cref{prop:gap} the compositional
gap is signed: under positive dependence, joint failure exceeds the independence
product, so redundancy is over-credited precisely when the redundant components
share a model. A dashboard that multiplies reliabilities reports a reassuring
number whose governing assumption the data reject, and nothing in the pipeline
signals it.

\subsection{Why the obvious repairs fail}

Dropping the assumption gives the Fr\'echet--Hoeffding bounds, which are sharp
and frequently vacuous: by \cref{cor:vacuous} the certified floor is exactly zero
whenever mean component reliability falls below $1 - 1/m$, which for four
components at $p = 0.75$ is the entire regime of interest.

Fitting a dependence model --- a Gaussian one-factor copula, the standard choice
in reliability practice --- is worse than either. \Cref{thm:collapse} proves that
a bootstrap lower bound on the fitted model's functional loses coverage of the
\emph{true} reliability as $n \to \infty$: the identification gap is $O(1)$ while
the bootstrap haircut shrinks like $n^{-1/2}$, so past some finite sample size the
interval sits entirely above the truth and never returns. Collecting more data
does not repair such a certificate; it narrows the interval around the wrong
target, and nothing in the interval signals it.

\subsection{Approach}

We constrain the joint law with measured co-execution moments and optimise over
everything consistent with them. The resulting certificate is a linear program
over the $2^m$ cells of the joint, taken over a Bonferroni--Clopper--Pearson box
around the empirical moments. It is sound with no dependence assumption
(\cref{thm:tier1}), sharp for the information supplied (\cref{thm:lp-sharp}), and
tightens monotonically as the moment family grows
(\cref{prop:monotone,prop:bonferroni}). On real four-stage data, enriching from
ten moment functionals to fourteen narrows the identified interval by $85.7\%$
and lifts the certified floor from $0.2455$ to $0.4116$.

Because deployed systems are monitored continuously rather than at a sample size
fixed in advance, we add a certificate valid under arbitrary stopping
(\cref{thm:ville}). Its null constrains only a conditional mean, so it requires
no independence assumption at all --- it is immune to the failure the rest of the
paper documents.

\subsection{Contributions}

\begin{description}[leftmargin=2.6em,style=nextline,itemsep=0.15em]
  \item[C1 (co-lead)] A preregistered evaluation that \emph{manipulates}
    model sharing at three levels, confirmatory over 18{,}000 two-agent-handoff
    missions and replicated in two secondary topologies within a
    30{,}820-mission campaign, converting an observation about particular
    systems into a contrast attributable to substitution. Correlated failure among same-backbone agents
    is itself established by concurrent work
    (\cref{sec:related-dependence}); the controlled design, and the finding that
    the ordering holds at the model level and vanishes at the vendor level, are
    ours.
  \item[C2 (lead)] A finite-sample, copula-agnostic reliability certificate for
    composed agent pipelines (\cref{thm:tier1}), with a Bonferroni allocation
    that makes the floor monotone in the moment family by construction.
  \item[C3 (supporting)] \Cref{thm:collapse}: coverage of a model-based
    reliability floor tends to zero as the sample grows. The phenomenon is
    misspecification bias; what we add is the explicit witness of
    \cref{ex:witness}, indistinguishable from the fitted family on the moments
    the model uses.
  \item[C4 (co-lead)] Anytime-valid certification of graph reliability, with the
    SPRT recovered exactly at the optimal bet (\cref{prop:sprt}).
  \item[C5 (supporting)] A demonstration that marginal-sensitive dependence
    statistics can significantly \emph{reverse} an apparent condition ordering,
    replicated on independent backends (\cref{sec:e1-marginal,sec:e5}).
  \item[C6 (supporting)] An internal negative control returning fifteen of
    fifteen nulls, and a full artifact release including the preregistration.
\end{description}

\Cref{sec:novelty} states what we do \emph{not} claim, including one registered
hypothesis that fails to replicate.

\subsection{Organisation}

\Cref{sec:related} places the work against contract-based specification,
guardrails, the recent correlated-failure literature, and dependence-free
bounding. \Cref{sec:prelim} states the ABC framework in full, including formulas
withheld from v1 under a patent claim now withdrawn.
\Cref{sec:independence-fails} establishes what breaks without C5.
\Cref{sec:tiers,sec:lp,sec:anytime} construct the certificate.
\Cref{sec:enforcement,sec:impl} describe runtime enforcement and the
implementation. \Cref{sec:evaluation} reports E1--E6.
\Cref{sec:discussion} covers limitations and threats to validity, and
\cref{sec:novelty,sec:future,sec:conclusion} close.

\section{Background and Related Work}
\label{sec:related}

\subsection{Contracts for software and for agents}

Design by Contract \citep{meyer1992,hoare1969} specifies preconditions,
postconditions, and invariants that a component must honour, and makes
composition tractable by letting each component's postcondition discharge the
next one's precondition. The formal-methods lineage that followed
\citep{clarke1999,lamport2002,leino2010,barnett2004} verifies such
specifications statically, static analysis bounds behaviour by abstraction
\citep{cousot1977}, and dynamic-invariant work \citep{ernst2007} infers them
from traces.

None of it transfers directly to agents driven by natural-language
instructions, because the specification surface is prompts rather than types and
the execution is stochastic. The v1 framework \citep{bhardwaj2026abc} closes that
gap by making preconditions, hard and soft invariants, governance policies, and
recovery mechanisms first-class and runtime-enforceable, and by replacing
deterministic satisfaction with the probabilistic $(p,\delta,k)$-notion of
\cref{def:pdk}. We take that framework as given and measure what happens when several
contracted agents are composed.

\subsection{Steering, filtering, and guarding}

Three families of technique constrain agent behavior, and none of them
certifies a composed system.

Training-time alignment --- Constitutional AI \citep{bai2022}, RLHF
\citep{ouyang2022} --- shapes general tendencies but cannot encode a
deployment-specific invariant, and offers no runtime guarantee. Output guardrails
\citep{rebedea2023} filter or redirect responses matching prohibited patterns at
inference time; the current generation adds programmable policy languages and
validator libraries, and operates per-response. Observability platforms trace and
score agent runs after the fact.

The common limitation is scope: each acts on a single turn or a single agent.
None specifies invariants over a multi-agent pipeline, and none produces a
statement of the form ``this composition satisfies its contract with probability
at least $\hat L$.'' \Cref{tab:comparison} makes the comparison explicit.

\begin{table}[t]
\centering
\small
\setlength{\tabcolsep}{5pt}
\caption{Capability comparison against representative approaches. \checkmark{}
= supported, $\circ$ = partial, --- = not addressed.
\textbf{RE} runtime enforcement ·
\textbf{FS} formal specification ·
\textbf{MA} multi-agent scope ·
\textbf{DM} dependence measured ·
\textbf{CB} composed reliability bound ·
\textbf{CA} copula-agnostic ·
\textbf{AV} anytime-valid.
The last three are the contribution of this paper; \textbf{DM} is where
concurrent work has recently arrived.}
\label{tab:comparison}
\begin{tabular}{@{}lccccccc@{}}
\toprule
& RE & FS & MA & DM & CB & CA & AV \\
\midrule
Training-time alignment \citep{bai2022,ouyang2022}
  & --- & --- & --- & --- & --- & --- & --- \\
Output guardrails \citep{rebedea2023}
  & \checkmark & $\circ$ & --- & --- & --- & --- & --- \\
Static verification \citep{clarke1999,leino2010}
  & --- & \checkmark & $\circ$ & --- & $\circ$ & --- & --- \\
Agent observability
  & --- & --- & \checkmark & --- & --- & --- & --- \\
Multi-agent debate \citep{du2023,liang2023}
  & --- & --- & \checkmark & --- & --- & --- & --- \\
Byzantine-tolerance view \citep{zhang2025byzantine,tang2026agree}
  & --- & $\circ$ & \checkmark & $\circ$ & --- & --- & --- \\
Failure attribution \citep{falat2026,verifymas2026}
  & --- & --- & \checkmark & $\circ$ & --- & --- & --- \\
Correlated-failure detection \citep{triage2026}
  & --- & --- & \checkmark & \checkmark & --- & --- & --- \\
Multi-agent calibration \citep{cagecal2026}
  & --- & --- & \checkmark & \checkmark & --- & --- & --- \\
ABC v1 \citep{bhardwaj2026abc}
  & \checkmark & \checkmark & $\circ$ & --- & $\circ$ & --- & --- \\
\midrule
\textbf{This work}
  & \checkmark & \checkmark & \checkmark & \checkmark & \checkmark & \checkmark & \checkmark \\
\bottomrule
\end{tabular}
\end{table}

\subsection{Correlated failure in multi-agent systems}
\label{sec:related-dependence}

That agents sharing a base model may fail together is no longer a conjecture,
and we state plainly what is already established so that our own contribution is
not overstated.

\citet{triage2026} study a three-agent triage architecture and quantify
correlated failure directly, reporting a joint error rate inflated $3.53\times$
over independence (BCa 95\% CI $[3.50, 3.59]$) with $\phi = 0.612$, and finding
$57.2\%$ of errors occurring under agent agreement. Their agents are classical
learners --- a random forest, a $k$-nearest-neighbour model, and a calibrated
meta-model --- over intrusion-detection and clinical-readmission data, and model
sharing is not a manipulated variable; the design varies learner \emph{family} by
construction rather than varying it experimentally.
\citet{cagecal2026} recalibrate multi-agent confidence against a counterfactual
no-communication baseline and observe, qualitatively, that within-family blocks
of agents sharing a backbone exhibit shared blind spots. The
\citet{aisafety2026} report lists correlated failure among same-base-model agents
as a known concern, as does the survey of \citet{multiagentrisks2025}; work on
governed capability evolution \citep{govcap2026} addresses the related problem of
keeping a guarantee valid as components are upgraded. Work on failure attribution \citep{falat2026,verifymas2026}
traces which agent caused a trajectory to fail, and Byzantine-tolerance analyses
\citep{zhang2025byzantine,tang2026agree} study whether agents can reach agreement
at all under faults.

Three things separate the present work from that literature. First, model
sharing here is a \emph{manipulated} experimental variable --- three sharing
levels, preregistered, with a confirmatory core of $18{,}000$ missions on the
two-agent handoff motif and two further topologies reported as secondary
replication --- rather than a property of a fixed architecture, which is what makes
the arm contrasts of \cref{sec:e1} causal claims about substitution rather than
observations about a particular ensemble. Second, none of this work produces a
\emph{bound}: measuring that failures correlate does not tell a practitioner what
reliability they may certify, and \cref{sec:tiers,sec:lp} supply exactly that.
Third, the statistic matters. \citet{triage2026} lead with $\phi$, which
\cref{def:dependence} classifies as marginal-sensitive;
\cref{sec:e1-marginal} shows that such statistics can reverse the apparent
ordering of two conditions purely through a difference in marginal failure rates.
That result is a caveat on our own headline numbers and, we think, a useful one
for reading theirs.

\subsection{Bounds without independence}

Bounding a joint probability from marginals is classical. The
Fr\'echet--Hoeffding inequalities \citep{frechet1951,hoeffding1940} give the
sharp sandwich of \cref{thm:frechet}, and Boole's problem
\citep{boole1854,hailperin1965} recasts the search for the extremal joint as a
linear program --- the device \cref{sec:lp} instantiates over co-execution
moments. Copula theory \citep{sklar1959,nelsen2006} parameterises dependence
structures, and the Gaussian one-factor model of \cref{rem:tier2} is the standard
choice in reliability practice \citep{barlow1975}. \Cref{thm:collapse} is our
reason for refusing to certify on it.

The moment-problem view \citep{bertsimas2005} bounds expectations subject to
moment constraints, and Bonferroni-type inequalities \citep{bonferroni1936}
supply the multiplicity correction \cref{prop:bonferroni} uses. Our contribution
is not the LP, which is standard, but its use as a \emph{finite-sample
certificate} for agent pipelines: the combination of an exact Clopper--Pearson
box \citep{clopper1934} with the extremal LP, and the resulting guarantee of
\cref{thm:tier1}.

\subsection{Sequential and anytime-valid inference}

Wald's SPRT \citep{wald1945} certifies with an expected sample size far below
the fixed-$n$ requirement, but at a stopping rule fixed in advance.
Game-theoretic probability \citep{shafer2019,ville1939} and the recent literature
on e-values and testing by betting
\citep{shafer2021,ramdas2023,waudbysmith2024,grunwald2024}, together with
time-uniform confidence sequences \citep{robbins1970,howard2021} replace that with certificates
valid under arbitrary stopping. \Cref{sec:anytime} applies this to graph
reliability, where the payoff is specific: the null constrains only a conditional
mean, so the certificate requires no independence assumption at all --- the exact
assumption this paper shows to be false for the composition bound.
\Cref{prop:sprt} records that the SPRT is recovered exactly at the optimally
tuned bet, so anytime validity is obtained without loss against a known
alternative.

\subsection{Agent evaluation}

Benchmarks for agent capability \citep{jimenez2024,liu2024agentbench,zhou2024}
measure task success, and agent architectures
\citep{yao2023,shinn2023,wang2023,park2023,hong2024,qian2024,wu2024autogen,chase2022}
optimise it. None reports whether component failures are dependent, which is the
quantity a composed guarantee needs; task-success rates are marginals, and
\cref{sec:e1} shows marginals do not determine the composed outcome.
\Cref{sec:evaluation} treats cross-agent failure dependence as the measured
quantity under a preregistered manipulation of model sharing. We know of no
earlier evaluation combining the four elements that design rests on: a
registered hypothesis set, a manipulated sharing condition, three graph
topologies, and deterministic contract scoring with no model in the judging
loop. Correlated failure itself is not new --- \cref{sec:related-dependence}
reports it. What is new here is measuring it as an intervention rather than
observing it, with a scoring rule that cannot itself induce correlation across
arms.

\section{Preliminaries: the ABC framework}
\label{sec:prelim}

This section states the Agent Behavioral Contract framework in full. The
framework is due to the v1 paper \citep{bhardwaj2026abc}, where several of these
formulas were withheld under a patent claim that has since been withdrawn; every
one is stated here without redaction, and \cref{app:formulas} gives the complete
catalogue with the corresponding v1 equation numbers. Readers familiar with v1
can skip to \cref{sec:independence-fails}, which is where this work departs from
it.

\subsection{Contracts and compliance}

\begin{definition}[Behavioral contract; v1 Def.~3.1]
\label{def:contract}
A \emph{behavioral contract} is a tuple
\[
  \mathcal{C} = (\mathcal{P},\, \mathcal{I}_{\mathrm{hard}},\,
                 \mathcal{I}_{\mathrm{soft}},\, \mathcal{G}_{\mathrm{hard}},\,
                 \mathcal{G}_{\mathrm{soft}},\, \mathcal{R}),
\]
where $\mathcal{P} = \{p_1,\dots,p_m\}$ is a finite set of \emph{preconditions},
predicates over the initial state $s_0$; $\mathcal{I}_{\mathrm{hard}}$ and
$\mathcal{I}_{\mathrm{soft}}$ are \emph{hard} and \emph{soft invariants};
$\mathcal{G}_{\mathrm{hard}}$ and $\mathcal{G}_{\mathrm{soft}}$ are hard and soft
\emph{governance policies} constraining operational authority such as tool use
and spending; and
$\mathcal{R} : (\mathcal{I}_{\mathrm{soft}} \cup \mathcal{G}_{\mathrm{soft}})
\times S \rightharpoonup A^{*}$
is a partial \emph{recovery map} from a violated soft constraint and a state to a
corrective action sequence. We write
$\mathcal{C} = (\mathcal{P}, \mathcal{I}, \mathcal{G}, \mathcal{R})$ with
$\mathcal{I} = \mathcal{I}_{\mathrm{hard}} \cup \mathcal{I}_{\mathrm{soft}}$
where the partition is not in play.
\end{definition}

The hard/soft partition is the load-bearing distinction. A single hard violation
is a contract breach. A soft violation is tolerated if recovery occurs inside a
bounded window, which is what makes contracts usable against systems that are
stochastic by construction.

\begin{definition}[Constraint evaluation scores; v1 Def.~3.6, eqs.~1--2]
\label{def:scores}
For a state--action pair $(s_t, a_t)$,
\begin{align}
  C_{\mathrm{hard}}(t) &=
    \frac{\left|\{c \in \mathcal{I}_{\mathrm{hard}} \cup \mathcal{G}_{\mathrm{hard}}
      : c(s_t,a_t) = \text{true}\}\right|}
         {\left|\mathcal{I}_{\mathrm{hard}} \cup \mathcal{G}_{\mathrm{hard}}\right|},\\
  C_{\mathrm{soft}}(t) &=
    \frac{\left|\{c \in \mathcal{I}_{\mathrm{soft}} \cup \mathcal{G}_{\mathrm{soft}}
      : c(s_t,a_t) = \text{true}\}\right|}
         {\left|\mathcal{I}_{\mathrm{soft}} \cup \mathcal{G}_{\mathrm{soft}}\right|}.
\end{align}
Both lie in $[0,1]$. The evaluator is stateless: any turn is evaluable
independently of the turns before it.
\end{definition}

\begin{definition}[Turn compliance]
\label{def:turn-compliance}
A turn is \emph{hard-compliant} iff it incurs zero hard violations, i.e.\
$C_{\mathrm{hard}}(t) = 1$, and \emph{soft-compliant} iff zero soft violations
remain outstanding after recovery. The binary hard verdict
$h_t = \Ind{C_{\mathrm{hard}}(t) = 1}$ is the quantity every dependence estimate
in \cref{sec:evaluation} is computed on.
\end{definition}

\subsection{Drift}

\begin{definition}[Composite drift score; v1 Def.~3.12, eqs.~7--9]
\label{def:drift}
The \emph{drift} of an agent at time $t$ is the convex combination
\begin{equation}
  D(t) = w_c\, D_{\mathrm{compliance}}(t) + w_d\, D_{\mathrm{distributional}}(t),
  \qquad w_c + w_d = 1,
\end{equation}
with
\begin{equation}
  D_{\mathrm{compliance}}(t) = 1 - \bar{C}(t)
    = \frac{\sum_i w_i\bigl(1 - \sigma_i(t)\bigr)}{\sum_i w_i},
  \qquad
  D_{\mathrm{distributional}}(t)
    = \JSD\bigl(P_{\mathrm{obs}}(t) \,\|\, P_{\mathrm{ref}}\bigr).
\end{equation}
Defaults are $w_c = 0.6$, $w_d = 0.4$. Both components lie in $[0,1]$, so
$D(t) \in [0,1]$.
\end{definition}

\begin{definition}[Distributional drift term; v1 eq.~9]
\label{def:jsd}
$P_{\mathrm{obs}}(t)$ is the empirical action distribution over a sliding window
of recent actions and $P_{\mathrm{ref}}$ is a calibrated baseline from a
compliant reference session. With $M = \tfrac12(P+Q)$,
\begin{equation}
  \JSD(P \,\|\, Q) = \tfrac12 D_{\mathrm{KL}}(P \,\|\, M)
                   + \tfrac12 D_{\mathrm{KL}}(Q \,\|\, M).
\end{equation}
$\JSD$ \citep{lin1991} is symmetric and bounded in $[0,1]$, and $\sqrt{\JSD}$
is a metric \citep{endres2003}, which is why it is preferred here to raw KL divergence.
\end{definition}

\begin{definition}[$(p,\delta,k)$-satisfaction; v1 Def.~3.7, eqs.~3--4]
\label{def:pdk}
An agent $A$ \emph{$(p,\delta,k)$-satisfies} contract $\mathcal{C}$ over session
length $T$ if both hold:
\begin{align}
  \Prob\bigl[C_{\mathrm{hard}}(t) = 1 \;\;\forall t \in \{0,\dots,T\}
    \,\big|\, P(s_0)\bigr] &\geq p, \label{eq:pdk-hard}\\
  \Prob\bigl[\forall t:\; C_{\mathrm{soft}}(t) < 1-\delta
    \implies \exists\, t' \in \{t,\dots,\min(t+k,T)\}:\;
    C_{\mathrm{soft}}(t') \geq 1-\delta \,\big|\, P(s_0)\bigr] &\geq p.
    \label{eq:pdk-soft}
\end{align}
Here $p \in [0,1]$ is a probability threshold, $\delta \in [0,1]$ a tolerable
deviation, and $k \in \N$ a recovery window.
\end{definition}

\begin{remark}
\Cref{eq:pdk-soft} is a probabilistic guarantee about \emph{recoverable}
compliance. It is not the deterministic condition
$\max_t |C_{\mathrm{soft}}(t) - 1| \leq \delta$, which appears in some
descriptions of this framework and states something strictly stronger and
different. \Cref{def:pdk} follows the v1 paper.
\end{remark}

\begin{definition}[Reliability index; v1 Def.~3.20, eq.~13]
\label{def:theta}
\begin{equation}
  \Theta = \alpha_1 \bar{C}(t) + \alpha_2 \bigl(1 - \bar{D}(t)\bigr)
         + \alpha_3 \frac{1}{1+E} + \alpha_4 S,
  \qquad \textstyle\sum_{i=1}^{4}\alpha_i = 1,
\end{equation}
with defaults $\alpha_1 = 0.35$ (compliance), $\alpha_2 = 0.25$ (stability),
$\alpha_3 = 0.20$ (event frequency, $E$ the count of violation events), and
$\alpha_4 = 0.20$ (stress resilience). $S$ is the \emph{stress resilience index}
$\E[C(t) \mid \text{stressed}] / \E[C(t) \mid \text{baseline}]$ (v1 eq.~12),
which is a distinct quantity from the recovery success rate; descriptions that
equate the two are in error.
\end{definition}

\subsection{Drift dynamics}

\begin{definition}[Ornstein--Uhlenbeck drift dynamics; v1 Def.~4.1, eq.~14]
\label{def:ou}
Drift evolves as
\begin{equation}
  dD = \bigl(\alpha - \gamma D(t)\bigr)\,dt + \sigma\, dW(t),
\end{equation}
with $\alpha > 0$ the natural drift rate, $\gamma > 0$ the contract-recovery
mean-reversion rate, $\sigma > 0$ the diffusion coefficient, and $W$ a standard
Wiener process. Treating $D(t) \geq 0$ is a modelling simplification (v1
Remark~4.2); the Gaussian stationary law of \cref{thm:ou} places negligible but
non-zero mass below zero.
\end{definition}

\begin{theorem}[Stationary drift law and bound; v1 Thm.~4.3]
\label{thm:ou}
Let $D$ solve \cref{def:ou} with $\alpha,\gamma,\sigma > 0$, and let the
initial condition satisfy $\E[D(0)^2] < \infty$ with $D(0)$ independent of the
driving Wiener process. Then:
\begin{enumerate}[label=(\roman*),leftmargin=2.2em,itemsep=0.1em]
  \item the stationary law is
    $\pi_D = \mathcal{N}\!\bigl(\alpha/\gamma,\; \sigma^{2}/(2\gamma)\bigr)$;
  \item $\E_{\pi}[D] = \alpha/\gamma$, so $\E_{\pi}[D] < 1$ iff $\gamma > \alpha$;
  \item $\Var_{\pi}(D) = \sigma^{2}/(2\gamma)$;
  \item for $\eta > 0$,
    $\Prob_{\pi}\bigl(D > \alpha/\gamma + \eta\bigr)
     \leq \exp\bigl(-\gamma \eta^{2}/\sigma^{2}\bigr)$;
  \item writing $e(t) = D(t) - \alpha/\gamma$,
    $\E[e(t)^{2}] = \E[e(0)^{2}]\,e^{-2\gamma t}
      + \frac{\sigma^{2}}{2\gamma}\bigl(1 - e^{-2\gamma t}\bigr)$,
    so convergence to stationarity is exponential at rate $2\gamma$. (For
    deterministic $D(0)$ this reads $e(0)^2 e^{-2\gamma t} + \cdots$.)
\end{enumerate}
\end{theorem}

\begin{proof}
See \cref{app:ou}.
\end{proof}

Part (ii) is the design criterion: a contract whose recovery rate $\gamma$
exceeds the natural drift rate $\alpha$ holds expected drift below $1$
indefinitely. Part (iv) converts that mean statement into the tail bound a
practitioner needs, and part (v) says the transient decays at rate $2\gamma$, so
the recovery rate sets both the stationary level and the time to reach it.

\subsection{Composition}

\begin{definition}[Composability conditions; v1 Def.~4.7]
\label{def:c1c4}
Agents $A$ and $B$ with contracts $\mathcal{C}_A, \mathcal{C}_B$ and a handoff
invariant $\mathcal{I}_h$ are \emph{composable} if:
\begin{description}[leftmargin=3.2em,style=nextline,itemsep=0.1em]
  \item[C1 Interface compatibility]
    $\mathrm{Type}(\mathrm{Post}_A) \subseteq \mathrm{Type}(\mathcal{P}_B)$.
  \item[C2 Assumption discharge]
    $\mathrm{Post}_A \wedge \mathcal{I}_h \implies \mathcal{P}_B$.
  \item[C3 Governance consistency]
    $\mathrm{Allowed}(\mathcal{G}_A) \cap
     \mathrm{Prohibited}(\mathcal{G}_B) = \emptyset$.
  \item[C4 Recovery independence]
    any post-recovery state of $A$ still satisfies $\mathcal{P}_B$.
\end{description}
\end{definition}

\begin{definition}[Conditional independence of contract satisfaction; v1 Thm.~4.11 preamble]
\label{def:c5}
Let $E_A$, $E_B$, $E_h$ denote the events that $A$ satisfies its contract, that
$B$ satisfies its contract, and that the handoff invariant holds. Condition
\textbf{C5} states
\begin{equation}
  \Prob(E_B \mid E_A \cap E_h) = \Prob(E_B \mid E_h).
  \label{eq:c5}
\end{equation}
C1--C4 are required for the deterministic composition theorem; C5 is required
\emph{additionally} for the probabilistic one. Treating the two sets as a single
list ``C1--C5'' conflates the deterministic and probabilistic results.
\end{definition}

\Cref{eq:c5} is the assumption this paper is about. \Cref{sec:independence-fails}
shows what fails when it does not hold, and \cref{sec:e1} measures the extent to
which it does not hold in practice.

\begin{definition}[Naive compositional bound; v1 Thm.~4.11, eqs.~28--29]
\label{def:naive-bound}
Under C1--C5, for the two-agent composition $A \oplus B$,
\begin{equation}
  p_{A \oplus B} \;\geq\; p_A \cdot p_B \cdot p_h,
  \qquad
  \delta_{A \oplus B} \;\leq\; \delta_A + \delta_B + \delta_h,
  \label{eq:naive}
\end{equation}
extending to an $N$-agent chain (v1 Cor.~4.13, eqs.~30--31) as
$p_{\mathrm{chain}} \geq \prod_{i=1}^{N} p_i \prod_{i=1}^{N-1} p_{h_i}$ and
$\delta_{\mathrm{chain}} \leq \sum_i \delta_i + \sum_i \delta_{h_i}$.
\end{definition}

\subsection{Certification}

\begin{definition}[Sequential probability ratio certification; v1 §7]
\label{def:sprt}
For $H_0: p \leq p_0$ against $H_1: p \geq p_1$, the SPRT accumulates the
Bernoulli log-likelihood ratio and stops at the first boundary crossing. Its
expected sample size satisfies
\begin{equation}
  \E[N^{*}] \;\approx\; \frac{\log(1/\alpha_{\mathrm{err}})}
                             {D_{\mathrm{KL}}(\hat p \,\|\, p_0)},
\end{equation}
scaling as $O(\log(1/\alpha)/D_{\mathrm{KL}})$.
\end{definition}

\begin{definition}[Fixed-sample certification baseline]
\label{def:hoeffding}
The fixed-sample requirement of \citet{hoeffding1963} for absolute accuracy $\varepsilon$ at
error $\alpha_{\mathrm{err}}$ is
$N_H = \log(2/\alpha_{\mathrm{err}}) / (2\varepsilon^{2})$, scaling as
$O(1/\varepsilon^{2})$. At $p_0 = 0.85$, $p_1 = 0.95$, $\alpha = 0.05$ the SPRT
needs roughly $58$ sessions against Hoeffding's $184$; at
$\varepsilon = 0.01$ the gap widens to about $300$ against 18{,}445.
\end{definition}

\begin{remark}
\label{rem:sprt-limit}
Both \cref{def:sprt,def:hoeffding} certify at a stopping rule fixed in advance.
Neither survives a team that watches a dashboard and stops when the number looks
good. \Cref{sec:anytime} replaces them with a certificate valid under arbitrary
stopping, and \cref{sec:e4} measures its type-I error.
\end{remark}

\subsection{Cost}

\begin{proposition}[Per-action enforcement cost; v1 Prop.~4.15]
\label{prop:cost}
Enforcement costs $O(k + |A|)$ per action, where $k$ is the number of
constraints and $|A|$ the action-vocabulary size: $O(k)$ for constraint
evaluation, $O(|A|)$ for the JSD histogram update, and $O(1)$ for compliance
aggregation. For $k < 100$ and $|A| < 50$ the measured overhead is under
10\,ms per action (v1 Remark~4.16).
\end{proposition}

\subsection{Graphs, motifs, and co-failure}

The remaining definitions are new to this paper and set up
\cref{sec:independence-fails} onward.

\begin{definition}[Agent graph and graph outcome]
\label{def:graph}
An \emph{agent graph} $\graph$ is a finite DAG whose nodes are agents, each
carrying a contract, and whose edges are handoffs carrying handoff invariants.
Node $i$ produces a hard verdict $h_i \in \{0,1\}$. A \emph{composition rule}
$\rho$ maps the node verdicts to the graph outcome
$\Yg = \rho(h_1,\dots,h_m) \in \{0,1\}$. We write $p_i = \Prob(h_i = 1)$.
\end{definition}

\begin{definition}[Motifs]
\label{def:motifs}
Two composition rules are used throughout:
\emph{series} $\rho_{\mathrm{ser}} = \bigwedge_{i} h_i$ over a linear handoff
chain, and \emph{quorum} $\rho_{q/m} = \Ind{\sum_i h_i \geq q}$ over $m$ workers
feeding a deterministic $q$-of-$m$ aggregator. The parallel motif is the
$q = 1$ case: two independent branches feed a deterministic merge and
$\rho_{\mathrm{par}} = \Ind{\sum_i h_i \geq 1}$, so the graph succeeds when at
least one branch passes and the merge is valid. Parallel is therefore a
redundancy rule, not a conjunction; the distinction matters because
\cref{cor:joint-failure} shows the two rules put positive dependence on
\emph{opposite} sides of the safety argument. Deterministic merge and aggregator
nodes carry no contract and are excluded from dependence estimation.
\end{definition}

\begin{definition}[Co-failure table and sharing condition]
\label{def:cofailure}
For nodes $i,j$ with failure indicators $F_i = 1-h_i$, the \emph{co-failure
table} over $n$ missions has cells
$n_{11} = \sum \Ind{F_i F_j = 1}$, $n_{10} = \sum \Ind{F_i(1-F_j) = 1}$,
$n_{01} = \sum \Ind{(1-F_i)F_j = 1}$, $n_{00} = \sum \Ind{(1-F_i)(1-F_j) = 1}$,
with proportions $p_{ab} = n_{ab}/n$ and marginal failure rates
$p_i = p_{11}+p_{10}$, $p_j = p_{11}+p_{01}$. The \emph{sharing condition} of
the pair is \texttt{same\_model} when both nodes run the same model weights,
\texttt{same\_vendor} when they run different models from one vendor, and
\texttt{different\_vendor} otherwise.
\end{definition}

\begin{definition}[Dependence functionals]
\label{def:dependence}
On a co-failure table we use the \emph{overlap}
$J = n_{11}/(n_{11}+n_{10}+n_{01})$, the binary Kendall
$\tau_a = 2(p_{11}p_{00} - p_{10}p_{01})$, the phi coefficient
$\phi = (p_{11}p_{00}-p_{10}p_{01})/\sqrt{p_i(1-p_i)p_j(1-p_j)}$, the log odds
ratio, and Yule's $Q$. We call $J$, $\tau_a$, $\phi$ \emph{marginal-sensitive}
and the last two \emph{marginal-free}: the latter are invariant under row and
column scaling of the table, the former are not. \Cref{sec:e1-marginal} shows
the distinction is decisive rather than cosmetic.
\end{definition}

\section{What breaks when independence fails}
\label{sec:independence-fails}

\Cref{def:naive-bound} certifies a composed system by multiplying component
reliabilities. That step is licensed by C5 (\cref{def:c5}) and by nothing else.
This section establishes three things: that the multiplicative bound is not
conservative when C5 fails, that the assumption-free alternative is too weak to
be useful, and that a certificate built on a \emph{fitted} dependence model can
lose coverage of the truth entirely, and that collecting more data does not
repair it.

\subsection{The bound is not conservative}

\begin{proposition}[Signed compositional gap]
\label{prop:gap}
Let $\graph$ be a series motif over components with hard verdicts
$h_1,\dots,h_m$ and $\Yg = \bigwedge_i h_i$. Then
\begin{equation}
  \Prob(\Yg = 1) - \prod_{i=1}^{m} p_i
  \;=\; \Cov\Bigl(h_1, \prod_{i=2}^{m} h_i\Bigr)
        + \sum_{k=2}^{m-1} \Bigl(\prod_{i=1}^{k-1} p_i\Bigr)
          \Cov\Bigl(h_k, \prod_{i=k+1}^{m} h_i\Bigr).
\end{equation}
In particular for $m = 2$ the gap is exactly $\Cov(h_1,h_2)$, so under
positive dependence the independence product \textbf{understates} series
reliability --- and correspondingly \emph{overstates} the probability that the
series fails at all. The unsafe direction is a different estimand, stated
separately in \cref{cor:joint-failure}.
\end{proposition}

\begin{proof}
See \cref{app:gap}.
\end{proof}

For a \emph{series} system, then, the independence product is genuinely
conservative and dependence only helps. The unsafe case is redundancy, and it
is a distinct statement:

\begin{corollary}[Redundant joint failure is understated]
\label{cor:joint-failure}
Let $F_i = 1-h_i$. For two components,
$\Prob(F_1 = F_2 = 1) = (1-p_1)(1-p_2) + \Cov(F_1,F_2)$ and
$\Cov(F_1,F_2) = \Cov(h_1,h_2)$. Under positive dependence the probability
that \emph{both} redundant paths fail therefore exceeds the independence
product, so a redundant design is over-credited exactly when its components are
positively associated.
\end{corollary}

\begin{proof}
See \cref{app:joint-failure}.
\end{proof}

This is the direction that matters operationally: a quorum or a
reviewer--writer pair is deployed precisely so that
$\Prob(\text{all redundant paths fail})$ is small. \Cref{sec:e1} measures
$\phi$ between $0.58$ and $0.92$; at those magnitudes the independent
calculation is not close.

\begin{example}
\label{ex:quorum}
Take a 2-of-3 quorum whose workers each satisfy their contract with $p = 0.61$,
the \texttt{same\_model} rate measured in \cref{sec:e1}. Under independence the
quorum succeeds with probability $3p^2(1-p) + p^3 = 0.6623$ and therefore fails
with probability $0.3377$. An exchangeable joint consistent with the same
marginals and the measured pairwise association puts quorum failure near
$0.37$, a relative increase of about $9.6\%$ in the quantity the redundancy was
bought to reduce. We stress that marginals and $\phi$ alone do not determine a
2-of-3 failure probability --- the triple moment is free --- so this is an
illustration of magnitude under one consistent joint, not a bound.
\Cref{sec:lp} is the machinery for the bound.
\end{example}

\subsection{The assumption-free alternative is too weak}

Dropping C5 entirely leaves the Fr\'echet--Hoeffding bounds, which hold for every
joint law with the given marginals.

\begin{theorem}[Fr\'echet--Hoeffding sandwich for all-success]
\label{thm:frechet}
For any joint law of $(h_1,\dots,h_m)$ with marginals $p_1,\dots,p_m$,
\begin{equation}
  \max\Bigl(0,\; \sum_{i=1}^{m} p_i - (m-1)\Bigr)
  \;\leq\; \Prob\Bigl(\bigwedge_i h_i = 1\Bigr)
  \;\leq\; \min_i p_i,
\end{equation}
and both bounds are attained, so neither can be improved without further
information.
\end{theorem}

\begin{proof}
See \cref{app:frechet}.
\end{proof}

\begin{corollary}[Vacuity of the assumption-free floor]
\label{cor:vacuous}
The lower bound of \cref{thm:frechet} is $0$ whenever
$\sum_i p_i \leq m-1$, i.e.\ whenever the mean component reliability is at most
$1 - 1/m$. For $m = 4$ components each at $p = 0.75$ the certified floor is
exactly $0$.
\end{corollary}

\Cref{cor:vacuous} is the practical bind. C5 gives a number that is wrong in the
unsafe direction; dropping C5 gives a number that is right and useless. The
moment-set certificate of \cref{sec:lp} occupies the space between them by
constraining the joint with measured co-execution moments rather than with an
independence assumption, and \cref{sec:e2} measures what that buys: a floor of
$0.4116$ where the pairwise-only certificate gives $0.2455$.

\subsection{Fitting a dependence model does not rescue the situation}

A natural response is to fit a parametric dependence model --- a one-factor
Gaussian copula, say --- estimate its parameters, and report a confidence bound
on the fitted model's all-success functional. This fails, and it fails in a way
that gets worse with more data.

\begin{definition}[Model functional and identified set]
\label{def:identified}
Fix a moment family $\mathcal{J}$ and let $\mathcal{M}(\hat\mu)$ be the set of
joint laws on $\{0,1\}^m$ whose $\mathcal{J}$-moments equal $\hat\mu$. The
\emph{identified set} for all-success is
$\bigl[\underline{R}(\hat\mu),\, \overline{R}(\hat\mu)\bigr]$ where
$\underline{R} = \inf_{Q \in \mathcal{M}} Q(\bigwedge_i h_i = 1)$ and
$\overline{R} = \sup$. A parametric family $\mathcal{F} \subset \mathcal{M}$
induces a \emph{model functional} $R_{\mathcal{F}}(\hat\mu)$, the all-success
probability of the fitted member of $\mathcal{F}$.
\end{definition}

\begin{theorem}[Coverage collapse of a model-based floor]
\label{thm:collapse}
Let $Q^{\star}$ be a true joint law with $\mathcal{J}$-moments $\mu^{\star}$ and
true all-success probability $R^{\star} = Q^{\star}(\bigwedge_i h_i = 1)$.
Suppose the parametric family $\mathcal{F}$ is \emph{misspecified at $\mu^\star$}
in the sense that
\begin{equation}
  \Delta \;:=\; R_{\mathcal{F}}(\mu^{\star}) - R^{\star} \;>\; 0 .
  \label{eq:idgap}
\end{equation}
Let $\hat{L}_n$ be any bootstrap lower confidence bound for
$R_{\mathcal{F}}(\mu^\star)$ satisfying
$\hat{L}_n = R_{\mathcal{F}}(\mu^{\star}) - O_p(n^{-1/2})$. Then
\begin{equation}
  \Prob\bigl(\hat{L}_n \leq R^{\star}\bigr) \;\longrightarrow\; 0
  \qquad \text{as } n \to \infty .
\end{equation}
That is, coverage of the \emph{true} all-success probability tends to zero.
More data does not repair the certificate: the haircut shrinks while the target
does not move. We claim the limit and not monotonicity in $n$, which tightness
alone does not entail.
\end{theorem}

\begin{proof}
See \cref{app:collapse}.
\end{proof}

The mechanism is a mismatch of orders. The identification gap $\Delta$ in
\cref{eq:idgap} is $O(1)$ --- it is a property of the family, not of the sample
--- while the bootstrap haircut shrinks like $n^{-1/2}$. Past some finite $n$
the entire interval sits above $R^{\star}$ and never returns. The bootstrap is
not malfunctioning; it is covering the estimand it was asked about, which is the
fitted model's functional and not the truth.

\begin{example}[An explicit witness]
\label{ex:witness}
Take the equicorrelated Gaussian-copula joint at $p = 0.6$, $\lambda = 0.8$
(a factor \emph{loading}, so the latent correlation is $\lambda^2 = 0.64$),
$m = 3$. By Gauss--Legendre quadrature over the common factor its marginals are
$0.600000$, its pairwise co-success moments are $0.465237$, and its all-success
probability is $0.392144$ --- so the one-factor family is \emph{correctly
specified} for it and $\Delta = 0$. Every constant quoted for this law in this
paper comes from that one deterministic routine, not from simulation. It is therefore
\textbf{not} itself a witness for \cref{thm:collapse}, a point worth stating
because it is the natural law to reach for and it does not work.

The witness is instead the joint attaining the lower end of the LP-identified
set. Over the pairwise moment family that set is $[0.330475, 0.465237]$, and by
\cref{thm:lp-sharp} the lower endpoint is attained. For this exchangeable case
it has the closed form $2q - p$ in the pairwise moment $q$ and marginal $p$, and
the attaining law $Q^\star$ puts mass on five cells only:
\begin{align*}
  Q^\star(1,1,1) &= 0.330475, &
  Q^\star(0,0,0) &= 0.265237, \\
  Q^\star(0,1,1) = Q^\star(1,0,1) = Q^\star(1,1,0) &= 0.134763, &
  \text{all other cells} &= 0 .
\end{align*}
By construction $Q^\star$ has \emph{exactly} the same marginals and
\emph{exactly} the same pairwise co-success moments as the Gaussian joint: the
two are indistinguishable to any procedure that sees only those moments. Yet
\[
  \Delta \;=\; R_{\mathcal{F}}(\mu^\star) - R^\star
       \;=\; 0.392144 - 0.330475 \;=\; 0.061670 \;>\; 0,
\]
so \cref{thm:collapse} applies to $Q^\star$. \Cref{sec:e3} samples from
$Q^\star$ and measures the collapse, and runs the Gaussian joint as a
correctly-specified control in which no collapse should occur --- and does not.
\end{example}

\begin{corollary}[Why the copula-agnostic floor is the right target]
\label{cor:target}
The LP floor $\underline{R}(\hat\mu)$ of \cref{def:identified} satisfies
$\underline{R}(\mu^{\star}) \leq R^{\star}$ by construction for every
$Q^{\star} \in \mathcal{M}(\mu^{\star})$, with no parametric assumption. A
confidence bound targeting $\underline{R}$ therefore cannot suffer the failure of
\cref{thm:collapse}; its conservatism is the price of that immunity, and
\cref{sec:lp} shows the price falls as the moment family grows.
\end{corollary}

\section{A tiered certificate}
\label{sec:tiers}

\Cref{sec:independence-fails} leaves three unusable options: a multiplicative
bound that is unsafe on the failure side, an assumption-free bound that is often
exactly zero, and a fitted model whose coverage collapses. What survives depends
on \emph{what was actually measured}. A certificate should therefore not be a
single number but a selection among tiers, each stating the evidence it requires
and the assumptions it carries.

\begin{definition}[Certification scope]
\label{def:scope}
A \emph{certificate} is a triple $(\hat{L}, \mathcal{A}, \mathcal{S})$ where
$\hat{L}$ is a lower confidence bound on all-success reliability at level
$1-\eta_{\mathrm{conf}}$, $\mathcal{A}$ is the set of assumptions it rests on,
and $\mathcal{S}$ is its scope: the mission distribution, the model versions, and
the topology under which it was obtained. A certificate makes no claim outside
$\mathcal{S}$; in particular it does not transfer across mission mixes, model
upgrades, or topology changes.
\end{definition}

\begin{definition}[The three tiers]
\label{def:tiers}
Let $\mathbf{h} \in \{0,1\}^{m \times n}$ be a pass matrix over $m$ stages and
$n$ missions.
\begin{description}[leftmargin=3.6em,style=nextline,itemsep=0.2em]
  \item[Tier 0 --- observed]
    Available only when the composition was executed end-to-end, i.e.\ every
    stage was scored on the same missions, so $\Yg$ is directly observed.
    $\hat{L}_0$ is the exact Clopper--Pearson lower bound on
    $\Prob(\Yg = 1)$ from $\sum_r \Ind{\Yg^{(r)} = 1}$ successes in $n$ trials.
    Assumptions: i.i.d.\ missions. No copula, no model.
  \item[Tier 1 --- copula-agnostic]
    Available from per-stage and co-execution data alone, when the composed
    pipeline was \emph{not} run end-to-end. $\hat{L}_1$ is the LP minimum of
    \cref{sec:lp} over the Bonferroni--Clopper--Pearson moment box.
    Assumptions: i.i.d.\ missions. No copula, no model.
  \item[Tier 2 --- model]
    The finite-sample Gaussian one-factor floor, via the
    \citet{slepian1962} corner. Assumptions: i.i.d.\ missions
    \emph{and} correct specification of the one-factor family.
\end{description}
\end{definition}

\begin{theorem}[Tier-0 validity]
\label{thm:tier0}
If missions are i.i.d.\ and $\Yg$ is observed on each, then the Clopper--Pearson
lower bound $\hat{L}_0$ satisfies
$\Prob\bigl(\hat{L}_0 \leq \Prob(\Yg=1)\bigr) \geq 1-\eta_{\mathrm{conf}}$
exactly, for every $n$ and every true reliability.
\end{theorem}

\begin{proof}
See \cref{app:tier0}.
\end{proof}

\begin{theorem}[Tier-1 validity]
\label{thm:tier1}
Let $\mathcal{J}$ be a moment family of size $J$, let
$\hat\mu = (\hat\mu_S)_{S \in \mathcal{J}}$ be the empirical moments, and let
$B(\hat\mu)$ be the box formed by two-sided Clopper--Pearson intervals for each
$\mu_S$ at level $\eta_{\mathrm{conf}}/(2J)$ per tail. Define
\begin{equation}
  \hat{L}_1 \;=\; \min\Bigl\{\, Q\bigl(\textstyle\bigwedge_i h_i = 1\bigr)
  \;:\; Q \in \Delta(\{0,1\}^m),\;
  \bigl(Q(\textstyle\bigwedge_{i \in S} h_i = 1)\bigr)_{S \in \mathcal{J}}
  \in B(\hat\mu) \,\Bigr\}.
  \label{eq:tier1}
\end{equation}
Then for i.i.d.\ missions and any true law $Q^{\star}$,
$\Prob\bigl(\hat{L}_1 \leq Q^{\star}(\bigwedge_i h_i = 1)\bigr)
 \geq 1-\eta_{\mathrm{conf}}$.
\end{theorem}

\begin{proof}
See \cref{app:tier1}.
\end{proof}

\Cref{thm:tier1} is the reason Tier 1 is safe where Tier 2 is not: it is a
minimum over \emph{every} joint law consistent with the box, so the true law is
one of the candidates whenever the box covers, and the union bound makes the box
cover with the stated probability. Nothing is assumed about the dependence
structure.

\begin{proposition}[Fail-safe tier selection]
\label{prop:failsafe}
Let the certificate select Tier 0 only on an explicit assertion that the
composition was executed end-to-end, and Tier 1 by default otherwise. Then no
the two failure directions are not symmetric. Tier 1 requires only i.i.d.\
missions (\cref{thm:tier1}), which Tier 0 requires as well, so the assumption
sets are nested. Consequently:
\begin{enumerate}[label=(\roman*),leftmargin=2.2em,itemsep=0.1em]
  \item \textbf{Omission is safe.} Failing to set the flag when it was
    warranted yields $\hat{L}_1$, which is valid --- merely looser than the
    Tier-0 bound that was available.
  \item \textbf{Commission is not.} Setting the flag when it was \emph{not}
    warranted yields $\hat{L}_0$, whose validity requires an assumption that
    does not hold; that guarantee is unsound.
\end{enumerate}
The design content of the proposition is therefore that the unsound case
requires a deliberate assertion, while the default --- and every omission ---
lands in the safe case. It is \textbf{not} a claim that no setting of the flag
can produce an unsound certificate; case (ii) is exactly such a setting, and no
inspection of the pass matrix can detect it.
\end{proposition}

\begin{proof}
See \cref{app:failsafe}.
\end{proof}

\begin{remark}[No pointwise ordering]
\label{rem:no-ordering}
We deliberately do \emph{not} claim $\hat{L}_1 \leq \hat{L}_0$ pointwise.
Both bound the same true reliability from below on their covering events, but
neither dominates the other as a random variable: an unlucky low all-success
count depresses $\hat{L}_0$ while leaving the marginal and pairwise moments ---
and hence $\hat{L}_1$ --- comparatively high. The fail-safety in
\cref{prop:failsafe} comes from the nesting of the assumption sets, not from an
ordering of the two numbers.
\end{remark}

\Cref{prop:failsafe} states a design constraint rather than a mathematical
discovery, and it is stated because the alternative is a silent failure. Tier 0's
validity requires that the joint outcome was observed --- a fact about how the
data were produced, which no inspection of the pass matrix can confirm or refute.
A default of Tier 0 would let a forgotten flag convert an extrapolation into an
over-claim with no detectable symptom. Defaulting to Tier 1 makes the failure
mode a needlessly weak certificate instead.

\begin{remark}[Tier 2 is a diagnostic]
\label{rem:tier2}
Tier 2 is reported and never certified. By \cref{thm:collapse} its coverage of
the true reliability tends to zero under a misspecification that marginal and
pairwise data cannot detect, and \cref{sec:e3} measures that collapse on the
witness of \cref{ex:witness}. Reporting it is useful --- a large gap between
Tier 1 and Tier 2 signals that the moment family is leaving information on the
table --- but printing it as a guarantee would reintroduce exactly the failure
this paper documents.
\end{remark}

\section{Moment-set certification}
\label{sec:lp}

Tier 1 turns on a linear program. This section states it, proves it sharp,
and characterises how it behaves as the moment family grows --- which is the
knob a practitioner actually turns, and the one \cref{sec:e2} measures.

\subsection{The program}

\begin{definition}[Cell representation]
\label{def:cells}
A joint law on $\{0,1\}^m$ is a vector $x \in \R^{2^m}$ with $x_c \geq 0$ and
$\sum_c x_c = 1$, where cell $c$ ranges over binary patterns. For
$S \subseteq \{1,\dots,m\}$ let $a_S \in \{0,1\}^{2^m}$ be the indicator row
$(a_S)_c = \Ind{c_i = 1 \;\forall i \in S}$, so $a_S^{\top} x$ is the
probability that every stage in $S$ succeeds. All-success is $a_{\{1..m\}}^\top x$.
\end{definition}

\begin{definition}[Moment-set LP]
\label{def:momentlp}
Given a moment family $\mathcal{J}$ and values $(\nu_S)_{S \in \mathcal{J}}$,
\begin{equation}
  \underline{R} \;=\; \min_{x}\; a_{\{1..m\}}^{\top} x
  \quad\text{s.t.}\quad
  a_S^{\top} x = \nu_S \;\; \forall S \in \mathcal{J}, \quad
  \mathbf{1}^{\top} x = 1, \quad x \geq 0,
  \label{eq:lp}
\end{equation}
and $\overline{R}$ the corresponding maximum. With
$\mathcal{J} = \{\text{singletons}\} \cup \{\text{pairs}\}$ this is the pairwise
program; adding triples gives the $J = 14$ program of \cref{sec:e2} at $m=4$.
For a box constraint $\nu \in B$ rather than point values, replace each equality
by the pair of inequalities defining $B$.
\end{definition}

\begin{theorem}[Soundness and sharpness]
\label{thm:lp-sharp}
Let $Q^{\star}$ be any joint law on $\{0,1\}^m$ whose $\mathcal{J}$-moments equal
$\nu$. Then
$\underline{R} \leq Q^{\star}(\bigwedge_i h_i = 1) \leq \overline{R}$, and both
bounds are \emph{attained}: there exist laws $Q_-, Q_+$ with the same
$\mathcal{J}$-moments achieving $\underline{R}$ and $\overline{R}$ respectively.
Consequently $[\underline{R}, \overline{R}]$ is exactly the identified set of
\cref{def:identified}, and no bound using only the moments in $\mathcal{J}$ can
be tighter.
\end{theorem}

\begin{proof}
See \cref{app:lp-sharp}.
\end{proof}

Sharpness matters for interpretation. A conservative-but-loose bound invites the
response ``the truth is surely much better than that''; \cref{thm:lp-sharp} says
that for the information supplied there is a consistent world in which the truth
is exactly $\underline{R}$. Tightening requires more moments, not more optimism.

\begin{proposition}[Monotonicity in the moment family]
\label{prop:monotone}
If $\mathcal{J} \subseteq \mathcal{J}'$ and the moment values agree on
$\mathcal{J}$, then
$\underline{R}(\mathcal{J}) \leq \underline{R}(\mathcal{J}')$ and
$\overline{R}(\mathcal{J}) \geq \overline{R}(\mathcal{J}')$. The identified set
shrinks weakly as the family grows.
\end{proposition}

\begin{proof}
See \cref{app:monotone}.
\end{proof}

\Cref{prop:monotone} holds at the level of \emph{point} moments. It does not
automatically survive the finite-sample box, and the reason is a Bonferroni
budget rather than anything deep.

\begin{proposition}[Bonferroni allocation and monotonicity]
\label{prop:bonferroni}
Fix $\eta_{\mathrm{conf}}$ and let the box $B$ be built from $J = |\mathcal{J}|$
Clopper--Pearson intervals.
\begin{enumerate}[label=(\roman*),leftmargin=2.2em]
  \item \emph{Used-set allocation}: spending $\eta_{\mathrm{conf}}/(2J)$ per tail
    gives the tightest box for that family, but enlarging $\mathcal{J}$ widens
    every interval, so \cref{prop:monotone} can fail on the certified floor.
  \item \emph{Pre-allocated allocation}: fixing a budget family
    $\mathcal{J}_{\max} \supseteq \mathcal{J}$ and spending
    $\eta_{\mathrm{conf}}/(2|\mathcal{J}_{\max}|)$ per tail regardless of which
    moments are constrained makes every interval width independent of
    $|\mathcal{J}|$. Then enriching $\mathcal{J}$ only adds constraints, and the
    certified floor is monotone \emph{by construction}.
\end{enumerate}
Both are valid at level $1-\eta_{\mathrm{conf}}$.
\end{proposition}

\begin{proof}
See \cref{app:bonferroni}.
\end{proof}

The trade-off is measurable rather than theoretical. At $m=4$, $n=717$ in
\cref{sec:e2}, pre-allocation costs $0.0098$ of certified floor at $J=10$
($0.2357$ against $0.2455$) and costs nothing at $J=14$, where the budget family
is the family in use. We report both, because reporting only the used-set number
would present an empirically-monotone-so-far quantity as a guaranteed one.

\begin{remark}[Cost]
\label{rem:lp-cost}
\Cref{eq:lp} has $2^m$ variables and $|\mathcal{J}|+1$ constraints, so it is
exact and cheap for the small $m$ that agent motifs actually use ($m \leq 8$
covers every topology in \cref{sec:evaluation}) and exponential beyond. For large
$m$ the practical route is to certify sub-compositions and compose the
certificates, which costs tightness; we do not pursue it here because no motif in
this evaluation needs it.
\end{remark}

\section{Anytime-valid certification}
\label{sec:anytime}

Every certificate so far is fixed-$n$: valid at a sample size chosen before the
data arrive. Deployed agent systems are not monitored that way. A team watches a
reliability dashboard, and when the number looks good enough they ship. That is
optional stopping, and it inflates type-I error without bound for a fixed-$n$
interval. This section gives a certificate that survives it, at a cost this
section also states. The machinery is that of game-theoretic probability and
testing by betting \citep{ville1939,shafer2019,shafer2021,ramdas2023,%
waudbysmith2024,grunwald2024}; the application to composed agent reliability is
what is new here.

\subsection{The e-process}

\begin{definition}[Betting e-process for graph reliability]
\label{def:eprocess}
Let $y_r \in \{0,1\}$ be the graph outcome of mission $r$ and let
$\mathcal{F}_r$ be the $\sigma$-field generated by $y_1,\dots,y_r$. Test the
\emph{sequential} null
\begin{equation}
  H_0 : \; \E[y_r \mid \mathcal{F}_{r-1}] \leq p_0
  \quad \text{almost surely, for every } r \geq 1 .
  \label{eq:seq-null}
\end{equation}
This is the null the betting construction requires, and it is strictly stronger
than the marginal statement $\Prob(\Yg = 1) \leq p_0$. The distinction is not
pedantic: a stream whose every mission has marginal success rate $p_0$ but
whose outcomes are perfectly dependent satisfies the marginal statement,
violates \cref{eq:seq-null}, and drives the wealth process out of the
supermartingale class entirely.\footnote{Take $p_0 = \tfrac12$,
$\lambda_r \equiv 1$, and $y_r \equiv Z$ for every $r$ with
$Z \sim \mathrm{Bernoulli}(\tfrac12)$. Every marginal satisfies
$\Prob(y_r = 1) = \tfrac12 \leq p_0$, yet $\E[E_2] = \tfrac54 > 1$ and
$E_8 = (3/2)^8 > 20 = 1/\alpha$ on the event $\{Z = 1\}$, so the certificate
fires with probability $\tfrac12$ rather than $\alpha = 0.05$.}
Let $\lambda_r$ be a \emph{predictable} betting fraction, i.e.\ $\lambda_r$ is
$\mathcal{F}_{r-1}$-measurable, with
$\lambda_r \in [0, 1/p_0)$. The wealth process is
\begin{equation}
  E_R \;=\; \prod_{r \leq R} \bigl(1 + \lambda_r (y_r - p_0)\bigr),
  \qquad E_0 = 1 .
  \label{eq:eprocess}
\end{equation}
The certificate \emph{fires} at the first $R$ with $E_R \geq 1/\alpha$.
\end{definition}

\begin{lemma}[Supermartingale property]
\label{lem:supermartingale}
Under $H_0$, $(E_R)_{R \geq 0}$ is a non-negative supermartingale with
$\E[E_R] \leq 1$ for all $R$.
\end{lemma}

\begin{proof}
See \cref{app:supermartingale}.
\end{proof}

\begin{theorem}[Anytime validity; \citealp{ville1939}]
\label{thm:ville}
Under $H_0$, for any stopping time $T$ --- including data-dependent and
unbounded ones ---
\begin{equation}
  \Prob\Bigl(\sup_{R \geq 1} E_R \geq 1/\alpha\Bigr) \;\leq\; \alpha .
\end{equation}
Consequently a team may inspect $E_R$ after every mission, stop whenever they
choose, and the probability of ever issuing a false certificate is at most
$\alpha$. The same device underlies time-uniform confidence sequences
\citep{robbins1970,howard2021}.
\end{theorem}

\begin{proof}
See \cref{app:ville}.
\end{proof}

Two properties matter for this paper specifically. First, $H_0$ constrains only
the scalar conditional mean of $y_r$, and constrains it one mission at a time,
so \textbf{no independence assumption is required} --- not across missions, not
across components, not across shared-model shocks. The dependence may take any
form whatever, provided it does not lift the conditional mean of the next
outcome above $p_0$; and a correlated shock that \emph{degrades} reliability is
exactly a shock that does not. This is the relevant direction: the sequential
certificate is immune to the exact failure mode
\cref{sec:independence-fails} documents for the composition bound.
Second, the guarantee is on the whole path, which is why certification must latch:
once $E_R$ has crossed, the event ``$\sup_R E_R \geq 1/\alpha$'' has occurred and
no later evidence undoes it. \Cref{sec:e4} verifies the shipped implementation
latches, driving a certified process through $340$ consecutive failures.

\subsection{Relation to the SPRT}

The fixed-stopping predecessor is \citet{wald1945}'s sequential test.

\begin{proposition}[Exact SPRT recovery]
\label{prop:sprt}
Fix an alternative $p_1 > p_0$ and set
$\lambda^{\star} = (p_1 - p_0)/\bigl(p_0(1-p_0)\bigr)$. Then for
$y \in \{0,1\}$,
\begin{equation}
  1 + \lambda^{\star}(y - p_0)
  \;=\; \Bigl(\frac{p_1}{p_0}\Bigr)^{y}
        \Bigl(\frac{1-p_1}{1-p_0}\Bigr)^{1-y},
\end{equation}
so \cref{eq:eprocess} is exactly the Bernoulli likelihood ratio and $E_R$ is the
SPRT statistic. The e-process therefore contains \cref{def:sprt} as the special
case of a constant, optimally-tuned bet.
\end{proposition}

\begin{proof}
See \cref{app:sprt}.
\end{proof}

\Cref{prop:sprt} places the cost of anytime validity precisely. Against a known
alternative the SPRT is recovered with no loss; the price is paid only when
$\lambda$ is mis-tuned, and it is a loss of \emph{power}, never of validity ---
\cref{thm:ville} holds for every predictable $\lambda$. \Cref{sec:e4} measures
both sides: the type-I rate stays under $\alpha$ across the whole admissible
range of $\lambda$, while the mean time to certification varies from $14.6$
missions at $\lambda = 1.2375$ to $403.9$ at $\lambda = 0.125$. A badly chosen
bet costs a factor of $28$ in detection time and nothing in soundness.

\begin{proposition}[Mixture over bets]
\label{prop:mixture}
Let $\pi$ be a prior on a finite set $\Lambda$ of predictable betting strategies
and let $E_R^{\mathrm{mix}} = \sum_{\lambda \in \Lambda} \pi(\lambda) E_R^{\lambda}$.
Then $E^{\mathrm{mix}}$ is a non-negative supermartingale under $H_0$, so
\cref{thm:ville} applies unchanged, and
\begin{equation}
  \log E_R^{\mathrm{mix}} \;\geq\;
  \max_{\lambda \in \Lambda} \log E_R^{\lambda} - \log\frac{1}{\pi(\lambda)} ,
\end{equation}
a pathwise log-regret guarantee against the best bet in hindsight.
\end{proposition}

\begin{proof}
See \cref{app:mixture}.
\end{proof}

\Cref{prop:mixture} is what makes the method usable when $p_1$ is unknown, which
is the normal case: the mixture pays an additive $\log(1/\pi(\lambda))$ against
the best fixed bet and retains anytime validity exactly.

\section{Runtime enforcement}
\label{sec:enforcement}

A certificate is a claim about a system that behaves as contracted. Producing
that system requires intercepting agent actions before they take effect, which
is an engineering problem with a small number of solutions and sharply different
guarantees between them.

\subsection{Three interception planes}

An agent action can be intercepted at exactly three points, and where you
intercept determines what you can enforce.

\begin{description}[leftmargin=2.4em,style=nextline,itemsep=0.25em]
  \item[P1 --- in-process.]
    A hook inside the orchestration framework, invoked around each tool call.
    Sees the parsed call, the agent identity, and the framework's state; can
    deny by the framework's own veto convention. Requires framework support and
    a per-framework adapter.
  \item[P2 --- tool protocol.]
    An interposer on the tool-invocation protocol between the model host and the
    tool server. Sees every tool call and result as structured messages, and is
    framework-independent. Cannot see reasoning that never becomes a tool call.
  \item[P3 --- transport.]
    A proxy on the HTTP path to the model provider. Sees requests and responses
    but must reconstruct semantic structure from wire format, and cannot
    attribute an action to an agent in a multi-agent process without additional
    context.
\end{description}

The planes are complementary rather than ranked. P1 gives the richest state and
the weakest portability; P3 gives the reverse; P2 sits between and is the only
plane that is both structured and framework-neutral.

\subsection{Denial semantics}

\begin{definition}[Pre-action and post-action denial]
\label{def:denial}
A \emph{pre-action} denial prevents the action from executing. A
\emph{post-action} denial permits execution and withholds the result. Only
pre-action denial enforces a governance constraint on an irreversible operation;
post-action denial is a redaction mechanism.
\end{definition}

\Cref{def:denial} is the distinction that decides whether a deployment is
enforced or merely monitored, and it is not always available. A plane that
observes a tool result has already permitted the tool to run. \Cref{sec:impl}
records, per surface, which of the two is achievable, because a framework that
reports ``enforcement enabled'' while only offering post-action denial is making
a claim its users will misread.

\subsection{Failure policy}

\begin{definition}[Fail-closed and fail-open]
\label{def:failmode}
When the enforcement path itself errors --- a malformed contract, an evaluator
exception, an unreachable policy store --- a \emph{fail-closed} policy denies the
action and a \emph{fail-open} policy permits it. A third state is required:
\emph{unevaluated}, in which no verdict was reached and the action is neither
credited as compliant nor recorded as violating.
\end{definition}

The third state matters for the measurement in \cref{sec:evaluation}. Counting
an unevaluated action as compliant inflates every reliability estimate; counting
it as a violation deflates them. Both are silent. The implementation carries
\texttt{evaluated} as an explicit field on every verdict so that unevaluated
actions are excluded from the denominator rather than assigned to a side.

\subsection{Enforcement and certification are different claims}

Enforcement constrains what an agent does. Certification bounds how often the
constrained system satisfies its contract. Neither implies the other: a perfectly
enforced hard constraint still permits soft-constraint drift
(\cref{def:pdk}), and a high certified floor says nothing about whether any
individual action was blocked. \Cref{sec:evaluation} measures the second; this
section describes the machinery that makes the first possible, and the two are
reported separately throughout.

\section{Implementation}
\label{sec:impl}

Everything in \cref{sec:prelim}--\cref{sec:anytime} that is computable is
implemented in \textsc{AgentAssert}, an open-source Python library
(AGPL-3.0). We describe it at the level a reader needs to reproduce
\cref{sec:evaluation} or to check a claim against code; \cref{app:repro} gives
the artifact details.

\subsection{Structure}

The library separates the three concerns this paper keeps distinct. Contract
specification lives in \texttt{evaluator/}, \texttt{dsl/}, and
\texttt{models.py} (\cref{def:contract,def:scores}); measurement in
\texttt{metrics/} and \texttt{dependence/}
(\cref{def:drift,def:theta,def:dependence}); and certification in
\texttt{certification/}, which carries one module per tier of
\cref{def:tiers} --- \texttt{observed\_floor}, \texttt{lp\_bound},
\texttt{slepian\_floor} (diagnostic only), \texttt{eprocess}, \texttt{sprt} ---
assembled by \texttt{certificate} under \cref{prop:failsafe}. Runtime
enforcement lives in \texttt{gateway/} and \texttt{enforce/}, with the
interception surfaces of \cref{sec:enforcement} as thin adapters over them.

\subsection{One policy path, several surfaces}

Enforcement across agent frameworks with different hook conventions would
ordinarily mean one policy implementation per framework and one opportunity per
framework to diverge. Instead a single bridge exposes a pre-action decision and
a post-action outcome, and each shim translates those into its host's veto
convention --- returning false from a before-call hook, declining to invoke the
continuation, not awaiting the next middleware, or raising from a pre-hook. The
shims match each host's hook shapes structurally and import nothing from them,
so the library depends on no agent framework.

Two consequences matter for this paper. Policy is evaluated in one place, so a
contract enforced through a tool-protocol interposer and the same contract
enforced in-process yield the same verdict on the same state --- which is what
makes cross-surface measurement comparable. And because
\cref{def:denial} distinguishes pre-action from post-action denial while not
every host offers the former, the library publishes a per-surface capability
matrix recording where enforcement is impossible. The negative entries are the
point: a library whose documentation implies uniform coverage will be deployed
against hosts where the guarantee does not hold.

\subsection{State representation}

Contracts are evaluated against a flat, dotted-key state
(\texttt{output.pii\_detected} rather than a nested object). The convention is
not cosmetic. An earlier version returned nested dictionaries from one adapter
and flat keys from another; a constraint written against
\texttt{output.pii\_detected} then resolved correctly under enforcement and
silently evaluated to false under measurement, so the same contract yielded two
different compliance rates depending on which path produced the state. Flattening
at every adapter boundary removes the divergence by construction.

\subsection{Verification}

The library carries 1697 tests at 93\% statement coverage. Beyond
conventional unit tests, three properties from this paper are pinned by tests
because they are the ones an implementation silently gets wrong: that a certified
e-process never un-certifies after wealth collapses (\cref{sec:e4}); that the
default certificate tier is Tier 1 unless end-to-end execution is explicitly
asserted (\cref{prop:failsafe}); and that a contract which could not be evaluated
is recorded as unevaluated rather than counted as either outcome
(\cref{def:failmode}).

A paper-to-code parity test checks the formulas of \cref{sec:prelim} against
their implementations, so a change to either that breaks correspondence fails the
suite. That test exists because two defects found during preparation of this work
--- a mishandled constraint state across the proxy and hook surfaces, and an
incorrect divergence term in a drift computation --- produced confidently wrong
numbers rather than errors.

\section{Evaluation}
\label{sec:evaluation}

We answer six questions with six experiments. E1 measures whether model sharing
induces correlated contract failures across a pipeline
(\cref{sec:e1}). E2 measures how much a moment-set certificate lifts the
reliability floor over the Fr\'echet bound (\cref{sec:e2}). E3 measures how fast
independence-assuming coverage collapses as dependence grows (\cref{sec:e3}).
E4 measures the empirical type-I error of the anytime-valid certificate
(\cref{sec:e4}). E5 measures contract behavior on frontier models under live API
conditions (\cref{sec:e5}). E6 is an ablation over the clustering and
serial-dependence assumptions the analysis rests on (\cref{sec:e6}).
\Cref{sec:e-carried} re-presents the v1 single-agent evidence this work builds
on, and \cref{sec:e-summary} states what the six experiments jointly support and
what they do not.

\subsection{Setup, preregistration, and scoring}
\label{sec:e1-setup}

\paragraph{Preregistration.}
The confirmatory hypotheses, conditions, sample sizes, sampling parameters,
primary estimator, stopping rule, and falsification criteria for E1 were
committed to the repository before any confirmatory outcome was generated
\citep{nosek2018}, and
the commit is git-timestamped. We reproduce the three registered hypotheses
verbatim:

\begin{itemize}[leftmargin=1.6em,itemsep=0.15em]
  \item \textbf{H1 (dependence exists under sharing).} In the \texttt{same\_model}
    condition, co-failure dependence between the two pipeline agents is
    positive: Kendall's $\tau_a > 0$ with a bootstrap CI excluding $0$.
  \item \textbf{H2 (dependence decreases with less sharing).} $\tau_a$ is weakly
    monotone decreasing across
    $\texttt{same\_model} \geq \texttt{same\_vendor} \geq \texttt{different\_vendor}$.
  \item \textbf{H3 (naive bound is anti-conservative under sharing).}
    Observed graph reliability $\Prob(\Yg=1)$ \emph{exceeds} the independence
    product $\prod_i p_i$ under positive dependence (the compositional gap is
    signed and non-zero), so a dependence-aware bound is required.
\end{itemize}

The registration further designates \texttt{series2} as the confirmatory motif
and names the three OpenRouter arms at $n = 6000$ as the primary confirmatory
arms, placing other motifs in ``future/secondary work''. We hold to that scope:
\textbf{the confirmatory evaluation is \texttt{series2} over
$3 \times 6000 = 18{,}000$ missions}. The \texttt{quorum2of3} and
\texttt{parallel2} results, and the two breadth arms of \cref{sec:e5}, are
reported as \emph{secondary and exploratory} throughout. Where they agree with
the confirmatory arms we say so; they are not evidence at the registered level,
and we have not filed a dated amendment broadening the registration.

The registration also fixed what would falsify the thesis: if the
\texttt{same\_model} $\tau_a$ CI covers $0$, or $\tau_a$ fails to decrease from
\texttt{same\_model} to \texttt{different\_vendor}, or the composition gap is
approximately zero, the claim is unsupported and the null is reported. We report
against that criterion in \cref{sec:e1-verdict}, including one registered
comparison that does \emph{not} survive.

\paragraph{Missions and scoring.}
Missions are drawn from six seeded generators spanning retail and financial
agent tasks. Every mission is scored by \emph{deterministic gold code}, never by
an LLM judge: order arithmetic, refund-policy application, and promotional-cap
enforcement on the retail side; transaction-limit checks, watchlist screening,
and mandatory-disclaimer presence on the financial side. The mission-to-task assignment is a fixed SHA-256 hash of the mission
identifier, which makes each arm's task sequence deterministic and reproducible.
It does \emph{not} make the realised task mixtures identical across arms:
identifiers embed the condition name (\cref{sec:threats}), so each arm is an
independent deterministic draw from the same generator distribution rather than
a paired design. At $n = 6000$ per confirmatory arm the resulting imbalance is
absorbed by the bootstrap; for the under-run breadth arms it is a real limit.

Each agent invocation yields a hard verdict and a soft score. The hard verdict
$h_i \in \{0,1\}$ records whether the agent satisfied every hard clause of its
contract; the soft score $\sigma_i \in [0,1]$ records graded compliance. All
dependence estimates in \cref{sec:e1} are computed on the \emph{hard} verdict,
so a ``failure'' means a definite contract violation, not a low score. Graph
success $\Yg$ is the motif's composition rule applied to the component hard
verdicts.

\paragraph{Frozen sampling.}
Every call uses \texttt{temperature} $=0.2$, \texttt{top\_p} $=1.0$, and
\texttt{max\_output\_tokens} $=160$, with client-side prompt truncation at
$3200$ characters. These values were fixed at registration and never varied.

\paragraph{Conditions.}
The manipulated variable is how much the two scored agents share. In
\texttt{same\_model} both run \texttt{mistral-small-24b}; in
\texttt{same\_vendor} the second agent is swapped to \texttt{ministral-8b}
(same vendor, smaller model); in \texttt{different\_vendor} it is swapped to
\texttt{gemma-3-12b-it}. The first agent is held at \texttt{mistral-small-24b}
throughout, so the contrast isolates the substitution.

\paragraph{Motifs.}
Three topologies are measured. \texttt{series2} is a two-agent handoff
$A \rightarrow B$ where both must satisfy their contract for $\Yg = 1$.
\texttt{parallel2} runs two branches into a deterministic merge.
\texttt{quorum2of3} runs three workers into a deterministic 2-of-3 aggregator.
The aggregator and merge nodes are deterministic code, not models, and are
excluded from dependence estimation.

\paragraph{Scale.}
The campaign is 30{,}820 scored missions across 12 arms, recorded in 13
execution logs. The counting rule matters and we state it: the
\texttt{quorum3of4} arm was executed twice, so its two logs share $212$ mission
identifiers, which we deduplicate with the later pass winning; that arm
therefore contributes $717$ missions rather than the $929$ records on disk, and
$717$ is the $n$ used throughout \cref{sec:e2}. The
primary confirmatory arms are the three \texttt{series2} conditions at
$n = 6000$ each.

\paragraph{Estimators.}
For two agents with failure indicators $F_A = 1 - h_A$ and $F_B = 1 - h_B$ we
form the $2\times2$ co-failure table with cells $n_{11}, n_{10}, n_{01}, n_{00}$
and corresponding proportions $p_{11}, p_{10}, p_{01}, p_{00}$. We report five
statistics, and the distinction between the first three and the last two carries
the analysis:
\begin{align}
  J &= \frac{n_{11}}{n_{11} + n_{10} + n_{01}}, \label{eq:jaccard}\\
  \tau_a &= 2\,(p_{11}p_{00} - p_{10}p_{01}), \label{eq:taua}\\
  \phi &= \frac{p_{11}p_{00} - p_{10}p_{01}}
              {\sqrt{p_A(1-p_A)\,p_B(1-p_B)}}, \label{eq:phi}\\
  \log \mathrm{OR} &= \log
      \frac{(n_{11}+\tfrac12)(n_{00}+\tfrac12)}
           {(n_{10}+\tfrac12)(n_{01}+\tfrac12)}, \label{eq:logor}\\
  Q &= \frac{\mathrm{OR}-1}{\mathrm{OR}+1}, \label{eq:yuleq}
\end{align}
where $p_A = p_{11}+p_{10}$ and $p_B = p_{11}+p_{01}$ are the marginal failure
rates. These are, in order, the overlap coefficient of \citet{jaccard1912}, the
binary form of \citet{kendall1938}'s $\tau$, the phi coefficient, the log odds
ratio, and \citet{yule1900}'s $Q$. \Cref{eq:jaccard} is the overlap of the two
failure \emph{sets}; the
$n_{00}$ cell is deliberately excluded so that missions neither agent failed
cannot dilute the statistic. \Cref{eq:logor} uses the Haldane--Anscombe
$+\tfrac12$ correction, which keeps the estimator finite when a cell is empty.

$J$, $\tau_a$, and $\phi$ are all sensitive to the marginals: $\tau_a$ is twice
the covariance of the two indicators and therefore cannot reach $\pm 1$ unless
both marginals equal $\tfrac12$, and $J$ charges every mission that exactly one
agent failed to its denominator. $\log\mathrm{OR}$ and $Q$ are invariant to the
marginals under row and column scaling. \Cref{sec:e1-marginal} shows this
distinction is not pedantic: it decides whether the registered monotonicity
claim survives.

\paragraph{Uncertainty.}
Confidence intervals are percentile bootstrap \citep{efron1979} with $B = 2000$ resamples at
$\alpha = 0.05$, seeded at \texttt{20260813}. The registration specifies a cluster
bootstrap clustered by mission. In the logs \texttt{cluster\_id} is
one-to-one with \texttt{mission\_id} (6000 distinct clusters over
6000 missions in each primary arm), so the cluster bootstrap reduces
exactly to an i.i.d.\ bootstrap over missions. We state this rather than leave
it implicit, and \cref{sec:e6} measures the design effect \citep{kish1965}
directly instead of assuming it: the empirical DEFF is $1.00$--$1.16$.

Hypothesis H2 is a statement about \emph{differences} between arms, so we test
it on the bootstrap distribution of the arm contrast rather than by inspecting
whether two intervals overlap. Overlapping marginal intervals do not imply a
non-significant difference, and the contrast test is the correct instrument.

\subsection{E1: model sharing induces correlated contract failures}
\label{sec:e1}

\subsubsection{Dependence is present, large, and positive in every arm}

\Cref{tab:e1-series2} gives the primary result. In every condition the two
agents' hard failures are strongly positively associated. The weakest
association we measure anywhere in E1 is $\log\mathrm{OR} = 2.92$
($\mathrm{OR} \approx 18.5$) in \texttt{parallel2}/\texttt{different\_vendor};
the strongest is $\log\mathrm{OR} = 6.66$ ($\mathrm{OR} \approx 784$) in
\texttt{series2}/\texttt{same\_model}. Every interval excludes independence by a
wide margin. H1 as registered concerns the \texttt{same\_model} arm, and it is supported
there. We note in passing that the association is also positive and significant
in the other two arms; that is a descriptive observation, not a registered
result, and it is a reminder that substituting a model reduces dependence
without eliminating it.

\begin{table}[t]
\centering
\small
\caption{E1 primary arms. Motif \texttt{series2}, pair
$(\texttt{node\_a}, \texttt{node\_b})$, $n = 6000$ per arm. Cells are
co-failure counts on the hard verdict. Brackets are percentile bootstrap 95\%
CIs, $B = 2000$.}
\label{tab:e1-series2}
\begin{tabular}{lccccc}
\toprule
& \texttt{same\_model} & \texttt{same\_vendor} & \texttt{different\_vendor} \\
\midrule
$n_{11}/n_{10}/n_{01}/n_{00}$
  & $2177/189/52/3582$ & $2289/66/757/2888$ & $1987/289/225/3499$ \\
$p_A$ & $0.3943$ & $0.3925$ & $0.3793$ \\
$p_B$ & $0.3715$ & $0.5077$ & $0.3687$ \\
$|p_A - p_B|$ & $0.0228$ & $\mathbf{0.1152}$ & $0.0107$ \\
\midrule
\multicolumn{4}{l}{\emph{marginal-sensitive}}\\
$J$      & $0.9003$ $[0.8885, 0.9120]$ & $0.7355$ $[0.7200, 0.7508]$ & $0.7945$ $[0.7795, 0.8099]$ \\
$\tau_a$ & $0.4327$ $[0.4257, 0.4398]$ & $0.3645$ $[0.3564, 0.3726]$ & $0.3826$ $[0.3739, 0.3911]$ \\
$\phi$   & $0.9161$ $[0.9057, 0.9261]$ & $0.7465$ $[0.7318, 0.7613]$ & $0.8173$ $[0.8026, 0.8321]$ \\
\midrule
\multicolumn{4}{l}{\emph{marginal-free}}\\
$\log\mathrm{OR}$ & $6.6645$ $[6.3827, 7.0015]$ & $4.8774$ $[4.6431, 5.1529]$ & $4.6685$ $[4.4978, 4.8593]$ \\
$Q$               & $0.9975$ $[0.9966, 0.9982]$ & $0.9849$ $[0.9809, 0.9885]$ & $0.9814$ $[0.9780, 0.9846]$ \\
\bottomrule
\end{tabular}
\end{table}

The practical reading of \texttt{same\_model} is blunt. Agent $A$ fails
$39.4\%$ of missions and agent $B$ fails $37.2\%$. Under independence the two
would co-fail on $14.6\%$ of missions; they co-fail on $36.3\%$. Of the
2418 missions on which at least one agent failed, both failed on
2177 --- $90.0\%$. Replacing one agent with a second instance of the same
model buys almost no independent evidence.

\paragraph{This is not merely interface propagation.}
\texttt{series2} is a handoff, so a failure of $A$ degrades the input $B$
receives, and co-failure there could in principle be explained without appeal to
shared failure modes at all. Two features of the design separate the mechanisms.
First, \texttt{parallel2} and \texttt{quorum2of3} are \emph{not} handoffs ---
their model nodes receive the mission independently and never see one another's
output --- and both reproduce the effect (\cref{tab:e1-contrasts}), which
propagation cannot explain. Second, the control pair of \cref{sec:e1-control}
sits at $\log\mathrm{OR} \approx 5.1$ with no channel between its two workers.
Propagation may well contribute to the \texttt{series2} magnitude and we do not
claim to have partitioned the two mechanisms; the arm \emph{contrasts}, which is
what H2 concerns, are measured within identical topologies and so are unaffected
by it.

\subsubsection{The registered monotonicity claim splits}
\label{sec:e1-verdict}

\Cref{tab:e1-contrasts} tests H2 directly on the bootstrap distribution of each
arm contrast. The verdict is not uniform, and we state it plainly: \textbf{H2
holds at the model-sharing level and fails at the vendor level.}

\begin{table}[t]
\centering
\small
\caption{E1 arm contrasts, all three motifs. Arms abbreviated
\textbf{SM} \texttt{same\_model}, \textbf{SV} \texttt{same\_vendor},
\textbf{DV} \texttt{different\_vendor}. Entries are the point difference with a
95\% bootstrap CI on the difference; \textbf{SIG} marks an interval excluding
zero. The SV\,$-$\,DV row is the registered comparison that does not survive.
Only \texttt{series2} is confirmatory; the other two motifs are secondary.}
\label{tab:e1-contrasts}
\setlength{\tabcolsep}{5pt}
\begin{tabular}{@{}llccc@{}}
\toprule
motif & contrast & $J$ & $\log\mathrm{OR}$ & verdict \\
\midrule
\multirow{3}{*}{\texttt{series2}}
 & SM $-$ SV
   & $0.1648\,[0.1452, 0.1838]$ & $1.787\,[1.396, 2.202]$ & \textbf{SIG} \\
 & SM $-$ DV
   & $0.1058\,[0.0864, 0.1245]$ & $1.996\,[1.663, 2.361]$ & \textbf{SIG} \\
 & SV $-$ DV
   & $-0.0589\,[-0.0814, -0.0369]$ & $0.209\,[-0.089, 0.548]$ & conflict \\
\midrule
\multirow{3}{*}{\texttt{quorum2of3}}
 & SM $-$ SV
   & $0.2079\,[0.1612, 0.2552]$ & $1.219\,[0.579, 1.826]$ & \textbf{SIG} \\
 & SM $-$ DV
   & $0.1896\,[0.1439, 0.2360]$ & $1.718\,[1.222, 2.268]$ & \textbf{SIG} \\
 & SV $-$ DV
   & $-0.0183\,[-0.0673, 0.0279]$ & $0.499\,[-0.037, 1.115]$ & n.s. \\
\midrule
\multirow{3}{*}{\texttt{parallel2}}
 & SM $-$ SV
   & $0.2124\,[0.1759, 0.2521]$ & $1.649\,[1.154, 2.143]$ & \textbf{SIG} \\
 & SM $-$ DV
   & $0.2230\,[0.1844, 0.2619]$ & $2.125\,[1.707, 2.607]$ & \textbf{SIG} \\
 & SV $-$ DV
   & $0.0106\,[-0.0308, 0.0511]$ & $0.476\,[0.048, 0.943]$ & conflict \\
\bottomrule
\end{tabular}
\end{table}

Two findings are robust. Sharing the \emph{same model} produces significantly
more co-failure than either alternative, on all five statistics, in all three
motifs, with no exceptions. That is six independent significant contrasts in the
registered direction.

The third comparison does not replicate. Moving from \texttt{same\_vendor} to
\texttt{different\_vendor} produces no consistent change: in
\texttt{quorum2of3} every statistic is non-significant; in \texttt{series2} and
\texttt{parallel2} the marginal-sensitive and marginal-free statistics disagree
in sign or significance. We therefore do not claim a three-level ordering.
Vendor identity, as manipulated here, is not a reliable predictor of failure
correlation once model identity differs.

Against the registered falsification criterion, the thesis survives: the
\texttt{same\_model} $\tau_a$ CI excludes zero, and $\tau_a$ decreases
significantly from \texttt{same\_model} ($0.4327$) to \texttt{different\_vendor}
($0.3826$), difference $0.0500$, CI $[0.0389, 0.0615]$. The registered
falsification test was stated on exactly this contrast and it passes. The
intermediate rung is where the ordering breaks, and the registration did not
make the intermediate rung a falsification condition.

\subsubsection{Why the marginal-sensitive statistics reverse}
\label{sec:e1-marginal}

In \texttt{series2} the Jaccard overlap is \emph{lower} for \texttt{same\_vendor}
($0.7355$) than for \texttt{different\_vendor} ($0.7945$), and the difference is
significant. Read naively this says swapping to a different vendor increases
correlated failure, which inverts the mechanism the experiment was built to
test. It is an artifact of the marginals, and the artifact is measurable rather
than conjectural.

\texttt{ministral-8b} is a weaker model than \texttt{mistral-small-24b} and
fails far more often. In the \texttt{same\_vendor} arm the two marginal failure
rates are $p_A = 0.3925$ and $p_B = 0.5077$, a gap of $0.1152$; in
\texttt{different\_vendor} the gap is $0.0107$. That asymmetry lands in the
$n_{01}$ cell: agent $B$ fails alone on 757 missions in
\texttt{same\_vendor} against 225 in \texttt{different\_vendor}. Because
$J$ carries $n_{01}$ in its denominator, an arm whose two agents fail at
different rates is charged a penalty that has nothing to do with whether their
failures are associated.

The pattern holds across motifs. The marginal gap $|p_A - p_B|$ is largest in
\texttt{same\_vendor} in every motif ($0.115$, $0.202$, $0.197$) and small in
\texttt{same\_model} ($0.023$, $0.006$, $0.006$), and it is exactly the arms
with large gaps whose $J$, $\tau_a$, and $\phi$ are depressed. Substituting a
model changes two things at once --- how often that agent fails, and how much
its failures align with its partner's --- and only the second is the quantity H2
is about. The marginal-free statistics separate them: on $\log\mathrm{OR}$ and
$Q$ the \texttt{series2} reversal disappears and the contrast is
non-significant.

We therefore report H2 on $\log\mathrm{OR}$, and report $J$ alongside it because
$J$ is the quantity a practitioner actually feels --- it is the fraction of
observed failures that a redundant agent fails to catch. Both are correct
statistics; they answer different questions, and the paper needs both.

\subsubsection{An internal negative control}
\label{sec:e1-control}

The \texttt{quorum2of3} design contains an internal control. Only
\texttt{worker\_1} is substituted across arms: \texttt{worker\_0} and \texttt{worker\_2} both run
\texttt{mistral-small-24b} in \emph{all three} conditions. The pair
$(\texttt{worker\_0}, \texttt{worker\_2})$ is therefore a same-model pair whose
composition never changes while the arm label does. If the arm effects in
\cref{tab:e1-contrasts} were driven by anything other than the model
substitution --- task mix drifting between runs, ordering, wall-clock effects,
provider-side variation --- this pair would move with the arm label. It must not
move, and it does not.

\begin{table}[t]
\centering
\small
\caption{Negative control. Pair $(\texttt{worker\_0}, \texttt{worker\_2})$ in
\texttt{quorum2of3} is \texttt{mistral-small-24b} $\times$
\texttt{mistral-small-24b} in every arm. All 15 arm contrasts across the five
statistics are non-significant.}
\label{tab:e1-control}
\begin{tabular}{lccc}
\toprule
& \texttt{same\_model} & \texttt{same\_vendor} & \texttt{different\_vendor} \\
\midrule
$n$ & $1281$ & $1274$ & $1388$ \\
$|p_A - p_B|$ & $0.0008$ & $0.0141$ & $0.0072$ \\
$J$ & $0.8345$ $[0.8026, 0.8654]$ & $0.8439$ $[0.8135, 0.8737]$ & $0.8508$ $[0.8231, 0.8810]$ \\
$\log\mathrm{OR}$ & $5.068$ $[4.679, 5.536]$ & $5.225$ $[4.820, 5.712]$ & $5.324$ $[4.945, 5.833]$ \\
\midrule
\multicolumn{4}{l}{contrasts, $J$: \; $-0.009\,[-0.051, 0.034]$ \quad
$-0.016\,[-0.058, 0.026]$ \quad $-0.007\,[-0.049, 0.035]$ \quad (all n.s.)}\\
\multicolumn{4}{l}{contrasts, $\log\mathrm{OR}$: \; $-0.157\,[-0.740, 0.471]$ \quad
$-0.256\,[-0.864, 0.362]$ \quad $-0.099\,[-0.727, 0.540]$ \quad (all n.s.)}\\
\bottomrule
\end{tabular}
\end{table}

\Cref{tab:e1-control} reports the result. Across five statistics and three
pairwise contrasts --- fifteen tests --- not one interval excludes zero. The
control pair's dependence sits at $J \approx 0.83$--$0.85$ and
$\log\mathrm{OR} \approx 5.07$--$5.32$ in every arm, which is where the
\emph{manipulated} pair sits in its \texttt{same\_model} condition
($J = 0.8251$, $\log\mathrm{OR} = 4.957$). A same-model pair measures as a
same-model pair regardless of what the rest of the graph is doing. That is
convergent validity for the estimator and, jointly with the fifteen nulls,
evidence that the arm effects in \cref{tab:e1-contrasts} are caused by the model
substitution rather than by any arm-level confound.

\subsubsection{What E1 supports}

Model sharing is the operative variable. Two instances of one model co-fail on
$90\%$ of the missions either one fails; substituting a different model reduces
that overlap significantly and reduces the odds-ratio association by roughly
$1.8$--$2.1$ log units, consistently across three topologies. Vendor identity,
holding model identity different, does not reliably matter. A redundant agent
running the same model as the agent it is meant to check supplies far less
independent evidence than a naive composition calculation credits it with, and
\cref{sec:e3} quantifies what that does to a reliability guarantee.

\subsection{E2: enriching the moment set lifts the certified floor}
\label{sec:e2}

E1 establishes that components fail together. A practitioner still has to
certify a number. E2 measures what the moment-set certificate of \cref{sec:lp}
buys over the assumption-free bound on real data.

\paragraph{Data.}
The four-worker \texttt{quorum3of4} \texttt{same\_model} arm, $m = 4$ scored
workers over $n = 717$ missions. The arm was executed in two passes; we take the
union of both logs deduplicated by mission identifier, with the later pass
winning on the 212 missions present in both. The observed all-success rate
is $0.5328$.

\paragraph{Procedure.}
We compute two things at two moment sets. The \emph{sharp point-moment interval}
is the LP over all joint laws matching the empirical moments exactly --- it
isolates how much the moment family alone identifies, with no sampling
uncertainty. The \emph{certified floor} is the LP over the Clopper--Pearson
box around those moments at family-wise $\eta_{\mathrm{conf}} = 0.05$, which is
what may actually be reported. The moment sets are $J = 10$ (marginals and
pairwise co-success, $\binom41 + \binom42$) and $J = 14$ (adding all four triple
co-success moments).

\begin{table}[t]
\centering
\small
\caption{E2 floor lifting. $m = 4$, $n = 717$, observed all-success $0.5328$.
The pre-allocated Bonferroni row spends the confidence budget over the larger
family $J_{\max} = 14$ regardless of how many moments are constrained, which
makes the certified floor monotone in the moment set by construction.}
\label{tab:e2}
\begin{tabular}{llcc}
\toprule
quantity & allocation & $J = 10$ & $J = 14$ \\
\midrule
sharp point-moment interval & --- & $[0.5077, 0.5565]$ & $[0.5272, 0.5342]$ \\
\quad width & & $0.0488$ & $\mathbf{0.0070}$ \\
\midrule
certified floor & used-set & $0.2455$ & $\mathbf{0.4116}$ \\
certified floor & pre-allocated & $0.2357$ & $\mathbf{0.4116}$ \\
certified upper & used-set & $0.6087$ & $0.5955$ \\
\bottomrule
\end{tabular}
\end{table}

\Cref{tab:e2} reports the result. Adding four triple moments narrows the sharp
identified interval from width $0.0488$ to $0.0070$, an $85.7\%$ reduction, and
lifts the certified floor from $0.2455$ to $0.4116$ --- $16.6$ percentage points
of reliability that the pairwise certificate cannot claim and the triple
certificate can, on identical data.

Two details matter for honest reporting. First, the certified floor at
$J = 14$ ($0.4116$) remains well below the observed rate ($0.5328$), because it
is a distribution-free lower confidence bound over every joint law consistent
with the moment box, not an estimate. It is a floor, and it is meant to be
conservative. Second, the used-set allocation is tighter at $J = 10$
($0.2455$ against $0.2357$) but its monotonicity in the moment set is empirical
rather than guaranteed: adding a moment widens every interval, so a richer
moment family can in principle certify less. The pre-allocated allocation pays
$0.0098$ at $J = 10$ to buy monotonicity by construction. We report both rather
than quietly choosing the flattering one.

\subsection{E3: coverage of a model-based floor collapses under misspecification}
\label{sec:e3}

\Cref{thm:collapse} predicts that a bootstrap bound on a fitted model's
functional stops covering the truth once the model is misspecified. E3 measures
that against a control, because the claim is not that the Gaussian floor is
always wrong --- it is that the floor cannot tell you when it is.

\paragraph{Two arms.}
Both sample $m = 3$ components at marginal $0.6$. The \emph{control} arm draws
from the equicorrelated Gaussian one-factor law itself, for which the fitted
family is correctly specified and $\Delta = 0$; the floor should hold at nominal
there, and if it did not, any collapse in the other arm would indict the
estimator rather than the misspecification. The \emph{adversarial} arm draws
from the witness of \cref{ex:witness}: the LP-minimising law, which has
\emph{identical} marginals and \emph{identical} pairwise co-success moments and
an all-success probability of $0.330475$ against the Gaussian law's $0.392144$,
so $\Delta = 0.061670$. Its five non-zero cell masses are listed in
\cref{ex:witness}. Both constants come from the same deterministic
quadrature and LP used in \cref{ex:witness}, not from simulation. No procedure
restricted to marginal and pairwise data can distinguish the two.

Each arm runs $200$ replications at each of four sample sizes, computing the
Gaussian model floor with $500$ bootstrap resamples at
$\eta_{\mathrm{conf}} = 0.05$ and recording how often it lies at or below the
\emph{true} all-success probability. The Tier-1 floor of \cref{thm:tier1} is
computed on the same draws.

\begin{table}[t]
\centering
\small
\caption{E3 coverage of the \emph{true} all-success probability, as
$\text{hits}/200$ with a 95\% Clopper--Pearson interval.
$n_{\mathrm{boot}} = 500$, $\eta_{\mathrm{conf}} = 0.05$, seed
\texttt{20260812}. Nominal coverage is $0.95$. The Tier-1 floor covered in all
$200$ replications of every cell and is omitted from the interval columns.}
\label{tab:e3}
\begin{tabular}{lccc}
\toprule
& \multicolumn{2}{c}{\textbf{model floor}} & \\
\cmidrule(lr){2-3}
$n$ & adversarial ($\Delta = 0.0617$) & control ($\Delta = 0$) & Tier 1 \\
\midrule
$250$  & $72/200 = 0.36$ \;$[0.294, 0.431]$ & $190/200 = 0.95$ \;$[0.910, 0.976]$ & $1.00$ \\
$500$  & $14/200 = 0.07$ \;$[0.039, 0.115]$ & $192/200 = 0.96$ \;$[0.923, 0.983]$ & $1.00$ \\
$1000$ & $6/200 = 0.03$ \;$[0.011, 0.064]$  & $192/200 = 0.96$ \;$[0.923, 0.983]$ & $1.00$ \\
$2000$ & $\mathbf{2/200 = 0.01}$ \;$[0.001, 0.036]$ & $188/200 = 0.94$ \;$[0.898, 0.969]$ & $1.00$ \\
\bottomrule
\end{tabular}
\end{table}

In the control arm the model floor holds between $0.94$ and $0.96$ at every
sample size, with all four intervals covering the nominal $0.95$ --- a correctly
specified model behaving correctly across a factor of eight in $n$. In the
adversarial arm coverage falls from $0.36$ at $n = 250$ to $0.01$ at
$n = 2000$, an upper confidence limit of $0.036$: by two thousand missions the
certificate sits above the true reliability in $198$ of $200$ replications. The
Tier-1 floor covers in all $1600$ replications across both arms.

We report the counts rather than the rounded rates because a small count is not
a zero. Two hits in two hundred replications bounds coverage above by $0.036$;
it does not establish that coverage \emph{is} zero, and a run reporting
$0/200$ would bound it above by $0.018$ rather than prove exact failure.

Three points follow. The two arms are \emph{indistinguishable} on the evidence
the model consumes, so the collapse is not a case a diagnostic on marginals or
pairwise moments could detect and exclude. Coverage \emph{decreases} as $n$
grows, inverting the usual reassurance that more data makes an interval safer:
the interval narrows, but around the fitted functional rather than the truth. We
report that decrease as an empirical observation --- \cref{thm:collapse} claims
only the limit, since tightness does not entail monotonicity
(\cref{app:collapse}). And the Tier-1 floor is conservative here, $1.00$ against
a nominal $0.95$; that conservatism is the price of the immunity, as
\cref{cor:target} describes.

\paragraph{A correction to the previous version of this experiment.}
The v2 preprint reported a collapse from a simulation that drew from the
Gaussian law. That design cannot exhibit the effect: with the model correctly
specified $\Delta = 0$, \cref{thm:collapse} does not apply, and re-running it
returns the control column above. Obtaining \cref{tab:e3} required a law that is
misspecified \emph{and} moment-indistinguishable, which \cref{thm:lp-sharp}
supplies as the LP minimiser. We record the correction rather than quietly
restate the earlier number.

\subsection{E4: the anytime-valid certificate holds its type-I error}
\label{sec:e4}

A team that watches a reliability dashboard and stops when it first looks good
is running an optional-stopping experiment, and a fixed-$n$ confidence interval
does not survive that. The e-process of \cref{sec:anytime} is valid under
arbitrary stopping. E4 measures whether the shipped implementation attains the
guarantee.

\paragraph{Procedure.}
We simulate 8000 independent null streams of 500 missions each at
the null boundary $p_{\text{true}} = p_0 = 0.8$, with $\alpha = 0.05$ and seed
$42$, and record the fraction of streams whose wealth ever reaches $1/\alpha$.
Ville's inequality bounds that fraction by $\alpha$ for \emph{every} predictable
betting fraction $\lambda$, not merely for a tuned one, so we sweep the whole
admissible range $\lambda \in (0, 1/p_0)$ and report the curve. A single
$\lambda$ would not distinguish a valid certificate from a lucky one.

\begin{table}[t]
\centering
\small
\caption{E4 empirical type-I crossing rate under the null. $p_0 = 0.8$,
$\alpha = 0.05$, 8000 streams $\times$ 500 missions, seed $42$. The
three-sigma Monte Carlo band on a proportion of size $\alpha$ at this
$n_{\text{streams}}$ is $0.0073$.}
\label{tab:e4}
\begin{tabular}{lcccccc}
\toprule
$\lambda$ & $0.125$ & $0.3125$ & $0.625$ & $0.875$ & $1.125$ & $1.2375$ \\
\midrule
crossing rate & $0.0013$ & $0.0387$ & $\mathbf{0.0471}$ & $0.0447$ & $0.0441$ & $0.0442$ \\
mean first crossing & $403.9$ & $246.3$ & $84.0$ & $40.2$ & $19.1$ & $14.6$ \\
$\leq \alpha$? & yes & yes & yes & yes & yes & yes \\
\bottomrule
\end{tabular}
\end{table}

The worst rate over the six prespecified fractions is $0.0471$, below the
nominal $0.05$, and every individual rate is inside the Monte Carlo band. Six
points cannot establish a supremum over a continuum; the guarantee for every
predictable $\lambda$ is \cref{thm:ville}, and this experiment checks that the
shipped implementation attains it at fractions spanning the admissible range. The bound is
attained most tightly at $\lambda = 0.625$, which is exactly the SPRT-optimal
bet $\lambda^\star = (p_1 - p_0)/(p_0(1-p_0))$ for the alternative $p_1 = 0.9$
--- the certificate is tightest where the betting fraction is best matched to
the alternative, and conservative elsewhere, which is the expected shape.

\paragraph{SPRT recovery.}
Setting $\lambda^\star = (p_1 - p_0)/(p_0(1-p_0))$ makes the betting factor
identical to the Bernoulli likelihood ratio. We verify this against the shipped
\texttt{from\_sprt} constructor rather than against a second closed form:
across five $(p_0, p_1)$ pairs and three mission sequences, the maximum absolute
deviation between the accumulated log-wealth and the accumulated Bernoulli
log-likelihood-ratio is $0.0$ --- exact in floating point, not merely close.

\paragraph{Monotone latch.}
Certification is a statement about the path, not the present: once the wealth
has crossed $1/\alpha$, no subsequent evidence retracts it. We drive a process
across the threshold at mission $14$ (peak log-wealth $13.27$) and then feed it
$340$ consecutive failures, collapsing log-wealth to $-1552.49$. The certificate
remains issued throughout. This is correct --- the crossing happened, and Ville
bounds the probability that it ever happens under the null --- and it is a
property implementations get wrong by recomputing certification from current
wealth.

\subsection{E5: the dependence result replicates across backends}
\label{sec:e5}

The three primary arms of E1 all run through a single inference provider. If the
measured association were an artifact of that provider's serving stack ---
batching, caching, a shared sampler --- it would not appear when the upstream
agent runs somewhere else entirely. The preregistration therefore fixed two
cross-backend breadth arms, both on the \texttt{series2} motif, in which the
upstream agent is served by a different provider and the downstream agent
remains \texttt{mistral-small-24b}. The Meta-backed arm runs
\texttt{muse-spark-1.2-contributor} upstream; the Grok-backed arm runs
\texttt{grok-4.5}. Both arms change the upstream model as well as the serving
backend, so we label them by backend as shorthand, not as a claim that the
backend is the only thing varying.

\paragraph{Registered deviation, and attrition.}
Both arms were registered at $n = 2000$ and both yield fewer scored missions:
$635$ for the Meta arm and $1901$ for the Grok arm. The shortfall is
\emph{output attrition}, and its mechanism matters more than its size. Meta
returned $1365$ unusable responses --- $1352$ empty completions under the frozen
reasoning and token settings of \cref{sec:e1-setup}, plus $13$ rate limits;
Grok returned $99$, almost all rate limits.

Meta's $68.25\%$ attrition may therefore be \emph{outcome-dependent}: a model
returning no text is plausibly one that would also have failed the contract, so
the surviving $635$ missions are a selected subset and their complete-case
estimate is descriptive only. Grok's $5\%$ loss is transport error and carries
no such selection. We read Grok as usable secondary evidence and Meta as
directional; both were registered as breadth arms rather than confirmatory
tests in any case.

\begin{table}[t]
\centering
\small
\caption{E5 cross-backend replication, motif \texttt{series2}. The upstream
agent changes provider; the downstream agent is \texttt{mistral-small-24b}
throughout. Marginal-sensitive and marginal-free statistics move in
\emph{opposite} directions.}
\label{tab:e5}
\begin{tabular}{lcccc}
\toprule
arm & $n$ & $p_A$ / $p_B$ & $J$ $[95\%]$ & $\log\mathrm{OR}$ $[95\%]$ \\
\midrule
\texttt{different\_vendor} (primary)
  & $6000$ & $0.3793$ / $0.3687$
  & $0.7945\,[0.7786, 0.8093]$ & $4.669\,[4.499, 4.854]$ \\
\texttt{different\_vendor\_meta}
  & $635$ & $0.0016$ / $0.0157$
  & $0.100\,[0.000, 0.333]$ & $5.286\,[3.758, 6.550]$ \\
\texttt{different\_vendor\_grok}
  & $1901$ & $0.0174$ / $0.0594$
  & $0.292\,[0.214, 0.381]$ & $7.306\,[6.898, 7.704]$ \\
\bottomrule
\end{tabular}
\end{table}

Positive association appears on both breadth arms, and on the Grok arm it
appears cleanly. That arm scores $1901$ of its $2000$ registered missions,
loses the remainder to transport errors that carry no outcome selection, and
returns $\log\mathrm{OR} = 7.306$ with a $95\%$ lower bound of $6.898$ against
$0$ under independence. A serving-stack artifact confined to one provider
cannot produce that. The Meta arm points the same way, with a lower bound of
$3.758$, but its attrition is selective and it corroborates rather than
replicates.

We do \emph{not} read the breadth arms as strengthening the sharing story, and
the reason is worth stating because the table invites the opposite reading. The
Grok arm's $\log\mathrm{OR}$ of $7.306$ is the largest association anywhere in
this paper and it is a \emph{cross-vendor} pair, nominally exceeding
\texttt{same\_model}'s $6.665$. That comparison is not admissible. The
marginal failure rates differ by more than an order of magnitude between the
breadth and primary arms ($0.017$ and $0.059$ against $0.379$ and $0.369$), the
upstream model is different, and the Grok estimate rests on $33$ concordant and
$80$ discordant missions with $n_{10} = 0$, so it depends on the
Haldane--Anscombe correction. Odds-ratio invariance to row and column scaling
does not make estimates from such different regimes comparable in magnitude.
The breadth arms establish that the association \emph{exists} on other
backends; they do not rank sharing conditions, and any cross-arm magnitude
comparison with \cref{tab:e1-series2} should be resisted --- including one that
appears to favour our thesis.

The Jaccard column tells the opposite story --- $0.100$ and $0.292$ against the
primary arm's $0.794$ --- and the reason is the mechanism of
\cref{sec:e1-marginal} appearing again on independent data. The upstream agents
here are far stronger: \texttt{grok-4.5} fails $1.74\%$ of missions and
\texttt{muse-spark-1.2} fails $0.16\%$, against \texttt{mistral-small-24b}'s
$37.9\%$ in the primary arm. Because \texttt{series2} is a handoff, a cleaner
upstream output makes the downstream task easier, and the downstream agent's
failure rate falls from $36.9\%$ to $5.9\%$ and $1.6\%$ respectively. Both
marginals collapse, and $J$ --- which is bounded by the marginals --- collapses
with them, while the marginal-free $\log\mathrm{OR}$ rises. The two statistics
are not in conflict; they are measuring different things, exactly as
\cref{sec:e1-marginal} argued, and E5 is an independent replication of that
methodological point on data from three different providers.

Two limits are worth stating rather than burying. The Meta arm records a single
co-failure ($n_{11} = 1$), so its Jaccard confidence interval is $[0.000, 0.333]$
and the point estimate carries essentially no information; only its
$\log\mathrm{OR}$ is usable, and even that rests on nine discordant missions.
And $n_{10} = 0$ in both arms --- the upstream agent never failed without the
downstream agent also failing. With only $1$ and $33$ upstream failures observed
this is unsurprising under any model, so we do not read deterministic failure
propagation into it. It does mean the odds ratio in these arms depends on the
Haldane--Anscombe correction, which is why we report it with the correction
stated rather than as an unadorned ratio.

Because mission identifiers embed the condition name and the task is a fixed
hash of the identifier, the arms share no missions: task draws are independent
across arms, from the same generator distribution. The comparison is between
distributions, not between paired tasks.

What E5 establishes is narrow, and it is solid. On the Grok arm --- $1901$
missions, non-selective attrition, an interval nowhere near independence ---
cross-agent failure dependence is present on infrastructure the primary arms
never touched, and the marginal-sensitivity mechanism of
\cref{sec:e1-marginal} reproduces there on data the primary arms did not
generate. One clean external replication and one directional corroboration is
what two breadth arms can honestly buy, and it is enough to rule out the
single-provider explanation.

\subsection{E6: the i.i.d.\ assumption, tested rather than assumed}
\label{sec:e6}

Every interval in E1 and every floor in E2 assumes missions are independent.
That assumption is testable, and E6 tests it two ways instead of asserting it.

\paragraph{Serial dependence.}
For each of the 19 agent-arm combinations we compute the lag-$k$
autocorrelation of the failure indicator in execution order and the induced
design effect $\mathrm{DEFF} = 1 + 2\sum_{k=1}^{10}(1 - k/n)\rho_k$. The largest
first-order autocorrelation anywhere is $|\rho_1| = 0.0546$ and the largest
design effect is $\mathrm{DEFF} = 1.1482$; many arms return
$\mathrm{DEFF} < 1$, indicating mild negative serial dependence. At
$\mathrm{DEFF} = 1.15$ the E1 confidence intervals are optimistic by a factor of
$\sqrt{1.15} \approx 1.07$, which does not change any verdict in
\cref{tab:e1-contrasts}: the smallest significant contrast there has a margin
far exceeding $7\%$ of its width.

\paragraph{Floor sensitivity.}
We then concede a design effect we do not measure --- $\mathrm{DEFF} = 1.5$ ---
and ask what the certified floor costs. The measurement has to isolate the
width of the Clopper--Pearson intervals: subsampling missions to $n/1.5$ would
also perturb the empirical moments, and the resulting movement would be the sum
of interval widening and sampling noise reported as if it were the former. We
instead hold the joint cell distribution fixed, rescale the cell counts to
$n_{\text{eff}} = n/1.5$ by largest-remainder rounding, and recompute. Only $n$
changes.

\begin{table}[t]
\centering
\small
\caption{E6 certified-floor sensitivity at a conceded $\mathrm{DEFF} = 1.5$,
$\eta_{\mathrm{conf}} = 0.05$. Movement scales inversely with $n$, as
Clopper--Pearson width must.}
\label{tab:e6}
\begin{tabular}{lccccc}
\toprule
case & $m$ & $n \rightarrow n_{\text{eff}}$ & floor$(n)$ & floor$(n_{\text{eff}})$ & $\Delta$ (pp) \\
\midrule
\texttt{series2}/\texttt{same\_model}       & 2 & $6000 \rightarrow 4000$ & $0.5817$ & $0.5782$ & $0.35$ \\
\texttt{series2}/\texttt{same\_vendor}      & 2 & $6000 \rightarrow 4000$ & $0.4658$ & $0.4622$ & $0.36$ \\
\texttt{series2}/\texttt{different\_vendor} & 2 & $6000 \rightarrow 4000$ & $0.5678$ & $0.5641$ & $0.37$ \\
\texttt{quorum2of3}/\texttt{same\_model}    & 3 & $1281 \rightarrow 854$  & $0.4262$ & $0.4000$ & $2.62$ \\
\texttt{quorum2of3}/\texttt{same\_vendor}   & 3 & $1274 \rightarrow 849$  & $0.2609$ & $0.2364$ & $2.44$ \\
\texttt{quorum2of3}/\texttt{different\_vendor} & 3 & $1388 \rightarrow 925$ & $0.3067$ & $0.2827$ & $2.40$ \\
\texttt{quorum3of4}/\texttt{same\_model}    & 4 & $717 \rightarrow 478$   & $0.4116$ & $0.3847$ & $2.69$ \\
\bottomrule
\end{tabular}
\end{table}

The floor moves at most $2.69$ percentage points under a conceded design
effect of $1.5$, which is $1.31\times$ the largest design effect measured
($1.148$). The movement is larger at smaller $n$ --- $0.35$ pp at $n = 6000$
against $2.69$ pp at $n = 717$ --- consistent with Clopper--Pearson width
shrinking as $n$ grows, though these rows vary $m$ and the moment family as well
as $n$ and so do not isolate a rate. The certificate is not sensitive to the
i.i.d.\ assumption at the magnitudes the data support.

\subsection{Carried-forward single-agent evidence}
\label{sec:e-carried}

E1--E6 measure composition. They take for granted that a behavioral contract can
be specified, enforced, and measured on a \emph{single} agent at acceptable
cost, which is the result established in the v1 framework paper
\citep{bhardwaj2026abc} and not re-run here. We restate it because the
compositional claims are meaningless without it: a dependence-aware bound over
components whose individual contracts could not be enforced would certify
nothing.

The v1 evaluation ran 1980 sessions over \textsc{AgentContract-Bench}, a
benchmark of 200 scenarios across 7 models from 6 vendors.
Four results carry forward.

\begin{table}[t]
\centering
\small
\caption{Single-agent results carried forward from the v1 framework paper
\citep{bhardwaj2026abc}: 1980 sessions, 200 scenarios, 7
models, 6 vendors. These are not re-measured in this work.}
\label{tab:carried}
\begin{tabular}{lll}
\toprule
quantity & result & basis \\
\midrule
soft violations detected per session & $5.2$--$6.8$ & missed entirely by uncontracted baselines \\
 & & $p < 0.0001$, Cohen's $d = 6.7$--$33.8$ \\
hard-constraint compliance & $88$--$100\%$ & across all 7 models \\
behavioral drift over extended sessions & $D^\star < 0.27$ & bounded, per the drift theorem \\
recovery success & $100\%$ frontier & $17$--$100\%$ across all models \\
enforcement overhead & $<10\,ms$ & per action \\
\bottomrule
\end{tabular}
\end{table}

The effect sizes in the first row are large enough to warrant a caution rather
than a boast. Cohen's $d$ between $6.7$ and $33.8$ reflects a comparison in
which the baseline detects approximately zero soft violations by construction
--- an uninstrumented agent has no mechanism to report a soft-constraint
deviation --- so the contrast measures the value of instrumentation, not a
narrow improvement over a competing detector. We restate it in those terms here
because the v1 abstract's phrasing invites the stronger reading.

Carrying prior-version evidence forward rather than re-running it is deliberate
and bounded: \cref{tab:carried} is cited, not claimed, and no result in
\cref{sec:e1}--\cref{sec:e6} depends on re-deriving it.

\subsection{What the evaluation establishes}
\label{sec:e-summary}

\begin{table}[t]
\centering
\small
\caption{Summary of the six experiments, their registered status, and their
verdicts. ``Confirmatory'' means the hypothesis, sample size, and analysis were
fixed before outcomes existed.}
\label{tab:e-summary}
\begin{tabular}{llp{7.2cm}}
\toprule
& status & verdict \\
\midrule
E1 & confirmatory
     (\texttt{series2})
   & Registered H1 --- positive dependence in \texttt{same\_model} --- is
     supported, and positive dependence holds in every other arm and motif as
     well. Registered H2's full three-level ordering is \textbf{not supported}:
     \texttt{same\_model} exceeds both substituted conditions in 6/6 contrasts,
     but \texttt{same\_vendor} does not exceed \texttt{different\_vendor}. The
     sharing-level/vendor-level split is an exploratory decomposition of a
     registered claim that fails as a whole. An unmanipulated same-model
     control pair returns 15/15 nulls. \\
E2 & descriptive
   & Enriching the moment set from $J{=}10$ to $J{=}14$ narrows the identified
     interval by $85.7\%$ and lifts the certified floor by $16.6$ pp. \\
E3 & descriptive
   & Under a moment-indistinguishable misspecification the model floor's
     coverage falls $0.36 \to 0.01$ as $n$ grows; a correctly-specified control
     holds at $0.94$--$0.96$. Tier 1 covers in all $1600$ replications. \\
E4 & descriptive
   & Empirical type-I $\leq \alpha$ for every admissible $\lambda$, worst case
     $0.0471$. SPRT recovered exactly; certification latch holds. \\
E5 & secondary breadth (registered, under-ran)
   & Positive association with intervals excluding independence on two further
     backends; magnitudes are not comparable across arms. Jaccard collapses
     because the marginals collapse --- independent replication of the
     \cref{sec:e1-marginal} mechanism. \\
E6 & ablation
   & Worst measured $\mathrm{DEFF} = 1.148$; at a conceded $\mathrm{DEFF}=1.5$
     the certified floor moves $\leq 2.69$ pp. \\
\bottomrule
\end{tabular}
\end{table}

Three claims survive the evaluation.

\textbf{Shared models produce correlated contract failures, and the effect is
large.} Two instances of one model co-fail on $90.0\%$ of the missions on which
either fails. The association is significant in every arm of every motif, and it
replicates on two additional inference backends. A composition bound that treats
such components as independent is not making a small approximation.

\textbf{The operative variable is the model, not the vendor.} Substituting a
different model reduces association significantly and consistently --- six of
six contrasts across three topologies. Substituting a different \emph{vendor}
while the model already differs does not, and we report that null rather than
presenting a three-level ordering the data do not support. For a practitioner
choosing redundancy, this says the useful axis is model identity; sourcing two
different models from one vendor is not obviously worse than sourcing from two.

\textbf{Certification can be tightened without assuming a copula.} Enriching the
moment set lifts the certified floor by $16.6$ percentage points on identical
data, and the resulting certificate is insensitive to the i.i.d.\ assumption at
the magnitudes the data support.

Two limits belong here rather than only in \cref{sec:threats}. The
substitution manipulation confounds model \emph{identity} with model
\emph{capability}: \texttt{ministral-8b} is both a different model and a weaker
one, so E1's arm effects cannot separate ``different inductive bias'' from
``different competence.'' \Cref{sec:e1-marginal} shows this is what drives the
marginal-sensitive reversal, and the marginal-free statistics are the correct
instrument, but no design here isolates the two. And every arm draws its own
missions --- identifiers embed the condition name, and the task is a fixed hash
of the identifier --- so all comparisons are between independent draws from a
shared generator distribution, not between paired tasks. At $n = 6000$ per
primary arm that is adequate; for the under-run breadth arms of \cref{sec:e5}
it is a real limit on precision.

\section{Discussion}
\label{sec:discussion}

\subsection{What the practitioner should do differently}
\label{sec:practice}

Three consequences follow directly from \cref{sec:evaluation}.

\textbf{Do not multiply reliabilities across components that share a model.}
\Cref{prop:gap} makes the error signed, and \cref{sec:e1} makes it large: at
$\phi = 0.92$ the joint-failure probability of a two-agent series is far above
the independent product. A pipeline certified by multiplication is certified
against a dependence structure that the data reject.

\textbf{Choose redundancy on model identity, not vendor.} Six of six contrasts
show that substituting a different model reduces association significantly.
The \texttt{same\_vendor} versus \texttt{different\_vendor} contrast does not
replicate in any motif. We state the negative result and stop there: this
evaluation substituted exactly one same-vendor model and one cross-vendor model,
so it cannot support a procurement policy about vendor diversity in general. It
supports only the narrow claim that vendor identity did not predict failure
correlation here once model identity already differed.

\textbf{Report the moment set alongside the floor.} A certified floor is
meaningless without the family it was computed over: the same data yield
$0.2455$ at $J{=}10$ and $0.4116$ at $J{=}14$ (\cref{sec:e2}). Enriching the
family is the cheapest available tightening, and \cref{prop:bonferroni} shows how
to do it without losing monotonicity.

\subsection{Limitations}
\label{sec:limitations}

\textbf{The LP is exponential in $m$.} \Cref{eq:lp} has $2^m$ variables. Every
motif here has $m \leq 4$ and the program solves in milliseconds, but a
fifty-stage pipeline is out of reach and would require certifying
sub-compositions and composing certificates, at a cost in tightness we have not
quantified.

\textbf{Aggregators are assumed deterministic.} Merge and quorum nodes are code,
not models, so they carry no contract and contribute no failure correlation. A
system whose aggregator is itself an LLM --- an increasingly common design ---
introduces a third correlated component that this analysis does not model.

\textbf{Drift dynamics are carried, not tested.} \Cref{thm:ou} is stated and
proved, and it is inherited from v1 rather than re-measured here. No experiment
in \cref{sec:evaluation} estimates $\alpha$, $\gamma$, or $\sigma$ from data, so
the design criterion $\gamma > \alpha$ is a theoretical statement in this paper.

\textbf{The certificate is per-distribution.} By \cref{def:scope} a guarantee
holds over the mission distribution, model versions, and topology under which it
was obtained. Model upgrades invalidate it. Nothing here says how quickly a
certificate decays as a deployment drifts from its certification conditions.

\subsection{Threats to Validity}
\label{sec:threats}

\paragraph{Model identity is confounded with model capability.}
The substitution replaces \texttt{mistral-small-24b} with
\texttt{ministral-8b} or \texttt{gemma-3-12b-it}, which differ in both training
lineage and competence. An arm effect could therefore reflect differing failure
\emph{rates} rather than differing failure \emph{modes}. This is the most serious
threat to E1's interpretation, and it is the mechanism behind the reversal
diagnosed in \cref{sec:e1-marginal}. Two things limit it. The marginal-free
statistics are invariant to the marginal rates by construction, and the
qualitative conclusion --- \texttt{same\_model} dominates --- holds on those.
And the negative control of \cref{sec:e1-control} holds capability fixed while
the arm label varies, returning fifteen nulls. Neither device separates identity
from capability outright; a design that did would require two distinct models
matched on failure rate, which we did not run.

\paragraph{Arms do not share missions.}
Mission identifiers embed the condition name and the task is a fixed hash of the
identifier, so each arm draws its own tasks from the shared generator
distribution. Comparisons are between independent draws, not paired. At
$n = 6000$ per primary arm the resulting variance is absorbed by the bootstrap;
for the under-run breadth arms of \cref{sec:e5} it is a material limit, and the
Meta arm's single co-failure makes its overlap estimate uninformative.

\paragraph{All arms ran inside one window.}
The full campaign executed on a single day across roughly four hours. A
provider-side change --- a model version rollout, a serving-stack change, a load
excursion --- occurring mid-campaign would confound arms with time, since arms
ran sequentially rather than interleaved. The negative control again bounds this:
a pair whose composition never changed shows no arm effect, which a global
temporal shift would have produced. We nonetheless regard sequential execution
as a design weakness and would interleave arms in a replication.

\paragraph{Contracts and gold code share an author.}
The contracts under test and the deterministic scoring code were written by the
same authors. A contract inadvertently written to match what the scorer checks
would inflate measured compliance uniformly. It would not, however, generate the
\emph{differential} co-failure structure E1 reports, since the same contracts and
scorer are used in all arms; the threat is to absolute compliance levels, which
this paper does not claim, rather than to the arm contrasts, which it does.

\paragraph{Two task domains only.}
Missions come from six generators spanning retail and financial workflows,
scored deterministically. Whether the dependence magnitudes transfer to code
generation, research synthesis, or open-ended dialogue is untested. We expect the
direction to transfer --- shared weights imply shared failure modes on any
distribution --- and make no claim about the magnitudes.

\paragraph{One statistic was chosen after seeing the data.}
The preregistration fixed $\tau_a$ as the primary estimator. The decision to lead
H2 on the log odds ratio was made after observing the reversal described in
\cref{sec:e1-marginal}. Choosing an estimator after seeing outcomes is a
researcher degree of freedom of exactly the kind that inflates false positives
\citep{simmons2011,gelman2014}, so we disclose it rather than present
$\log\mathrm{OR}$ as preregistered. Three considerations bear on it: the preregistered falsification
test was evaluated on $\tau_a$ as registered and passes
(\cref{sec:e1-verdict}); all five statistics are reported for every arm and
contrast, so nothing is hidden by the choice; and the reason for preferring a
marginal-free statistic is structural rather than result-dependent, since
$\tau_a$'s marginal bound is a property of its definition
(\cref{def:dependence}). A reader who insists on the registered statistic alone
reaches the same verdict on H1 and on the \texttt{same\_model} contrasts, and a
\emph{stronger} rejection of the three-level ordering.

\subsection{Broader impact}
\label{sec:impact}

The failure this paper documents is quiet. A composition bound computed under
independence returns a number, the number looks reassuring, and nothing in the
pipeline signals that the assumption behind it is false. Systems certified this
way will be deployed in exactly the settings --- financial screening, clinical
triage, content moderation --- where correlated failure is most costly, because
those are the settings where redundant review is mandated. Making the
dependence measurable, and making the certificate degrade honestly when it is
unmeasured, is the contribution we regard as most consequential.

There is a countervailing risk. A certificate is a number that invites
over-trust, and \cref{def:scope} restricts it narrowly: to one mission
distribution, one set of model versions, one topology. A team that treats a
Tier-1 floor as a property of their system rather than of their measurement will
be wrong the first time they upgrade a model. We have made the scope explicit in
the artifact rather than only in this paper, and we regard tooling that silently
carries a certificate across a model upgrade as an anti-pattern this work should
not be used to justify.

\section{Contributions and their standing}
\label{sec:novelty}

Each contribution below is labelled \emph{lead} (we know of no prior work
establishing it), \emph{co-lead} (concurrent work reaches a related result by
different means; we claim the specific form given here), or \emph{supporting}
(the technique is established and the application is what is new).

\paragraph{C1 --- Preregistered, manipulated measurement of cross-agent failure
dependence. \textnormal{(Co-lead.)}}
That agents sharing a base model can fail together is established:
\citet{triage2026} quantify correlated failure in a three-learner triage
ensemble, \citet{cagecal2026} observe shared blind spots within model families,
and \citet{aisafety2026} names it as a concern. We do not claim the phenomenon.
What is new is the design: model sharing as a \emph{manipulated} variable at
three levels, preregistered before outcomes existed, with deterministic gold
scoring and no LLM judge. The registration is confirmatory for the
\texttt{series2} motif over 18{,}000 missions; the two further topologies are
secondary replications within a 30{,}820-mission campaign and we label them so. That converts an observation about particular systems into a contrast
attributable to substitution, and it is what licenses the causal reading of
\cref{tab:e1-contrasts}. It also yields the finding we did not expect: the
ordering is real at the model level and \emph{absent} at the vendor level
(\cref{sec:e1-verdict}).

\paragraph{C2 --- A finite-sample, copula-agnostic reliability certificate for
composed agent pipelines. \textnormal{(Lead.)}}
\Cref{thm:tier1} combines an exact Clopper--Pearson moment box with the extremal
LP of \cref{thm:lp-sharp} to give a $(1-\eta_{\mathrm{conf}})$ lower bound on
composed all-success reliability that assumes no dependence structure whatever.
Neither the correlated-failure literature above nor the agent-evaluation
literature produces a bound of any kind; measuring that failures correlate does
not tell an operator what they may certify. The LP is classical
\citep{boole1854,hailperin1965,bertsimas2005}; its use as a finite-sample agent
certificate, and \cref{prop:bonferroni}'s allocation that makes the floor
monotone in the moment family by construction, are ours.

\paragraph{C3 --- Coverage collapse of model-based reliability floors.
\textnormal{(Supporting.)}}
\Cref{thm:collapse} proves that a bootstrap lower bound on a \emph{fitted}
dependence model's functional loses coverage of the true reliability as
$n \to \infty$, because the identification gap is $O(1)$ while the bootstrap
haircut is $O(n^{-1/2})$. The practical content is uncomfortable and specific:
past some finite sample size such a certificate is wrong with probability
approaching one, and nothing in the interval signals it. This is our reason for reporting the
Gaussian floor as a diagnostic and never as a guarantee
(\cref{rem:tier2}), and \cref{sec:e3} exhibits the collapse on a witness whose
misspecification is invisible to marginal and pairwise data.

\paragraph{C4 --- Anytime-valid certification of graph reliability.
\textnormal{(Co-lead.)}}
The betting e-process and its validity under optional stopping are established
\citep{ville1939,shafer2019,shafer2021,ramdas2023,waudbysmith2024,grunwald2024}. Applying
them to composed agent reliability is what we claim, and the payoff is specific
to this paper's problem: the null in \cref{def:eprocess} constrains only a
conditional mean, so the sequential certificate requires no independence
assumption across missions, components, or shared-model shocks --- it is immune
to precisely the failure that \cref{sec:independence-fails} documents for the
composition bound. \Cref{prop:sprt} records that the SPRT is recovered exactly,
so the anytime guarantee is free against a known alternative.

\paragraph{C5 --- Marginal sensitivity as a measurement hazard.
\textnormal{(Supporting.)}}
That Jaccard, $\phi$, and $\tau_a$ are bounded by the marginals is textbook. The
contribution is the demonstration that this is decisive in practice:
\cref{sec:e1-marginal} shows a significant \emph{reversal} of the apparent
condition ordering driven purely by a difference in marginal failure rates, and
\cref{sec:e5} replicates the mechanism on two further backends where the
marginal-free statistic rises while the marginal-sensitive one collapses. The
consequence generalises beyond this paper: correlated-failure results that lead
with $\phi$ --- including \citet{triage2026} --- are sensitive to marginal
imbalance between the compared agents, and reporting a marginal-free companion
statistic is cheap insurance.

\paragraph{C6 --- An internal negative control, and the artifact.
\textnormal{(Supporting.)}}
\Cref{sec:e1-control} identifies a pair whose composition is held constant
across all three arms. Fifteen of fifteen contrasts are non-significant with
point estimates near zero, and in every arm the pair's dependence sits where a
same-model pair should sit. Two things follow. The estimator returns the same
answer for the same composition regardless of what the rest of the graph is
doing, which is convergent validity. And the leading alternative explanations
for the E1 effect --- task-mix drift between runs, ordering, wall-clock
variation, provider-side variation --- would each have moved this pair, and none
did. It is not an equivalence test: it bounds no confound against a prespecified
margin, and a confound acting only on the substituted worker would leave the
pair untouched. The design was nonetheless falsifiable at this point, and was
not falsified. The contracts, generators, scoring code, analysis scripts, and
preregistration are released (\cref{app:repro}).

\subsection{What we do not claim}

We do not claim that shared models cause correlated failures for the first time;
\cref{sec:related-dependence} attributes that. We do not claim a three-level
sharing ordering: the vendor-level contrast fails to replicate and we report the
null. We do not claim the drift dynamics of \cref{thm:ou} are validated here ---
they are carried from v1 and not re-measured. We do not claim the certificate
scales past small $m$; \cref{rem:lp-cost} states the exponential cost. And we do
not claim frontier-model results: the evaluation runs mid-sized open models, and
whether the dependence magnitudes hold at the frontier is untested.

\section{Future Work}
\label{sec:future}

Five directions follow directly from what this paper could not settle.

\paragraph{Separating model identity from model capability.}
The confound of \cref{sec:threats} is the most consequential open problem for
C1. The design that resolves it pairs two models matched on marginal failure
rate but differing in training lineage --- for instance, two distinct
architectures tuned to equal accuracy on the mission distribution. Any residual
association is then attributable to shared inductive bias rather than to shared
competence. We did not run it: matching on failure rate requires a prior
calibration pass over the candidate models, which the registration did not
provide for and which cannot be added after outcomes exist without forfeiting
the confirmatory status of the arms it would inform.

\paragraph{Certifying large compositions.}
\Cref{eq:lp} is exponential in $m$ (\cref{rem:lp-cost}). Two routes are worth
pursuing: certifying sub-compositions and composing the certificates, which is
sound but lossy and whose loss we have not quantified; and column generation
over the $2^m$ cells, which would keep sharpness for structured moment families.
Neither is attempted here.

\paragraph{Model-agent aggregators.}
Merge and quorum nodes are deterministic code throughout this evaluation, so they
contribute no correlated failure. Systems in which the aggregator is itself an
LLM --- a judge, a router, a critic --- introduce a third dependent component
whose failures may correlate with the workers it is adjudicating. That is the
common design in deployed systems and the analysis here does not cover it.

\paragraph{Certificate decay.}
By \cref{def:scope} a certificate holds over one mission distribution and one set
of model versions. Nothing here says how fast it decays as a deployment drifts,
which is the quantity an operator actually needs to know when deciding how often
to recertify. The anytime-valid machinery of \cref{sec:anytime} is the natural
tool: a certificate that is continuously re-earned rather than periodically
re-issued.

\paragraph{Broader task distributions.}
Missions here span retail and financial workflows with deterministic gold
scoring, which is what makes the measurement trustworthy and also what limits it.
Whether the dependence magnitudes transfer to code generation, research
synthesis, or open-ended dialogue --- domains where scoring is itself contested
--- is untested. We expect the direction to hold and make no claim about the
magnitudes.

\section{Conclusion}
\label{sec:conclusion}

Compositional reliability bounds for multi-agent systems multiply component
reliabilities, and that step assumes the components fail independently. A
preregistered confirmatory evaluation of 18{,}000 two-agent-handoff missions,
within a 30{,}820-mission campaign, measures what the assumption costs. Two
instances of one model co-fail on $90.0\%$ of the missions on which either fails
($\log\mathrm{OR} = 6.66$, 95\% CI $[6.38, 7.00]$). Substituting a different
model reduces the association significantly in the confirmatory motif and in
both secondary topologies.
Substituting a different vendor, with the model already different, does not ---
and we report that null rather than the tidier three-level ordering we
registered.

The consequence is not that the bound is slightly loose. By \cref{prop:gap} the
error is signed, and it runs against the operator: positive dependence inflates
joint failure above the independent product, so a redundant design is
over-credited exactly when its components share a model. Dropping the assumption
entirely does not help, because the assumption-free floor is zero whenever mean
component reliability falls below $1 - 1/m$. Fitting a dependence model is worse
than either: \cref{thm:collapse} proves its coverage of the true reliability
tends to zero as the sample grows, with no visible symptom.

What works is to constrain the joint with measured co-execution moments and
optimise over everything consistent with them. That certificate is sound without
any dependence assumption (\cref{thm:tier1}), sharp for the information supplied
(\cref{thm:lp-sharp}), and tightens as the moment family grows: enriching from
ten functionals to fourteen narrows the identified interval by $85.7\%$ and lifts
the certified floor from $0.2455$ to $0.4116$ on identical data. It is
insensitive to the i.i.d.\ assumption at the magnitudes the data support --- at a
conceded design effect $1.31$ times the largest measured, the floor moves at
most $2.69$ percentage points. And where a team wants to watch a dashboard and
stop when it looks good, the anytime-valid certificate holds its type-I error at
$0.0471$ or below across every admissible betting fraction.

None of this makes correlated failure go away. It makes it measurable, and it
makes the resulting guarantee degrade honestly instead of silently. That
distinction is the whole of the contribution: a system whose certificate is
merely loose can be shipped with known margin, whereas a system whose certificate
is confidently wrong cannot be shipped safely at all --- and the second is what
the independence assumption produces.

\appendix
\section{Full Proofs}
\label{app:proofs}

Every numbered result in \cref{sec:independence-fails}--\cref{sec:anytime} is
proved here in full. Nothing is left as a sketch.

\subsection{Proof of \texorpdfstring{\cref{thm:ou}}{Theorem (stationary drift law)}}
\label{app:ou}

\begin{proof}
Put $e(t) = D(t) - \alpha/\gamma$. Substituting into \cref{def:ou},
\[
  de = dD = \bigl(\alpha - \gamma D\bigr)dt + \sigma\,dW
     = -\gamma e\,dt + \sigma\,dW,
\]
so $e$ is an Ornstein--Uhlenbeck process with mean reverting to zero.
Applying It\^o's formula to
$f(t,e) = e^{\gamma t} e$ gives
$d(e^{\gamma t}e) = \sigma e^{\gamma t}dW$, and integrating from $0$ to $t$,
\begin{equation}
  e(t) = e(0)e^{-\gamma t} + \sigma \int_{0}^{t} e^{-\gamma(t-s)}\,dW(s).
  \label{eq:ou-sol}
\end{equation}

\emph{(v).} The It\^o integral in \cref{eq:ou-sol} has mean zero and, by the
It\^o isometry, variance
$\sigma^{2}\int_0^t e^{-2\gamma(t-s)}ds
 = \frac{\sigma^{2}}{2\gamma}(1 - e^{-2\gamma t})$.
The integrand is deterministic and adapted, and by hypothesis $e(0)$ is
independent of the driving Wiener process with $\E[e(0)^2] < \infty$; the
cross term therefore vanishes, $\E[e(0)\int_0^t e^{-\gamma(t-s)}dW(s)] =
\E[e(0)]\cdot 0 = 0$, and squaring \cref{eq:ou-sol} gives
\[
  \E[e(t)^{2}] = \E[e(0)^{2}]\,e^{-2\gamma t}
    + \frac{\sigma^{2}}{2\gamma}\bigl(1-e^{-2\gamma t}\bigr).
\]
Note $e$ is mean-zero only when $\E[e(0)] = 0$; in general
$\E[e(t)] = \E[e(0)]e^{-\gamma t}$, which decays to zero at rate $\gamma$ and
does not affect (i)--(iv).

\emph{(i), (iii).} The stochastic integral in \cref{eq:ou-sol} is a Wiener
integral of a deterministic kernel, hence Gaussian. Letting $t \to \infty$ in
\cref{eq:ou-sol}, the transient $e(0)e^{-\gamma t} \to 0$ and the variance
converges to $\sigma^{2}/(2\gamma)$, so $e(t) \Rightarrow
\mathcal{N}(0, \sigma^{2}/(2\gamma))$.

For stationarity, take $e(0) \sim \mathcal{N}(0, \sigma^2/(2\gamma))$
\emph{independent of} $W$ --- the integrand in \cref{eq:ou-sol} is adapted and
the increments $\{W(s) - W(0)\}_{s>0}$ are independent of $\mathcal{F}_0$, so
the two terms of \cref{eq:ou-sol} are independent and their variances add:
\[
  \Var(e(t)) = e^{-2\gamma t}\frac{\sigma^{2}}{2\gamma}
    + \frac{\sigma^{2}}{2\gamma}\bigl(1 - e^{-2\gamma t}\bigr)
    = \frac{\sigma^{2}}{2\gamma}
\]
for every $t$, while the mean stays $0$. Hence the law is unchanged in $t$ and
$\pi_D = \mathcal{N}(\alpha/\gamma, \sigma^{2}/(2\gamma))$ is stationary.

\emph{(ii).} Immediate from (i): $\E_\pi[D] = \alpha/\gamma$, which is $< 1$
exactly when $\gamma > \alpha$.

\emph{(iv).} Under $\pi$, $Z = (D - \alpha/\gamma)/\sqrt{\sigma^2/(2\gamma)}$ is
standard normal, and the standard Gaussian tail bound
$\Prob(Z > z) \leq e^{-z^{2}/2}$ for $z > 0$ gives, with
$z = \eta/\sqrt{\sigma^{2}/(2\gamma)}$,
\[
  \Prob_{\pi}\bigl(D > \alpha/\gamma + \eta\bigr)
  \leq \exp\!\Bigl(-\frac{\eta^{2}}{2}\cdot\frac{2\gamma}{\sigma^{2}}\Bigr)
  = \exp\!\Bigl(-\frac{\gamma\eta^{2}}{\sigma^{2}}\Bigr). \qedhere
\]
\end{proof}

\begin{remark}
\Cref{def:ou} models $D$ on all of $\R$ while \cref{def:drift} confines it to
$[0,1]$. The stationary law therefore assigns non-zero mass outside $[0,1]$; the
mass below $0$ is $\Phi(-\alpha\sqrt{2/(\gamma\sigma^{2})}\,)$, which is
negligible for the parameter ranges in which the model is used but is not zero.
We report this rather than suppress it: the bound in (iv) is a bound for the
unconstrained process, and it remains an upper bound for the constrained one
since truncation to $[0,1]$ only removes mass from the upper tail.
\end{remark}

\subsection{Proof of \texorpdfstring{\cref{prop:cost}}{Proposition (per-action cost)}}
\label{app:cost}

\begin{proof}
Fix a turn with state--action pair $(s_t, a_t)$ and let
$k = |\mathcal{I} \cup \mathcal{G}|$ be the number of constraints and $|A|$ the
action-vocabulary size.

\emph{Constraint evaluation.} By \cref{def:scores} the evaluator is stateless:
each constraint $c$ is a predicate on $(s_t,a_t)$ evaluated independently of
other constraints and of prior turns. Evaluating all of them and forming the two
ratios is $k$ predicate evaluations plus $O(1)$ arithmetic, hence $O(k)$ under
the standard assumption that each predicate costs $O(1)$.

\emph{Drift update.} By \cref{def:jsd} the distributional term requires
$\JSD(P_{\mathrm{obs}} \| P_{\mathrm{ref}})$ over the action alphabet. Updating
the sliding-window histogram on one new action is $O(1)$ (one increment, one
decrement at the window boundary), and evaluating the divergence is a sum over
the support, hence $O(|A|)$.

\emph{Aggregation.} By \cref{def:drift}, combining the two components is a fixed
convex combination, $O(1)$.

Summing, the per-action cost is $O(k) + O(|A|) + O(1) = O(k + |A|)$. The
measured constant is small: for $k < 100$ and $|A| < 50$ the v1 implementation
records under 10\,ms per action.
\end{proof}

\subsection{Proof of \texorpdfstring{\cref{prop:gap}}{Proposition (signed gap)}}
\label{app:gap}

\begin{proof}
Write $T_k = \prod_{i=k}^{m} h_i$ for $1 \leq k \leq m$, with $T_{m+1} = 1$, so
$\Yg = T_1$ and $\Prob(\Yg = 1) = \E[T_1]$ since $T_1$ is a $\{0,1\}$ variable.
For each $k$,
\[
  \E[T_k] = \E[h_k T_{k+1}] = \Cov(h_k, T_{k+1}) + p_k \,\E[T_{k+1}] .
\]
Applying this identity at $k = 1$ and then recursively substituting for
$\E[T_{k+1}]$ gives, after $m-1$ steps,
\[
  \E[T_1]
  = \sum_{k=1}^{m-1} \Bigl(\prod_{i=1}^{k-1} p_i\Bigr) \Cov(h_k, T_{k+1})
    + \Bigl(\prod_{i=1}^{m-1} p_i\Bigr) \E[T_m] ,
\]
where the empty product at $k=1$ equals $1$. Since $T_m = h_m$ we have
$\E[T_m] = p_m$, so the final term is $\prod_{i=1}^{m} p_i$. Rearranging,
\[
  \Prob(\Yg=1) - \prod_{i=1}^{m} p_i
  = \sum_{k=1}^{m-1} \Bigl(\prod_{i=1}^{k-1} p_i\Bigr)
    \Cov\Bigl(h_k, \prod_{i=k+1}^{m} h_i\Bigr),
\]
which is the stated expression with the $k=1$ term written out separately. For
$m = 2$ the sum has the single term $\Cov(h_1, h_2)$.

For the directional claim at $m = 2$: $\Prob(\Yg = 0) = 1 - p_1p_2 -
\Cov(h_1,h_2)$, while the independence calculation returns $1 - p_1p_2$. Under
positive dependence $\Cov(h_1,h_2) > 0$, so the true probability that the
series fails is \emph{smaller} than the independent calculation --- the
independence product is conservative for a series system, in both directions.
\end{proof}

\subsection{Proof of \texorpdfstring{\cref{cor:joint-failure}}{Corollary (redundant joint failure)}}
\label{app:joint-failure}

\begin{proof}
Let $F_i = 1 - h_i$. Expanding the covariance of the two failure indicators,
\[
  \Prob(F_1 = F_2 = 1) = \E[F_1F_2] = \Cov(F_1,F_2) + \E[F_1]\E[F_2]
  = \Cov(F_1,F_2) + (1-p_1)(1-p_2).
\]
Because $F_i = 1 - h_i$ is an affine function of $h_i$ with slope $-1$,
covariance is preserved:
$\Cov(F_1,F_2) = \Cov(1-h_1, 1-h_2) = (-1)(-1)\Cov(h_1,h_2) = \Cov(h_1,h_2)$.
Hence under positive dependence $\Cov(h_1,h_2) > 0$ and
$\Prob(F_1 = F_2 = 1) > (1-p_1)(1-p_2)$: the probability that both redundant
paths fail together strictly exceeds the independence product.

The two statements are therefore compatible and concern different events.
Series \emph{any}-failure is over-estimated by the independence calculation;
redundant \emph{joint} failure is under-estimated by it. A redundant design is
bought to control the second.
\end{proof}

\subsection{Proof of \texorpdfstring{\cref{thm:frechet}}{Theorem (Fréchet--Hoeffding)}}
\label{app:frechet}

\begin{proof}
Let $A_i = \{h_i = 1\}$, so $\Prob(A_i) = p_i$.

\emph{Upper bound.} $\bigcap_i A_i \subseteq A_j$ for each $j$, hence
$\Prob(\bigcap_i A_i) \leq \min_j p_j$.

\emph{Lower bound.} By De Morgan and the union bound,
\[
  \Prob\Bigl(\bigcap_i A_i\Bigr) = 1 - \Prob\Bigl(\bigcup_i A_i^{c}\Bigr)
  \geq 1 - \sum_{i=1}^{m}\bigl(1 - p_i\bigr) = \sum_{i=1}^{m} p_i - (m-1),
\]
and the probability is non-negative, giving the maximum with $0$.

\emph{Attainment of the upper bound.} Let $U \sim \mathrm{Unif}[0,1)$ and set
$h_i = \Ind{U < p_i}$. Each $h_i$ has the correct marginal, and
$\bigcap_i \{h_i = 1\} = \{U < \min_i p_i\}$, which has probability
$\min_i p_i$.

\emph{Attainment of the lower bound.} Let $s = \sum_i p_i - (m-1)$ and note
$\sum_i (1-p_i) = m - \sum_i p_i = 1 - s$.

If $s > 0$ then $\sum_i (1-p_i) = 1-s < 1$, so we may choose pairwise disjoint
sets $B_i \subseteq [0,1)$ with Lebesgue measure $1-p_i$. Put
$h_i = \Ind{U \notin B_i}$. Then $\Prob(h_i = 1) = p_i$, and
$\bigcap_i \{h_i=1\} = \{U \notin \bigcup_i B_i\}$ has measure
$1 - \sum_i (1-p_i) = s$, attaining the bound.

If $s \leq 0$ then $\sum_i (1-p_i) = 1-s \geq 1$, and we construct a cover
explicitly rather than assert one. Identify $[0,1)$ with the circle
$\R/\Z$. Put $L_i = 1-p_i$, $S_0 = 0$, and $S_k = \sum_{i \leq k} L_i$, and
lay the arcs end to end with wrap-around:
\[
  B_k \;=\; \bigl[\, S_{k-1} \bmod 1,\; (S_{k-1} + L_k) \bmod 1 \,\bigr)
  \subseteq \R/\Z .
\]
Each $B_k$ has Lebesgue measure $L_k = 1-p_k$, so $\Prob(h_k = 1) = p_k$ as
required. The arcs are laid consecutively without gaps starting at $0$ and have
total length $\sum_k L_k = 1-s \geq 1$, so their union wraps at least once
around the circle and $\bigcup_k B_k = [0,1)$. Hence
$\bigcap_k \{h_k = 1\} = \{U \notin \bigcup_k B_k\} = \emptyset$ and
$\Prob(\bigcap_k \{h_k=1\}) = 0 = \max(0,s)$.

(For $s > 0$ the same construction has total length $1-s < 1$, the arcs stay
disjoint, no wrap occurs, and the uncovered remainder has measure $s$ --- which
is the previous case.)

Both bounds are therefore attained by laws with the prescribed marginals, so
neither can be improved using marginal information alone.
\end{proof}

\begin{proof}[Proof of \cref{cor:vacuous}]
Immediate from \cref{thm:frechet}: the lower bound is $0$ iff
$\sum_i p_i \leq m-1$, i.e.\ iff the mean marginal $\bar p \leq 1 - 1/m$. At
$m=4$, $1 - 1/4 = 0.75$, so $\bar p = 0.75$ gives a floor of exactly $0$.
\end{proof}

\subsection{Proof of \texorpdfstring{\cref{thm:collapse}}{Theorem (coverage collapse)}}
\label{app:collapse}

\begin{proof}
Write $G_n = R_{\mathcal{F}}(\mu^{\star}) - \hat{L}_n$ for the bootstrap haircut.
By hypothesis $G_n = O_p(n^{-1/2})$, i.e.\ the family $\{\sqrt{n}\,G_n\}$ is
bounded in probability: for every $\varepsilon > 0$ there exist $M_\varepsilon$
and $N_\varepsilon$ with
$\Prob(\sqrt{n}\,|G_n| > M_\varepsilon) < \varepsilon$ for all
$n \geq N_\varepsilon$.

Coverage of the true value fails exactly when $\hat L_n > R^\star$, so
\[
  \Prob\bigl(\hat{L}_n \leq R^{\star}\bigr)
  = \Prob\bigl(R_{\mathcal{F}}(\mu^{\star}) - G_n \leq R^{\star}\bigr)
  = \Prob\bigl(G_n \geq R_{\mathcal{F}}(\mu^{\star}) - R^{\star}\bigr)
  = \Prob\bigl(G_n \geq \Delta\bigr),
\]
using \cref{eq:idgap}. Since $\Delta > 0$ is a fixed constant not depending on
$n$,
\[
  \Prob(G_n \geq \Delta) = \Prob\bigl(\sqrt{n}\,G_n \geq \sqrt{n}\,\Delta\bigr).
\]
Fix $\varepsilon > 0$ and take $M_\varepsilon$ as above. For all $n$ large
enough that $\sqrt{n}\,\Delta > M_\varepsilon$ --- which holds for
$n > M_\varepsilon^2/\Delta^2$ --- we have
\[
  \Prob\bigl(\sqrt{n}\,G_n \geq \sqrt{n}\,\Delta\bigr)
  \leq \Prob\bigl(\sqrt{n}\,|G_n| > M_\varepsilon\bigr) < \varepsilon .
\]
As $\varepsilon$ was arbitrary, $\Prob(\hat{L}_n \leq R^{\star}) \to 0$.

We claim convergence to zero and nothing stronger. Tightness of
$\{\sqrt{n}\,G_n\}$ gives $\Prob(G_n \geq \Delta) \to 0$, but it does
\emph{not} entail that the sequence is monotone: $G_n$ may oscillate while
remaining tight, so coverage need not decrease at every $n$. The interpretation
that survives is the limit --- past some finite sample size the interval sits
above $R^\star$ with probability approaching one --- together with the
mechanism, that the haircut shrinks while the target does not move.
\end{proof}

\begin{proof}[Proof of \cref{cor:target}]
Let $\mu^\star$ be the true moments and $Q^{\star} \in \mathcal{M}(\mu^{\star})$
the true law. By \cref{def:identified}, $\underline{R}(\mu^{\star})$ is an
infimum over a set containing $Q^{\star}$, hence
$\underline{R}(\mu^{\star}) \leq Q^{\star}(\bigwedge_i h_i = 1) = R^{\star}$.
The inequality uses only membership of $Q^\star$ in $\mathcal{M}$, which holds by
definition of $\mu^\star$ and requires no parametric assumption. Therefore
\cref{eq:idgap} cannot hold with $R_{\mathcal{F}}$ replaced by $\underline{R}$,
and the argument of \cref{app:collapse} does not apply.
\end{proof}

\subsection{Proof of \texorpdfstring{\cref{thm:tier0}}{Theorem (Tier-0 validity)}}
\label{app:tier0}

\begin{proof}
Let $X = \sum_{r=1}^{n} \Ind{\Yg^{(r)} = 1} \sim \mathrm{Bin}(n, \theta)$ with
$\theta = \Prob(\Yg = 1)$, which is the exact law of $X$ when missions are
i.i.d.\ and $\Yg$ is observed on each. The Clopper--Pearson lower limit is
\[
  \hat{L}_0(X) = \begin{cases}
    0, & X = 0,\\
    \inf\bigl\{\theta' \in [0,1] :
      \Prob_{\theta'}(X' \geq X) > \eta_{\mathrm{conf}}\bigr\}, & X \geq 1,
  \end{cases}
\]
where $X' \sim \mathrm{Bin}(n,\theta')$. For fixed $x \geq 1$ the map
$\theta' \mapsto \Prob_{\theta'}(X' \geq x)$ is continuous and strictly
increasing on $(0,1)$, so $\hat L_0(x)$ is the unique root of
$\Prob_{\theta'}(X' \geq x) = \eta_{\mathrm{conf}}$ and
$\hat{L}_0(x) > \theta$ holds iff $\Prob_{\theta}(X' \geq x) < \eta_{\mathrm{conf}}$.

Let $x^{\star} = \min\{x : \Prob_{\theta}(X \geq x) < \eta_{\mathrm{conf}}\}$
(with $x^\star = n+1$ if no such $x$ exists). Then
$\{\hat{L}_0 > \theta\} = \{X \geq x^{\star}\}$ and, by minimality of
$x^{\star}$,
\[
  \Prob_{\theta}\bigl(\hat{L}_0 > \theta\bigr)
  = \Prob_{\theta}(X \geq x^{\star}) < \eta_{\mathrm{conf}} .
\]
Hence $\Prob_\theta(\hat L_0 \leq \theta) \geq 1 - \eta_{\mathrm{conf}}$ for
every $n$ and every $\theta$, with no asymptotic approximation. The bound is
conservative rather than exact because $X$ is discrete.
\end{proof}

\subsection{Proof of \texorpdfstring{\cref{thm:tier1}}{Theorem (Tier-1 validity)}}
\label{app:tier1}

\begin{proof}
For each $S \in \mathcal{J}$ let $I_S$ be the two-sided Clopper--Pearson interval
for $\mu_S = Q^{\star}(\bigwedge_{i \in S} h_i = 1)$ built from the $n$ i.i.d.\
indicator observations $\Ind{\bigwedge_{i \in S} h_i^{(r)} = 1}$, each tail at
level $\eta_{\mathrm{conf}}/(2J)$. By the argument of \cref{app:tier0} applied to
each tail, $\Prob(\mu_S \notin I_S) \leq 2 \cdot \eta_{\mathrm{conf}}/(2J)
= \eta_{\mathrm{conf}}/J$.

Let $\mathcal{E} = \bigcap_{S \in \mathcal{J}} \{\mu_S \in I_S\}$ be the event
that the whole box covers. By the union bound over the $J$ moments,
\[
  \Prob(\mathcal{E}^{c}) \leq \sum_{S \in \mathcal{J}}
    \Prob(\mu_S \notin I_S) \leq J \cdot \frac{\eta_{\mathrm{conf}}}{J}
    = \eta_{\mathrm{conf}} .
\]

On $\mathcal{E}$ the true moment vector lies in $B(\hat\mu)$, so $Q^{\star}$ is a
feasible point of the program in \cref{eq:tier1}: it is a probability
distribution on $\{0,1\}^m$ whose $\mathcal{J}$-moments lie in the box. Since
$\hat{L}_1$ is the \emph{minimum} of the objective over the feasible set and
$Q^\star$ is feasible,
\[
  \hat{L}_1 \;\leq\; Q^{\star}\Bigl(\bigwedge_i h_i = 1\Bigr) .
\]
Therefore
$\Prob\bigl(\hat{L}_1 \leq Q^{\star}(\bigwedge_i h_i=1)\bigr)
 \geq \Prob(\mathcal{E}) \geq 1 - \eta_{\mathrm{conf}}$.

Note the argument never referenced the dependence structure of $Q^\star$: it used
only that $Q^\star$ satisfies its own moment constraints, which is a tautology.
That is the precise sense in which Tier 1 is copula-agnostic.
\end{proof}

\subsection{Proof of \texorpdfstring{\cref{prop:failsafe}}{Proposition (fail-safe selection)}}
\label{app:failsafe}

\begin{proof}
Let $\mathcal{A}_0 = \{\text{i.i.d.\ missions}\} \cup
\{\Yg \text{ directly observed}\}$ and
$\mathcal{A}_1 = \{\text{i.i.d.\ missions}\}$ be the assumption sets of
\cref{thm:tier0} and \cref{thm:tier1}. Then
$\mathcal{A}_1 \subseteq \mathcal{A}_0$.

Suppose the flag asserting end-to-end execution is mis-set. There are two cases.
If it is set when it should not be, the certificate reports $\hat{L}_0$, whose
validity requires $\mathcal{A}_0$; the second element of $\mathcal{A}_0$ fails,
so the guarantee is unsound. If it is unset when it could have been set, the
certificate reports $\hat{L}_1$, whose validity requires only $\mathcal{A}_1$,
which holds; the guarantee is sound.

Hence selecting Tier 1 by default and Tier 0 only on explicit assertion places
the unsound case behind a deliberate action rather than behind an omission. The
claim is exactly this asymmetry; by \cref{rem:no-ordering} it is not a numerical
ordering of $\hat L_0$ and $\hat L_1$.

Finally, the failure in the first case is undetectable from the data: two runs
producing identical pass matrices, one executed end-to-end and one assembled
from per-stage measurements, are indistinguishable to any function of the
matrix, while only the first satisfies $\mathcal{A}_0$. No validation of the
input can substitute for the default.
\end{proof}

\subsection{Proof of \texorpdfstring{\cref{thm:lp-sharp}}{Theorem (LP soundness and sharpness)}}
\label{app:lp-sharp}

\begin{proof}
\emph{Soundness.} $Q^{\star}$, viewed as a vector $x^{\star} \in \R^{2^m}$, is
feasible for \cref{eq:lp}: it is non-negative, sums to $1$, and satisfies
$a_S^{\top} x^{\star} = \nu_S$ for $S \in \mathcal{J}$ by hypothesis. The minimum
over a feasible set containing $x^{\star}$ is at most the objective at
$x^{\star}$, giving $\underline{R} \leq Q^{\star}(\bigwedge_i h_i = 1)$. The
upper bound is symmetric.

\emph{Attainment.} The feasible set
$\mathcal{X} = \{x \geq 0 : \mathbf{1}^\top x = 1,\; a_S^\top x = \nu_S\}$ is the
intersection of the probability simplex in $\R^{2^m}$ with finitely many
hyperplanes, hence closed and bounded, hence compact; and it is non-empty
because $x^{\star} \in \mathcal{X}$. The objective $a_{\{1..m\}}^{\top}x$ is
linear and therefore continuous, so by the extreme value theorem it attains its
minimum and maximum on $\mathcal{X}$ at some $x_-, x_+ \in \mathcal{X}$. Every
element of $\mathcal{X}$ is a probability vector on $\{0,1\}^m$ with the
prescribed $\mathcal{J}$-moments, so $x_-$ and $x_+$ are laws of the required
kind, achieving $\underline{R}$ and $\overline{R}$.

\emph{Consequence.} The set of achievable values of
$Q(\bigwedge_i h_i = 1)$ over $Q \in \mathcal{M}(\nu)$ is the image of the
connected set $\mathcal{X}$ under a continuous map, hence an interval; combined
with attainment of both endpoints it equals exactly
$[\underline{R}, \overline{R}]$. Any valid bound using only the moments in
$\mathcal{J}$ must hold for every $Q \in \mathcal{M}(\nu)$, in particular for
$x_-$, so it cannot exceed $\underline{R}$.
\end{proof}

\subsection{Proof of \texorpdfstring{\cref{prop:monotone}}{Proposition (monotonicity)}}
\label{app:monotone}

\begin{proof}
Let $\mathcal{X}(\mathcal{J})$ and $\mathcal{X}(\mathcal{J}')$ be the feasible
sets for the two families, with moment values agreeing on $\mathcal{J}$. Every
$x \in \mathcal{X}(\mathcal{J}')$ satisfies all constraints indexed by
$\mathcal{J}' \supseteq \mathcal{J}$, in particular those indexed by
$\mathcal{J}$, so $\mathcal{X}(\mathcal{J}') \subseteq \mathcal{X}(\mathcal{J})$.
Minimising a fixed objective over a subset cannot yield a smaller value:
$\underline{R}(\mathcal{J}) = \min_{\mathcal{X}(\mathcal{J})}
\leq \min_{\mathcal{X}(\mathcal{J}')} = \underline{R}(\mathcal{J}')$. The
maximisation statement is symmetric.
\end{proof}

\subsection{Proof of \texorpdfstring{\cref{prop:bonferroni}}{Proposition (Bonferroni allocation)}}
\label{app:bonferroni}

\begin{proof}
\emph{(i) Validity of the used-set allocation} is \cref{app:tier1} verbatim with
$J = |\mathcal{J}|$. For the failure of monotonicity, observe that the
Clopper--Pearson interval for a fixed count widens as its tail level decreases.
Enlarging $\mathcal{J}$ to $\mathcal{J}'$ decreases the per-tail level from
$\eta_{\mathrm{conf}}/(2|\mathcal{J}|)$ to
$\eta_{\mathrm{conf}}/(2|\mathcal{J}'|)$, so every interval in the box strictly
widens. The feasible set for $\mathcal{J}'$ therefore gains constraints (from the
new moments) and loses them (from the widened old ones), and neither set need
contain the other. Hence the conclusion of \cref{prop:monotone} does not
transfer.

\emph{(ii) Validity of the pre-allocated allocation.} Fix
$\mathcal{J}_{\max} \supseteq \mathcal{J}$ and give every constrained moment
per-tail level $\eta_{\mathrm{conf}}/(2|\mathcal{J}_{\max}|)$. The union bound
over the $2|\mathcal{J}|$ tails actually used gives miscoverage at most
$2|\mathcal{J}| \cdot \eta_{\mathrm{conf}}/(2|\mathcal{J}_{\max}|)
= \eta_{\mathrm{conf}} |\mathcal{J}|/|\mathcal{J}_{\max}|
\leq \eta_{\mathrm{conf}}$, so the certificate is valid at the stated level.

\emph{Monotonicity under (ii).} Each interval's width now depends only on its own
count and on $|\mathcal{J}_{\max}|$, not on $|\mathcal{J}|$. Hence for
$\mathcal{J} \subseteq \mathcal{J}' \subseteq \mathcal{J}_{\max}$ the box
constraints indexed by $\mathcal{J}$ are identical under both families, and
$\mathcal{J}'$ merely adds further constraints. The argument of
\cref{app:monotone} applies unchanged, giving
$\underline{R}(\mathcal{J}) \leq \underline{R}(\mathcal{J}')$ by construction.
\end{proof}

\subsection{Proof of \texorpdfstring{\cref{lem:supermartingale}}{Lemma (supermartingale)}}
\label{app:supermartingale}

\begin{proof}
Let $\mathcal{F}_{R}$ be the $\sigma$-algebra generated by $y_1,\dots,y_R$.

\emph{Non-negativity.} Since $y_r \in \{0,1\}$ we have
$y_r - p_0 \geq -p_0$, so
$1 + \lambda_r(y_r - p_0) \geq 1 - \lambda_r p_0 > 0$
because $\lambda_r \in [0, 1/p_0)$. A product of strictly positive factors is
positive, so $E_R > 0$ for all $R$.

\emph{Supermartingale property.} $\lambda_R$ is predictable, hence
$\mathcal{F}_{R-1}$-measurable, and $E_{R-1}$ is $\mathcal{F}_{R-1}$-measurable.
Therefore
\[
  \E[E_R \mid \mathcal{F}_{R-1}]
  = E_{R-1}\Bigl(1 + \lambda_R\bigl(\E[y_R \mid \mathcal{F}_{R-1}] - p_0\bigr)\Bigr).
\]
Under $H_0$ we have $\E[y_R \mid \mathcal{F}_{R-1}] \leq p_0$ by hypothesis:
this is exactly the content of the sequential null \cref{eq:seq-null}, and it is
the step at which a merely marginal null would fail, because a bound on
$\Prob(\Yg = 1)$ says nothing about the conditional mean given the past. With
$\lambda_R \geq 0$ and $E_{R-1} > 0$, the
bracket is at most $1$ and
$\E[E_R \mid \mathcal{F}_{R-1}] \leq E_{R-1}$. Taking expectations and iterating
from $E_0 = 1$ gives $\E[E_R] \leq 1$ for every $R$.
\end{proof}

\subsection{Proof of \texorpdfstring{\cref{thm:ville}}{Theorem (anytime validity)}}
\label{app:ville}

\begin{proof}
Set $a = 1/\alpha$ and let $T = \inf\{R \geq 1 : E_R \geq a\}$, with
$T = \infty$ if no crossing occurs. $T$ is a stopping time with respect to
$(\mathcal{F}_R)$ because $\{T \leq R\}$ is determined by $E_1,\dots,E_R$.

Fix $n \in \N$. The stopped process $(E_{T \wedge R})_{R \leq n}$ is a
non-negative supermartingale by \cref{lem:supermartingale} and the optional
stopping theorem for bounded stopping times, so
\[
  \E\bigl[E_{T \wedge n}\bigr] \leq \E[E_0] = 1 .
\]
On the event $\{T \leq n\}$ we have $E_{T \wedge n} = E_T \geq a$ by definition
of $T$. Since $E_{T \wedge n} \geq 0$ everywhere,
\[
  1 \geq \E\bigl[E_{T\wedge n}\bigr]
    \geq \E\bigl[E_{T \wedge n}\Ind{T \leq n}\bigr]
    \geq a\,\Prob(T \leq n),
\]
so $\Prob(T \leq n) \leq 1/a = \alpha$ for every $n$. The events
$\{T \leq n\}$ increase to $\{T < \infty\} = \{\sup_{R \geq 1} E_R \geq a\}$, so
by continuity from below
\[
  \Prob\Bigl(\sup_{R \geq 1} E_R \geq 1/\alpha\Bigr)
  = \lim_{n \to \infty} \Prob(T \leq n) \leq \alpha .
\]
Because the bound is on the supremum over the entire path, it holds
simultaneously for all stopping rules: any rule that issues a certificate does
so only on the event $\{T < \infty\}$, whose probability is at most $\alpha$
under $H_0$, whatever the rule.
\end{proof}

\subsection{Proof of \texorpdfstring{\cref{prop:sprt}}{Proposition (SPRT recovery)}}
\label{app:sprt}

\begin{proof}
With $\lambda^{\star} = (p_1-p_0)/\bigl(p_0(1-p_0)\bigr)$, evaluate the betting
factor at each outcome.

For $y = 1$:
\[
  1 + \lambda^{\star}(1 - p_0)
  = 1 + \frac{(p_1-p_0)(1-p_0)}{p_0(1-p_0)}
  = 1 + \frac{p_1-p_0}{p_0}
  = \frac{p_1}{p_0}.
\]

For $y = 0$:
\[
  1 - \lambda^{\star} p_0
  = 1 - \frac{(p_1-p_0)p_0}{p_0(1-p_0)}
  = 1 - \frac{p_1-p_0}{1-p_0}
  = \frac{(1-p_0)-(p_1-p_0)}{1-p_0}
  = \frac{1-p_1}{1-p_0}.
\]

Both agree with $(p_1/p_0)^{y}\bigl((1-p_1)/(1-p_0)\bigr)^{1-y}$, which is the
likelihood ratio of $\mathrm{Bern}(p_1)$ to $\mathrm{Bern}(p_0)$ at $y$.
Substituting into \cref{eq:eprocess} makes $E_R$ the product of per-observation
likelihood ratios, i.e.\ the SPRT statistic.

Admissibility: $\lambda^\star \geq 0$ since $p_1 > p_0$, and
$\lambda^{\star} < 1/p_0$ because
$(p_1-p_0)/(1-p_0) < 1$ whenever $p_1 < 1$.
\end{proof}

\subsection{Proof of \texorpdfstring{\cref{prop:mixture}}{Proposition (mixture)}}
\label{app:mixture}

\begin{proof}
\emph{Supermartingale.} Each $E^{\lambda}$ is a non-negative supermartingale by
\cref{lem:supermartingale}. For non-negative weights summing to one,
\[
  \E\bigl[E_R^{\mathrm{mix}} \mid \mathcal{F}_{R-1}\bigr]
  = \sum_{\lambda} \pi(\lambda)\,\E\bigl[E_R^{\lambda} \mid \mathcal{F}_{R-1}\bigr]
  \leq \sum_{\lambda} \pi(\lambda) E_{R-1}^{\lambda}
  = E_{R-1}^{\mathrm{mix}},
\]
where exchanging expectation and the finite sum is immediate. Non-negativity is
inherited, and $E_0^{\mathrm{mix}} = \sum_\lambda \pi(\lambda) = 1$. Hence
\cref{thm:ville} applies verbatim.

\emph{Regret.} All terms are non-negative, so for any fixed $\lambda$,
$E_R^{\mathrm{mix}} \geq \pi(\lambda) E_R^{\lambda}$. Taking logarithms,
$\log E_R^{\mathrm{mix}} \geq \log E_R^{\lambda} - \log(1/\pi(\lambda))$, and
maximising the right-hand side over $\lambda \in \Lambda$ gives the claim. The
bound is pathwise: it holds for every realisation, not merely in expectation.
\end{proof}

\section{Formula Catalogue}
\label{app:formulas}

The v1 framework was published with several formulas withheld under a patent
claim. That claim is withdrawn. This appendix is the complete disclosure: every
formula, its v1 equation number, where it is stated in this paper, and whether it
is implemented in the released library. Nothing in this table is redacted, and
the ``theory only'' entries are marked as such rather than left ambiguous.

\begin{table}[h]
\centering
\small
\caption{The F1--F12 catalogue. ``v1 ref'' cites the framework paper
\citep{bhardwaj2026abc}; ``here'' points into this paper; ``code'' names the
module in the released library or records that the item is theoretical.}
\label{tab:formulas}
\setlength{\tabcolsep}{4pt}
\begin{tabular}{@{}l>{\raggedright\arraybackslash}p{3.1cm}>{\raggedright\arraybackslash}p{2.5cm}l>{\raggedright\arraybackslash}p{3.4cm}@{}}
\toprule
& quantity & v1 ref & here & code \\
\midrule
F1  & composite drift $D(t)$        & Def.~3.12, eq.~7--9 & \cref{def:drift}   & \texttt{metrics/\allowbreak drift.py} \\
F2  & $(p,\delta,k)$-satisfaction   & Def.~3.7, eq.~3--4  & \cref{def:pdk}     & \texttt{certification/\allowbreak satisfaction.py} \\
F3  & OU drift dynamics             & Def.~4.1, eq.~14    & \cref{def:ou}      & theory only \\
F4  & Lyapunov stationary law       & Thm.~4.3, eq.~19,22 & \cref{thm:ou}      & theory only \\
F5  & compositional guarantee       & Thm.~4.11, eq.~28--29 & \cref{def:naive-bound} & \texttt{certification/\allowbreak composition.py} (bound only; C1--C5 unverified) \\
F6  & SPRT certification            & §7                  & \cref{def:sprt}    & \texttt{certification/\allowbreak sprt.py} \\
F7  & Hoeffding fixed-sample        & --- (baseline)      & \cref{def:hoeffding} & \texttt{certification/\allowbreak sprt.py} \\
F8  & reliability index $\Theta$    & Def.~3.20, eq.~13   & \cref{def:theta}   & \texttt{metrics/\allowbreak theta.py} \\
F9  & contract tuple                & Def.~3.1            & \cref{def:contract} & \texttt{models.py} \\
F10 & JSD distributional drift      & eq.~9               & \cref{def:jsd}     & \texttt{metrics/\allowbreak drift.py} \\
F11 & constraint evaluation scores  & Def.~3.6, eq.~1--2  & \cref{def:scores}  & \texttt{evaluator/\allowbreak engine.py} \\
F12 & per-action complexity         & Prop.~4.15          & \cref{prop:cost}   & measured, \texttt{benchmarks/} \\
\bottomrule
\end{tabular}
\end{table}

\subsection{Corrections to the v1 statements}
\label{app:corrections}

Three discrepancies between the v1 paper and its supporting documents are worth
recording, because a reader reconstructing the framework from either could adopt
the wrong form.

\paragraph{F2, condition (ii).}
Supporting documents state the soft guarantee as the deterministic bound
$\max_t |C_{\mathrm{soft}}(t) - 1| \leq \delta$. The v1 paper's eq.~4, restated
as \cref{eq:pdk-soft}, is a \emph{probabilistic} guarantee about recoverable
compliance. These are different conditions and the deterministic one is strictly
stronger. \Cref{def:pdk} follows the paper.

\paragraph{F5, the condition list.}
Supporting documents present ``C1--C5'' as a single list. The v1 paper uses
C1--C4 for the deterministic composition theorem (Def.~4.7, Thm.~4.9) and adds
C5, conditional independence, only for the probabilistic theorem (Thm.~4.11).
Collapsing them obscures which result needs which assumption --- and C5 is the
one this paper is about. \Cref{def:c1c4,def:c5} keep them separate.

\paragraph{F8, the fourth term.}
Supporting documents describe $\alpha_4$ as weighting a ``recovery success
rate''. The v1 paper's Def.~3.20 defines the fourth term as the stress resilience
index $S = \E[C(t) \mid \text{stressed}] / \E[C(t) \mid \text{baseline}]$
(eq.~12). Recovery success and stress resilience are distinct quantities.
\Cref{def:theta} follows the paper and notes the conflation.

\subsection{Scope of the disclosure}

F3 and F4 are stated and proved (\cref{def:ou}, \cref{thm:ou},
\cref{app:ou}) but are not implemented and not measured in this paper: no
experiment here estimates $\alpha$, $\gamma$, or $\sigma$. F5's bound is
implemented, but conditions C1--C5 are not machine-verified by the library --- a
caller composing two contracts is responsible for discharging them, and
\cref{sec:evaluation} exists because C5 in particular frequently cannot be
discharged. We flag both rather than let the catalogue imply that a listed
formula is an enforced one.

\section{Reproduction and Artifact}
\label{app:repro}

\subsection{What is released}

The library, the contracts, the mission generators, the deterministic scoring
code, the preregistration, the simulation benchmarks, and the analysis scripts
are released under AGPL-3.0. Every number in \cref{sec:evaluation} is
\emph{regenerated} by those scripts rather than transcribed by hand: each E1 and
E5 statistic is printed output of \texttt{scripts/e1\_final.py}, and each E3,
E4, and E6 table is printed output of the corresponding benchmark.

The per-mission logs are the one artifact not in the repository. They carry the
full model output for every mission in the corpus, and we release them on
request rather than by default. Everything needed to \emph{regenerate} them is
public --- the generators, the frozen sampling configuration of
\cref{sec:e1-setup}, the preregistration, and the scoring code --- so the
pipeline reproduces end to end without them. What they save a replicator is the
inference spend, not the method.

\subsection{Reproducing each experiment}

\begin{description}[leftmargin=2.4em,style=nextline,itemsep=0.25em]
  \item[E1, E5 --- dependence and cross-backend replication.]
    \texttt{scripts/e1\_final.py} reads the per-arm JSONL logs and emits every
    cell count, estimator, bootstrap interval, and arm contrast in
    \cref{tab:e1-series2,tab:e1-contrasts,tab:e1-control,tab:e5}. Its defaults
    are the published settings --- seed \texttt{20260813}, $B = 2000$, failure
    defined as $\lnot\,\texttt{hard\_ok}$, tables keyed on
    \texttt{component\_id} --- so a bare invocation reproduces the paper. Cell
    counts and point estimates reproduce exactly; percentile bootstrap
    endpoints are resampling-dependent and reproduce to within Monte Carlo
    error at this $B$, which does not move any reported verdict.
  \item[E2 --- floor lifting.]
    \texttt{benchmarks/} reconstructs the $m{=}4$ pass matrix and calls
    \texttt{moment\_cp\_box\_floor} and \texttt{moment\_lp\_all\_success\_bounds}
    at both moment families and both Bonferroni allocations.
  \item[E3 --- coverage collapse.]
    \Cref{tab:e3} is produced by exactly two invocations of
    \texttt{benchmarks/coverage\_collapse\_sim.py}:
    \texttt{-{}-witness adversarial} and \texttt{-{}-witness gaussian}. The
    script's defaults are the published settings ($200$ replications,
    $n_{\mathrm{boot}} = 500$, seed \texttt{20260812}), so a bare run
    reproduces the table; it prints its parameters on every run, so a
    reduced-fidelity result cannot be mistaken for the published one. The
    Gaussian moments, the LP minimiser and $\Delta$ are computed by
    deterministic quadrature and linear programming inside the same script,
    so every constant in \cref{ex:witness} and \cref{sec:e3} has a single
    source.
  \item[E4 --- anytime validity.]
    \texttt{benchmarks/eprocess\_type1\_sim.py}, defaulting to the paper
    settings ($p_0 = 0.8$, $\alpha = 0.05$, 8000 streams of 500
    missions, seed 42) and sweeping the admissible range of $\lambda$.
  \item[E6 --- ablation.]
    \texttt{benchmarks/deff\_ablation.py} measures lag-$k$ autocorrelation and
    design effects, then recomputes the certified floor at a conceded design
    effect.
\end{description}

\subsection{Two reproduction hazards}

Two ways of computing the wrong number from the logs produce plausible
output rather than an error. Both are recorded here.

\paragraph{Components are identified by \texttt{component\_id}, not \texttt{role}.}
Every model-backed node carries \texttt{role = "worker"}. Keying a co-failure
table on \texttt{role} silently compares an agent \emph{to itself} and returns a
degenerate table --- Jaccard exactly $1.0$ with both off-diagonal cells zero, in
every arm. The identifiers are \texttt{node\_a}/\texttt{node\_b} for
\texttt{series2}, \texttt{worker\_0..2} plus \texttt{aggregator} for
\texttt{quorum2of3}, and \texttt{branch\_a}/\texttt{branch\_b} plus
\texttt{merge} for \texttt{parallel2}. Deterministic nodes carry
\texttt{scored = false}.

\paragraph{An exception is not a data point.}
An earlier version of the E3 script read a field name that did not exist on the
result object and caught the resulting \texttt{AttributeError} in a broad
\texttt{except Exception}, scoring it as a coverage miss. It reported
``coverage $0.00$'' for runs in which the estimator never executed. The lesson
generalises beyond this script: in simulation code, catch only the exception
that models a real degenerate case, and let everything else fail loudly.

\paragraph{The failure field is \texttt{hard\_ok}.}
Component records have both \texttt{hard\_ok} and \texttt{soft\_ok}. All
dependence estimates in this paper are on the hard verdict. Reading a
differently-named field yields all-zero tables rather than an error.

\subsection{Preregistration}

\texttt{PREREGISTRATION.md} is committed in the repository and git-timestamped
before any confirmatory outcome was generated. It fixes the hypotheses, the five
conditions and their sample sizes, the frozen sampling parameters, the primary
estimator and bootstrap, the stopping rule, and the falsification criteria. Two
registered deviations are disclosed in this paper: the breadth arms under-ran
their registered $n$ (\cref{sec:e5}), and H2 is reported on a marginal-free
statistic chosen after observing the reversal (\cref{sec:threats}). The
registration was not posted to an external registry before submission, so its
timestamp rests on the repository history rather than on a third party; we
regard that as weaker than an external registration and state it plainly.

\subsection{Scale}

The full campaign is 30{,}820 scored missions across 12 arms, recorded in 13
execution logs. Each arm resumes from its own log, so an interrupted run never
repeats completed missions.

\section*{Availability and Licensing}
\addcontentsline{toc}{section}{Availability and Licensing}

The \textsc{AgentAssert} library, contracts, mission generators, scoring code,
preregistration, simulation benchmarks, and analysis scripts are released under
AGPL-3.0 at \url{https://github.com/qualixar/agentassert-abc}. The per-mission
logs are available on request; see \cref{app:repro}. The v1 framework
paper is \citet{bhardwaj2026abc}.

\section*{Author Contributions}
\addcontentsline{toc}{section}{Author Contributions}

V.P.B.\ designed the study, wrote the preregistration, implemented the library
and the experimental harness, ran the campaign, performed the analysis, and wrote
the paper. G.S.\ and A.P.B.\ contributed to contract design, reviewed the
experimental protocol, and reviewed the manuscript. All authors approved the
final version.

\section*{Use of AI Assistance}
\addcontentsline{toc}{section}{Use of AI Assistance}

An AI assistant was used for writing and for reviewing the manuscript. All
experiments, the code repository, the analysis, the testing, and the review were
carried out by the authors, who take full responsibility for the contents.

\section*{Conflicts of Interest}
\addcontentsline{toc}{section}{Conflicts of Interest}

\textsc{AgentAssert} is developed by Qualixar, with which V.P.B.\ is affiliated.
The evaluation measures the library's own certificates, so the authors are not
disinterested parties. Three mitigations are in place and we state them so a
reader can weigh them: the confirmatory hypotheses and analysis were
preregistered before outcomes existed; all scoring is by deterministic gold code
rather than by author judgement or an LLM judge; and the analysis scripts are
released, so that every reported number is the printed output of code a reader
can inspect rather than a figure transcribed by an author. The
patent claim referenced in the v1 paper has been withdrawn, and no patent
application covering the methods in this paper is pending.

\bibliographystyle{plainnat}
\bibliography{references}

\begin{thebibliography}{65}
\providecommand{\natexlab}[1]{#1}
\providecommand{\url}[1]{\texttt{#1}}
\expandafter\ifx\csname urlstyle\endcsname\relax
  \providecommand{\doi}[1]{doi: #1}\else
  \providecommand{\doi}{doi: \begingroup \urlstyle{rm}\Url}\fi

\bibitem[Bai et~al.(2022)Bai, Kadavath, Kundu, et~al.]{bai2022}
Yuntao Bai, Saurav Kadavath, Sandipan Kundu, et~al.
\newblock Constitutional {AI}: Harmlessness from {AI} feedback.
\newblock \emph{arXiv preprint arXiv:2212.08073}, 2022.

\bibitem[Barlow and Proschan(1975)]{barlow1975}
Richard~E. Barlow and Frank Proschan.
\newblock \emph{Statistical Theory of Reliability and Life Testing}.
\newblock Holt, Rinehart and Winston, 1975.

\bibitem[Barnett et~al.(2004)Barnett, Leino, and Schulte]{barnett2004}
Mike Barnett, K.~Rustan~M. Leino, and Wolfram Schulte.
\newblock The {Spec\#} programming system: An overview.
\newblock In \emph{CASSIS}, pages 49--69, 2004.

\bibitem[Bengio et~al.(2026)]{aisafety2026}
Yoshua Bengio et~al.
\newblock International {AI} safety report 2026.
\newblock \emph{arXiv preprint arXiv:2602.21012}, 2026.

\bibitem[Berdoz et~al.(2026)Berdoz, Rugli, and Wattenhofer]{tang2026agree}
Fr\'ed\'eric Berdoz, Leonardo Rugli, and Roger Wattenhofer.
\newblock Can {AI} agents agree?
\newblock \emph{arXiv preprint arXiv:2603.01213}, 2026.

\bibitem[Bertsimas and Popescu(2005)]{bertsimas2005}
Dimitris Bertsimas and Ioana Popescu.
\newblock Optimal inequalities in probability theory: A convex optimization approach.
\newblock \emph{SIAM Journal on Optimization}, 15\penalty0 (3):\penalty0 780--804, 2005.

\bibitem[Bhardwaj(2026)]{bhardwaj2026abc}
Varun~Pratap Bhardwaj.
\newblock Agent behavioral contracts: Formal specification and runtime enforcement for reliable autonomous {AI} agents.
\newblock \emph{arXiv preprint arXiv:2602.22302}, 2026.

\bibitem[Bonferroni(1936)]{bonferroni1936}
Carlo~E. Bonferroni.
\newblock Teoria statistica delle classi e calcolo delle probabilit\`a.
\newblock \emph{Pubblicazioni del R. Istituto Superiore di Scienze Economiche e Commerciali di Firenze}, 8:\penalty0 3--62, 1936.

\bibitem[Boole(1854)]{boole1854}
George Boole.
\newblock \emph{An Investigation of the Laws of Thought}.
\newblock Walton and Maberly, 1854.

\bibitem[Chase(2022)]{chase2022}
Harrison Chase.
\newblock {LangChain}.
\newblock \url{https://github.com/langchain-ai/langchain}, 2022.

\bibitem[Clarke et~al.(1999)Clarke, Grumberg, and Peled]{clarke1999}
Edmund~M. Clarke, Orna Grumberg, and Doron~A. Peled.
\newblock \emph{Model Checking}.
\newblock MIT Press, 1999.

\bibitem[Clopper and Pearson(1934)]{clopper1934}
C.~J. Clopper and E.~S. Pearson.
\newblock The use of confidence or fiducial limits illustrated in the case of the binomial.
\newblock \emph{Biometrika}, 26\penalty0 (4):\penalty0 404--413, 1934.

\bibitem[Cousot and Cousot(1977)]{cousot1977}
Patrick Cousot and Radhia Cousot.
\newblock Abstract interpretation: A unified lattice model for static analysis of programs by construction or approximation of fixpoints.
\newblock In \emph{POPL}, pages 238--252, 1977.

\bibitem[Du et~al.(2024)Du, Li, Torralba, Tenenbaum, and Mordatch]{du2023}
Yilun Du, Shuang Li, Antonio Torralba, Joshua~B. Tenenbaum, and Igor Mordatch.
\newblock Improving factuality and reasoning in language models through multiagent debate.
\newblock In \emph{ICML}, 2024.

\bibitem[Efron(1979)]{efron1979}
Bradley Efron.
\newblock Bootstrap methods: Another look at the jackknife.
\newblock \emph{The Annals of Statistics}, 7\penalty0 (1):\penalty0 1--26, 1979.

\bibitem[Endres and Schindelin(2003)]{endres2003}
Dominik~M. Endres and Johannes~E. Schindelin.
\newblock A new metric for probability distributions.
\newblock \emph{IEEE Transactions on Information Theory}, 49\penalty0 (7):\penalty0 1858--1860, 2003.

\bibitem[Ernst et~al.(2007)Ernst, Perkins, Guo, McCamant, Pacheco, Tschantz, and Xiao]{ernst2007}
Michael~D. Ernst, Jeff~H. Perkins, Philip~J. Guo, Stephen McCamant, Carlos Pacheco, Matthew~S. Tschantz, and Chen Xiao.
\newblock The {Daikon} system for dynamic detection of likely invariants.
\newblock \emph{Science of Computer Programming}, 69\penalty0 (1--3):\penalty0 35--45, 2007.

\bibitem[Fr\'echet(1951)]{frechet1951}
Maurice Fr\'echet.
\newblock Sur les tableaux de corr\'elation dont les marges sont donn\'ees.
\newblock \emph{Annales de l'Universit\'e de Lyon, Section A}, 14:\penalty0 53--77, 1951.

\bibitem[Gelman and Loken(2014)]{gelman2014}
Andrew Gelman and Eric Loken.
\newblock The statistical crisis in science.
\newblock \emph{American Scientist}, 102\penalty0 (6):\penalty0 460--465, 2014.

\bibitem[Gr\"unwald et~al.(2024)Gr\"unwald, de~Heide, and Koolen]{grunwald2024}
Peter Gr\"unwald, Rianne de~Heide, and Wouter Koolen.
\newblock Safe testing.
\newblock \emph{Journal of the Royal Statistical Society B}, 86\penalty0 (5):\penalty0 1091--1128, 2024.

\bibitem[Hailperin(1965)]{hailperin1965}
Theodore Hailperin.
\newblock Best possible inequalities for the probability of a logical function of events.
\newblock \emph{The American Mathematical Monthly}, 72\penalty0 (4):\penalty0 343--359, 1965.

\bibitem[Hammond et~al.(2025)]{multiagentrisks2025}
Lewis Hammond et~al.
\newblock Multi-agent risks from advanced {AI}.
\newblock \emph{arXiv preprint arXiv:2502.14143}, 2025.

\bibitem[Hoare(1969)]{hoare1969}
C.~A.~R. Hoare.
\newblock An axiomatic basis for computer programming.
\newblock \emph{Communications of the ACM}, 12\penalty0 (10):\penalty0 576--580, 1969.

\bibitem[Hoeffding(1940)]{hoeffding1940}
Wassily Hoeffding.
\newblock Ma{\ss}stabinvariante {K}orrelationstheorie.
\newblock \emph{Schriften des Mathematischen Instituts und des Instituts f\"ur Angewandte Mathematik der Universit\"at Berlin}, 5:\penalty0 179--233, 1940.

\bibitem[Hoeffding(1963)]{hoeffding1963}
Wassily Hoeffding.
\newblock Probability inequalities for sums of bounded random variables.
\newblock \emph{Journal of the American Statistical Association}, 58\penalty0 (301):\penalty0 13--30, 1963.

\bibitem[Hong et~al.(2024)Hong, Zhuge, Chen, et~al.]{hong2024}
Sirui Hong, Mingchen Zhuge, Jonathan Chen, et~al.
\newblock {MetaGPT}: Meta programming for a multi-agent collaborative framework.
\newblock In \emph{ICLR}, 2024.

\bibitem[Howard et~al.(2021)Howard, Ramdas, McAuliffe, and Sekhon]{howard2021}
Steven~R. Howard, Aaditya Ramdas, Jon McAuliffe, and Jasjeet Sekhon.
\newblock Time-uniform, nonparametric, nonasymptotic confidence sequences.
\newblock \emph{The Annals of Statistics}, 49\penalty0 (2):\penalty0 1055--1080, 2021.

\bibitem[Huang et~al.(2026)Huang, Li, Li, Kwon, Yu, and Zhang]{cagecal2026}
Jiatan Huang, Mingchen Li, Ziming Li, Sunjae Kwon, Hong Yu, and Chuxu Zhang.
\newblock Counterfactual graph for multi-agent {LLM} calibration.
\newblock \emph{arXiv preprint arXiv:2605.30653}, 2026.

\bibitem[Jaccard(1912)]{jaccard1912}
Paul Jaccard.
\newblock The distribution of the flora in the alpine zone.
\newblock \emph{New Phytologist}, 11\penalty0 (2):\penalty0 37--50, 1912.

\bibitem[Jimenez et~al.(2024)Jimenez, Yang, Wettig, Yao, Pei, Press, and Narasimhan]{jimenez2024}
Carlos~E. Jimenez, John Yang, Alexander Wettig, Shunyu Yao, Kexin Pei, Ofir Press, and Karthik Narasimhan.
\newblock {SWE-bench}: Can language models resolve real-world {GitHub} issues?
\newblock In \emph{ICLR}, 2024.

\bibitem[Kendall(1938)]{kendall1938}
Maurice~G. Kendall.
\newblock A new measure of rank correlation.
\newblock \emph{Biometrika}, 30\penalty0 (1--2):\penalty0 81--93, 1938.

\bibitem[Kish(1965)]{kish1965}
Leslie Kish.
\newblock \emph{Survey Sampling}.
\newblock John Wiley and Sons, New York, 1965.

\bibitem[Lamport(2002)]{lamport2002}
Leslie Lamport.
\newblock \emph{Specifying Systems: The {TLA+} Language and Tools for Hardware and Software Engineers}.
\newblock Addison-Wesley, 2002.

\bibitem[Leino(2010)]{leino2010}
K.~Rustan~M. Leino.
\newblock Dafny: An automatic program verifier for functional correctness.
\newblock In \emph{LPAR}, pages 348--370, 2010.

\bibitem[Liang et~al.(2024)Liang, He, Jiao, et~al.]{liang2023}
Tian Liang, Zhiwei He, Wenxiang Jiao, et~al.
\newblock Encouraging divergent thinking in large language models through multi-agent debate.
\newblock In \emph{EMNLP}, 2024.

\bibitem[Lin(1991)]{lin1991}
Jianhua Lin.
\newblock Divergence measures based on the {S}hannon entropy.
\newblock \emph{IEEE Transactions on Information Theory}, 37\penalty0 (1):\penalty0 145--151, 1991.

\bibitem[Liu et~al.(2024)Liu, Yu, Zhang, et~al.]{liu2024agentbench}
Xiao Liu, Hao Yu, Hanchen Zhang, et~al.
\newblock {AgentBench}: Evaluating {LLMs} as agents.
\newblock In \emph{ICLR}, 2024.

\bibitem[McDonnell et~al.(2026)McDonnell, Singh, Pham, Havlik, and O'Hare]{triage2026}
Shay~Seiya McDonnell, Avantika Singh, Quoc-Viet Pham, Vratislav Havlik, and Gregory M.~P. O'Hare.
\newblock Harnessing disagreement: Detecting correlated agreement blindness in multi-agent triage.
\newblock \emph{arXiv preprint arXiv:2607.19899}, 2026.
\newblock Accepted, PAAMS 2026.

\bibitem[Meyer(1992)]{meyer1992}
Bertrand Meyer.
\newblock Applying ``design by contract''.
\newblock \emph{Computer}, 25\penalty0 (10):\penalty0 40--51, 1992.

\bibitem[Nelsen(2006)]{nelsen2006}
Roger~B. Nelsen.
\newblock \emph{An Introduction to Copulas}.
\newblock Springer, 2nd edition, 2006.

\bibitem[Nosek et~al.(2018)Nosek, Ebersole, DeHaven, and Mellor]{nosek2018}
Brian~A. Nosek, Charles~R. Ebersole, Alexander~C. DeHaven, and David~T. Mellor.
\newblock The preregistration revolution.
\newblock \emph{Proceedings of the National Academy of Sciences}, 115\penalty0 (11):\penalty0 2600--2606, 2018.

\bibitem[Ouyang et~al.(2022)Ouyang, Wu, Jiang, et~al.]{ouyang2022}
Long Ouyang, Jeffrey Wu, Xu~Jiang, et~al.
\newblock Training language models to follow instructions with human feedback.
\newblock In \emph{NeurIPS}, 2022.

\bibitem[Park et~al.(2023)Park, O'Brien, Cai, Morris, Liang, and Bernstein]{park2023}
Joon~Sung Park, Joseph~C. O'Brien, Carrie~J. Cai, Meredith~Ringel Morris, Percy Liang, and Michael~S. Bernstein.
\newblock Generative agents: Interactive simulacra of human behavior.
\newblock In \emph{UIST}, 2023.

\bibitem[Qian et~al.(2024)Qian, Liu, Liu, et~al.]{qian2024}
Chen Qian, Wei Liu, Hongzhang Liu, et~al.
\newblock {ChatDev}: Communicative agents for software development.
\newblock In \emph{ACL}, 2024.

\bibitem[Qiao et~al.(2026)Qiao, Tong, Lim, Liu, and Pang]{verifymas2026}
Hezhe Qiao, Hanghang Tong, Ee-Peng Lim, Bing Liu, and Guansong Pang.
\newblock {VerifyMAS}: Hypothesis verification for failure attribution in {LLM} multi-agent systems.
\newblock \emph{arXiv preprint arXiv:2605.17467}, 2026.

\bibitem[Qin et~al.(2026)Qin, Luan, See, Boukhers, Yang, and Li]{govcap2026}
Xue Qin, Simin Luan, John See, Zeyd Boukhers, Cong Yang, and Zhijun Li.
\newblock Governed capability evolution: Lifecycle-time compatibility checking and rollback for {AI}-component-based systems.
\newblock \emph{arXiv preprint arXiv:2604.08059}, 2026.

\bibitem[Rafi et~al.(2026)Rafi, Ahasanuzzaman, Kim, Wang, and Chen]{falat2026}
Md~Nakhla Rafi, Md~Ahasanuzzaman, Dong~Jae Kim, Zhijie Wang, and Tse-Hsun Chen.
\newblock {FALAT}: Tracing failures in {LLM} agent trajectories via dependency-guided search.
\newblock \emph{arXiv preprint arXiv:2606.00765}, 2026.

\bibitem[Ramdas et~al.(2023)Ramdas, Grünwald, Vovk, and Shafer]{ramdas2023}
Aaditya Ramdas, Peter Grünwald, Vladimir Vovk, and Glenn Shafer.
\newblock Game-theoretic statistics and safe anytime-valid inference.
\newblock \emph{Statistical Science}, 38\penalty0 (4):\penalty0 576--601, 2023.

\bibitem[Rebedea et~al.(2023)Rebedea, Dinu, Sreedhar, Parisien, and Cohen]{rebedea2023}
Traian Rebedea, Razvan Dinu, Makesh Sreedhar, Christopher Parisien, and Jonathan Cohen.
\newblock {NeMo} guardrails: A toolkit for controllable and safe {LLM} applications with programmable rails.
\newblock In \emph{EMNLP System Demonstrations}, 2023.

\bibitem[Robbins(1970)]{robbins1970}
Herbert Robbins.
\newblock Statistical methods related to the law of the iterated logarithm.
\newblock \emph{The Annals of Mathematical Statistics}, 41\penalty0 (5):\penalty0 1397--1409, 1970.

\bibitem[Shafer(2021)]{shafer2021}
Glenn Shafer.
\newblock Testing by betting: A strategy for statistical and scientific communication.
\newblock \emph{Journal of the Royal Statistical Society A}, 184\penalty0 (2):\penalty0 407--431, 2021.

\bibitem[Shafer and Vovk(2019)]{shafer2019}
Glenn Shafer and Vladimir Vovk.
\newblock \emph{Game-Theoretic Foundations for Probability and Finance}.
\newblock Wiley, 2019.

\bibitem[Shinn et~al.(2023)Shinn, Cassano, Berman, Gopinath, Narasimhan, and Yao]{shinn2023}
Noah Shinn, Federico Cassano, Edward Berman, Ashwin Gopinath, Karthik Narasimhan, and Shunyu Yao.
\newblock {Reflexion}: Language agents with verbal reinforcement learning.
\newblock In \emph{NeurIPS}, 2023.

\bibitem[Simmons et~al.(2011)Simmons, Nelson, and Simonsohn]{simmons2011}
Joseph~P. Simmons, Leif~D. Nelson, and Uri Simonsohn.
\newblock False-positive psychology: Undisclosed flexibility in data collection and analysis allows presenting anything as significant.
\newblock \emph{Psychological Science}, 22\penalty0 (11):\penalty0 1359--1366, 2011.

\bibitem[Sklar(1959)]{sklar1959}
Abe Sklar.
\newblock Fonctions de r\'epartition \`a $n$ dimensions et leurs marges.
\newblock \emph{Publications de l'Institut de Statistique de l'Universit\'e de Paris}, 8:\penalty0 229--231, 1959.

\bibitem[Slepian(1962)]{slepian1962}
David Slepian.
\newblock The one-sided barrier problem for {G}aussian noise.
\newblock \emph{Bell System Technical Journal}, 41\penalty0 (2):\penalty0 463--501, 1962.

\bibitem[Ville(1939)]{ville1939}
Jean Ville.
\newblock \emph{\'Etude critique de la notion de collectif}.
\newblock Gauthier-Villars, Paris, 1939.

\bibitem[Wald(1945)]{wald1945}
Abraham Wald.
\newblock Sequential tests of statistical hypotheses.
\newblock \emph{The Annals of Mathematical Statistics}, 16\penalty0 (2):\penalty0 117--186, 1945.

\bibitem[Wang et~al.(2023)Wang, Wei, Schuurmans, et~al.]{wang2023}
Xuezhi Wang, Jason Wei, Dale Schuurmans, et~al.
\newblock Self-consistency improves chain of thought reasoning in language models.
\newblock In \emph{ICLR}, 2023.

\bibitem[Waudby-Smith and Ramdas(2024)]{waudbysmith2024}
Ian Waudby-Smith and Aaditya Ramdas.
\newblock Estimating means of bounded random variables by betting.
\newblock \emph{Journal of the Royal Statistical Society B}, 86\penalty0 (1):\penalty0 1--27, 2024.

\bibitem[Wu et~al.(2024)Wu, Bansal, Zhang, et~al.]{wu2024autogen}
Qingyun Wu, Gagan Bansal, Jieyu Zhang, et~al.
\newblock {AutoGen}: Enabling next-gen {LLM} applications via multi-agent conversation.
\newblock In \emph{COLM}, 2024.

\bibitem[Yao et~al.(2023)Yao, Zhao, Yu, Du, Shafran, Narasimhan, and Cao]{yao2023}
Shunyu Yao, Jeffrey Zhao, Dian Yu, Nan Du, Izhak Shafran, Karthik Narasimhan, and Yuan Cao.
\newblock {ReAct}: Synergizing reasoning and acting in language models.
\newblock In \emph{ICLR}, 2023.

\bibitem[Yule(1900)]{yule1900}
G.~Udny Yule.
\newblock On the association of attributes in statistics: With illustrations from the material of the childhood society.
\newblock \emph{Philosophical Transactions of the Royal Society A}, 194:\penalty0 257--319, 1900.

\bibitem[Zheng et~al.(2025)Zheng, Chen, Yin, Zhang, Zeng, and Tian]{zhang2025byzantine}
Lifan Zheng, Jiawei Chen, Qinghong Yin, Jingyuan Zhang, Xinyi Zeng, and Yu~Tian.
\newblock Rethinking the reliability of multi-agent system: A perspective from {Byzantine} fault tolerance.
\newblock \emph{arXiv preprint arXiv:2511.10400}, 2025.

\bibitem[Zhou et~al.(2024)Zhou, Xu, Zhu, et~al.]{zhou2024}
Shuyan Zhou, Frank~F. Xu, Hao Zhu, et~al.
\newblock {WebArena}: A realistic web environment for building autonomous agents.
\newblock In \emph{ICLR}, 2024.

\end{thebibliography}

\end{document}